\documentclass{article}

\usepackage[final]{neurips_2022}

\usepackage[utf8]{inputenc} 
\usepackage[T1]{fontenc}    
\usepackage{hyperref}       
\usepackage{url}            
\usepackage{booktabs}       
\usepackage{amsfonts}       
\usepackage{nicefrac}       
\usepackage{microtype}      
\usepackage{xcolor}         

\usepackage{amssymb,amsmath,amsthm,bbm,mathrsfs}
\newtheorem{ass}{Assumption}
\newtheorem{clm}{Claim}
\newtheorem{lem}{Lemma}
\newtheorem{thm}{Theorem}
\newtheorem{rem}{Remark}
\usepackage{mathtools}
\mathtoolsset{showonlyrefs}
\newcommand{\E}{\mathbb{E}} 
\renewcommand{\P}{\mathbb{P}} 
\newcommand{\R}{\mathbb{R}} 
\newcommand{\N}{\mathbb{N}} 
\newcommand{\ind}{\mathbbm{1}}
\newcommand{\trans}{\mathsf{T}}
\renewcommand{\epsilon}{\varepsilon}

\newcommand{\mbf}[1]{\mathbf{#1}}
\newcommand{\mc}[1]{\mathcal{#1}}
\newcommand{\ip}[2]{\langle #1, #2 \rangle}
\DeclareMathOperator{\Tr}{Tr}
\DeclareMathOperator*{\argmax}{arg\,max}
\DeclarePairedDelimiter{\ceil}{\lceil}{\rceil}
\DeclarePairedDelimiter{\floor}{\lfloor}{\rfloor}
\usepackage{natbib}
\usepackage{color}
\usepackage{hyperref}
\hypersetup{
    colorlinks=true,
    linkcolor=blue,
   citecolor=blue
}
\usepackage[hang,flushmargin]{footmisc}
\usepackage[ruled,noend,nosemicolon]{algorithm2e}
\makeatletter
\renewcommand{\@algocf@capt@plain}{above}
\makeatother

\title{Minimax Regret for Cascading Bandits}

\author{Daniel Vial \\ UT Austin \& UIUC \\ \texttt{dvial@utexas.edu} \And Sujay Sanghavi \\ UT Austin \& Amazon  \\ \texttt{sanghavi@mail.utexas.edu} \AND Sanjay Shakkottai \\ UT Austin \\ \texttt{sanjay.shakkottai@utexas.edu} \And R.\ Srikant \\ UIUC \\ \texttt{rsrikant@illinois.edu}}

\begin{document}

\maketitle

\begin{abstract}
Cascading bandits is a natural and popular model that frames the task of learning to rank from Bernoulli click feedback in a bandit setting. For the case of unstructured rewards, we prove matching upper and lower bounds for the problem-independent (i.e., gap-free) regret, both of which strictly improve the best known. A key observation is that the hard instances of this problem are those with small mean rewards, i.e., the small click-through rates that are most relevant in practice. Based on this, and the fact that small mean implies small variance for Bernoullis, our key technical result shows that variance-aware confidence sets derived from the Bernstein and Chernoff bounds lead to optimal algorithms (up to log terms), whereas Hoeffding-based algorithms suffer order-wise suboptimal regret. This sharply contrasts with the standard (non-cascading) bandit setting, where the variance-aware algorithms only improve constants. In light of this and as an additional contribution, we propose a variance-aware algorithm for the structured case of linear rewards and show its regret strictly improves the state-of-the-art.
\end{abstract}

\section{Introduction} \label{secIntro}

The cascading click model describes users interacting with ranked lists, such as search results or online advertisements \citep{craswell2008experimental}. In this model, there is a set of items $[L] = \{1,\ldots,L\}$. The user is given a list of $K$ items $\mbf{A} = ( \mbf{a}_1,\ldots,\mbf{a}_K )$, sequentially examines the list, and clicks on the first attractive item (if any). If a click occurs, the user leaves without examining the subsequent items. It is assumed that the $e$-th item has an attraction probability $\bar{w}(e)$ (which we also call the mean reward), and that the random (Bernoulli) clicks are conditionally independent given $\mbf{A}$.

Cascading bandits, introduced concurrently by \cite{kveton2015cascading} and \cite{combes2015learning}, are a sequential learning version of the model where the mean rewards $\{ \bar{w}(e) \}_{e=1}^L$ are initially unknown. At each round $t \in [n]$, the learner chooses an action, which is a list of items $\mbf{A}_t = ( \mbf{a}_1^t , \ldots , \mbf{a}_K^t )$. As in the basic click model, the user scans the list and clicks on the first attractive item (if any). Thus, if the $\mbf{C}_t$-th item is clicked, the learner knows that the user was not attracted to $\{ \mbf{a}_k^t \}_{k=1}^{\mbf{C}_t-1}$ but was attracted to $\mbf{a}_{\mbf{C}_t}^t$. However, the learner receives \textit{no} feedback on the items $\{ \mbf{a}_k^t \}_{k=\mbf{C}_t+1}^K$ that the user did not examine before leaving. The objective for the learner is to choose the sequence $\{ \mbf{A}_t \}_{t=1}^n$ to maximize the expected number of clicks, or equivalently, minimize the regret defined in \eqref{eqRegret}.

This work provides problem-independent (i.e., gap-free) regret bounds for cascading bandits that strictly improve the state-of-the-art. In the case of unstructured rewards, our results provide the first minimax-optimal regret bounds (up to log terms). Our key insight is that, compared to the standard bandit problem, the reward variance plays an outsized role in the gap-free analysis. In particular, we show that \textit{for cascading bandits, the worst-case problem instances are those with low mean rewards} -- namely, $\bar{w}(e) \leq 1/K$, where the variance $\bar{w}(e)(1-\bar{w}(e))$ is also small. We emphasize that these worst-case instances are not pathological; rather, they model the low click-through rates that prevail in practice \citep{richardson2007predicting}. Further, we argue that \textit{adapting to low variance is crucial to cope with the worst-case instances}. Put differently, algorithms should be \textit{variance-aware}, i.e., more exploitative when the variance is small. We provide the intuition behind this key insight in Section \ref{secIntuition} and show how to formalize it with a proof sketch in Section \ref{secAnalysis}.

More specifically, our first goal is to establish upper bounds on the problem-independent regret (i.e., the maximum over $\bar{w}$) for cascading bandit algorithms, as well as minimax lower bounds (i.e., the infimum over algorithms of their problem-independent regret). As shown in Table \ref{tabSummary}, the tightest existing such bounds are $\tilde{O}(\sqrt{nLK})$ and $\Omega(\sqrt{nL/K})$, respectively. What is surprising is that, not only is there a gap between the bounds, but they \textit{increase} and \textit{decrease} in $K$, respectively. In other words, the following fundamental question is unresolved: \textit{as $K$ grows, does the problem become harder} (as suggested by the upper bound) \textit{or easier} (as suggested by the lower bound)\textit{?}

As discussed by \cite{zhong2021thompson}, this question lacks an obvious answer. On the one hand, larger $K$ means that the learner needs to identify more good items, which hints at a harder problem. On the other hand, the learner receives more feedback as $K$ grows, which intuitively makes the problem easier. As we show later, variance-aware algorithms are the key to resolving this tradeoff.

In addition to this basic version of the model -- hereafter, the \textit{tabular case}, where no structure is assumed for the mean rewards -- we are also interested in the \textit{linear case}, where $\bar{w}(e) = \ip{ \phi(e) }{\theta}$ for some known feature map $\phi : [L] \rightarrow \R^d$ and unknown parameter vector $\theta \in \R^d$. A second goal of this work is to apply our insights from the tabular case to the linear one, in hopes of improving the best known upper bound $\tilde{O}(\sqrt{n d^2 K})$ (see Table \ref{tabSummary}).

\begin{table*}[t]
\begin{center}
\caption{Problem-independent upper bounds and minimax lower bounds for cascading bandits with $L$ total items, $K$ recommended items, horizon $n$, and feature dimension $d$. The tabular columns assume unstructured mean reward. The linear column assumes reward is linear in unit-norm features and does not suppress any $L$ dependencies in the $\tilde{O}(\cdot)$ notation. See Appendix \ref{appRelated} for other related work.}\label{tabSummary}
\centering
\begin{tabular}{l @{\hspace{6pt}} c @{\hspace{6pt}} c @{\hspace{6pt}} c}
\hline & \multicolumn{2}{c}{{\bf Tabular case}} & {\bf Linear case} \\ 
{\bf Paper} & {\bf Upper bound} & {\bf Lower bound} & {\bf Upper bound}\\ \hline
\cite{zong2016cascading} & none & none & $\tilde{O}(\sqrt{n d^2 K^2})$ \\ 
\cite{wang2017improving} & $\tilde{O}(\sqrt{nLK})$ & none$^*$ & none \\ 
\cite{lattimore2018toprank} & $\tilde{O}(\sqrt{nLK^3} )$ & none$^*$ & none \\ 
\cite{li2018online} & none & none &  $\tilde{O}(\sqrt{n d^2 K}  )$$^\S$ \\
\cite{zhong2021thompson} & $\tilde{O}(\sqrt{nLK})$ & $\Omega(\sqrt{nL/K})$$^\dag$ & $\tilde{O}(\sqrt{n d^2 K^3 \min \{ d , \log L \}})$ \\
\cite{kveton2022value} & $\tilde{O}(\sqrt{nLK})$$^\ddag$ & none & none \\ 
Ours & $\tilde{O}(\sqrt{nL})$ (Thm \ref{thmKl}) & $\Omega(\sqrt{nL})$$^\dag$ (Thm \ref{thmLower}) & $\tilde{O}(\sqrt{nd(d+K)})$$^\S$ (Thm \ref{thmLinear}) \\ \hline
\end{tabular} 
\end{center}
{\footnotesize $^*$These papers contain minimax lower bounds but for click models that are distinct from cascading bandits.}

{\footnotesize $^\dag$These bounds assume $L$ is large compared to $K$ and $n$ is large compared to $L$ and $K$.}

{\footnotesize $^\ddag$This bound holds for a Bayesian notion of regret and includes a dependence on the prior not shown above.}

{\footnotesize $^\S$These bounds assume $n$ is large compared to $d$ and $K$ to suppress some additive $o(\sqrt{n})$ terms.}
\end{table*}

\subsection{Main contributions}

{\bf Tabular case.} First, we show no algorithm can achieve $o(\sqrt{nL})$ regret uniformly across $\bar{w}$, assuming $L$ is large compared to $K$ and $n$ is large compared to $L$ (see Theorem \ref{thmLower}). Next, we consider three algorithms: \texttt{CascadeKL-UCB}, \texttt{CascadeUCB-V}, and \texttt{CascadeUCB1} (the first and third are due to \cite{kveton2015cascading}; \texttt{CascadeUCB-V} is new). All three rank the $L$ items using upper confidence bounds (UCBs) and choose $\mbf{A}_t$ as the $K$ highest ranked items. They differ in the choice of UCB. As the names suggest, \texttt{CascadeKL-UCB} and \texttt{CascadeUCB-V} use \texttt{KL-UCB} \citep{garivier2011kl,maillard2011finite,cappe2013kullback} and \texttt{UCB-V} \citep{audibert2009exploration}, respectively. Both are variance-aware, in the sense that their respective UCBs are derived from the Chernoff and Bernstein inequalities. We show that both algorithms have near-optimal regret $\tilde{O}(\sqrt{nL})$ for any $\bar{w}$ (see Theorem \ref{thmKl}). In contrast, \texttt{CascadeUCB1} relies on the Hoeffing-style \texttt{UCB1} \citep{auer2002finite}, and we show that it suffers suboptimal regret $\Omega(\sqrt{nLK})$ on some $\bar{w}$ (see Theorem \ref{thmUcb}).

In summary, we prove (i) the minimax regret is $\tilde{\Theta}(\sqrt{nL})$ for cascading bandits, (ii) the variance-aware algorithms \texttt{CascadeKL-UCB} and \texttt{CascadeUCB-V} are minimax-optimal up to log terms, and (iii) the variance-unaware algorithm \texttt{CascadeUCB1} is decidely \textit{sub}optimal. Moreover, note from Table \ref{tabSummary} that we strictly improve \textit{both} the upper and lower bounds for this problem.

{\bf Discussion.} There are two surprising aspects to these results. First, the minimax bound $\tilde{\Theta}(\sqrt{nL})$ shows that (in a worst-case sense) the number of recommended items $K$ plays \textit{no role.}\footnote{The $\tilde{O}(\cdot)$ notation hides $\log K$ terms, but they can be bounded by $\log L$ while retaining $\tilde{O}(\sqrt{nL})$ regret.} In other words, the aforementioned tradeoff (identifying more good items but receiving more feedback) is perfectly balanced when the correct algorithm (e.g., \texttt{CascadeKL-UCB}) is employed.

The second (and arguably more surprising) aspect is that \texttt{CacadeKL-UCB} and \texttt{CascadeUCB-V} are optimal but \texttt{CascadeUCB1} is not. This stands in contrast to the standard $L$-armed bandit problem, where the analogous algorithms \texttt{KL-UCB}, \texttt{UCB-V}, and \texttt{UCB1} all achieve the minimax regret $\tilde{\Theta}(\sqrt{nL})$, and the main advantage of the former two is only to improve the constants in the gap-dependent bounds. We discuss the intuition behind this contrast in Section \ref{secIntuition}.

{\bf Linear case.} Motivated by these findings, we also propose a variance-aware algorithm for the linear case called \texttt{CascadeWOFUL}.\footnote{\texttt{WOFUL} stands for {\underline{w}}eighted {\underline{o}}ptimism in the {\underline{f}}ace of {\underline{u}}ncertainty for {\underline{l}}inear bandits.} \texttt{CascadeWOFUL} proceeds in two steps. First, we use the Hoeffding-style UCBs of \cite{abbasi2011improved} to upper bound the mean rewards $\bar{w}(e)$, and thus the Bernoulli variances $\bar{w}(e) ( 1 - \bar{w}(e) )$, with high probability. Second, we use these Hoeffding UCBs as proxies for the true variances in the \texttt{WOFUL} algorithm of \cite{zhou2021nearly}, which computes a variance-weighted estimate of $\theta$ that enjoys Bernstein-style concentration. In Theorem \ref{thmLinear}, we show the regret of \texttt{CascadeWOFUL} is $\tilde{O}(\sqrt{n d (d + K)})$ for large $n$, which improves existing bounds by a factor of at least $\sqrt{ \min \{d,K\} }$ (see Table \ref{tabSummary}).

\subsection{Why do variance-aware algorithms succeed but variance-unaware algorithms fail?} \label{secIntuition}

To answer this question, we focus on the tabular case and contrast the standard $L$-armed bandit setting with the cascading one. We first recall how the $\tilde{O}(\sqrt{nL})$ bound for $L$-armed bandits is derived. For simplicity, we restrict to instances with $\bar{w}(1) = p$ and $\bar{w}(2) = \cdots = \bar{w}(L) = p-\Delta$ for some $p \in (0,1)$ and $\Delta \in (0,p)$. In this case, \texttt{UCB1} plays each of the $L-1$ suboptimal items $1 / \Delta^2$ times, up to constants and log factors \citep{auer2002finite}. Each such play costs regret $\Delta$, for a total regret that scales as $L / \Delta$. Alternatively, regret can simply be bounded by $\Delta n$, since the $n$ plays incur at most $\Delta$ regret each. Combining the bounds gives $\min \{ L / \Delta , \Delta n \}$. The worst case occurs when the two bounds are equal, i.e., when $\Delta = \sqrt{L/n}$, which implies $\sqrt{nL}$ regret.

In contrast, \texttt{UCB-V} plays each suboptimal item $\sigma^2 / \Delta^2$ times (in an order sense), where $\sigma^2 \leq p$ is the variance of the $\text{Bernoulli}(p-\Delta)$ reward \citep{audibert2009exploration}. Therefore, the argument of the previous paragraph shows that regret grows as $\min \{ L p / \Delta , \Delta n \}$. Here the worst case occurs when $p$ is non-vanishing and $\Delta = \sqrt{L/n}$, which gives the same regret scaling as \texttt{UCB1}. A similar argument holds for \texttt{KL-UCB}, because the number of plays grows as $1 / d ( p - \Delta , p )$ \citep{cappe2013kullback} and one can show $d(p-\Delta,p) = \Omega ( \Delta^2 / p )$ (see Claim \ref{clmKlLower} in Appendix \ref{appKlProof}).

The analysis is more complicated for cascading bandits, because regret is nonlinear in the mean rewards and the amount of feedback is random. To oversimplify things, we draw an analogy with the above and assume $\bar{w}(1) = \cdots = \bar{w}(K) = p$ and $\bar{w}(K+1) = \cdots = \bar{w}(L) = p-\Delta$. In this case, \texttt{CascadeUCB1} similarly plays the $L-K$ suboptimal items $1/\Delta^2$ times each, which costs $L/\Delta$ regret when $L \gg K$. However, we can no longer bound regret by $\Delta n$, because the total number of plays depends on the random number of items the user examines at each round, which is roughly $\min \{ 1/p , K \}$ (the mean of $\text{Geometric}(p)$ random variable, truncated to the maximum $K$). Thus, the $\Delta n$ bound inflates to $\Delta n \min \{ 1/p, K \}$, which gives $\min \{  L / \Delta , \Delta n \min \{ 1/p, K \} \}$ regret. For $p \leq 1/K$ and $\Delta = \sqrt{L/(nK)}$, this yields the best known bound $\sqrt{nLK}$. We emphasize that, unlike the previous paragraph, \textit{the worst case here occurs when $p$ (i.e., the click-through rate) is small}.

On the other hand, the analogous bound for \texttt{CascadeUCB-V} and \texttt{CascadeKL-UCB} scales as $\min \{ L p / \Delta  , \Delta n \min \{ 1 / p , K \} \}$. Crucially, the factor of $p$ in the first term -- which arises due to the variance-aware nature of the algorithms -- offsets the factor $1/p$ in the second term. Thus, in the hard case $p \leq 1/K$, the bound becomes $\min \{ L / ( \Delta K ) , \Delta n K \}$. Here the worst case is $\Delta = \sqrt{L / ( n K^2)}$, which yields the minimax regret $\sqrt{nL}$ that we establish in Theorems \ref{thmLower} and \ref{thmKl}.

\section{Preliminaries} \label{secModel}

In this section, we precisely formulate our problem. A cascading bandit instance is defined by the triple $(L,K,\bar{w})$, where $L \in \N$ is the total number of items, $K \in [L] = \{1,\ldots,L\}$ is the number of items the learner displays to the user at each round, and $\bar{w} \in [0,1]^L$ is the vector of attraction probabilities. We define a sequential game as follows. At each round $t \in [n]$, the learner chooses an \textit{action} $\mbf{A}_t = ( \mbf{a}_1^t, \ldots , \mbf{a}_K^t )$, where $\mbf{a}_k^t \in [L]$ and $\mbf{a}_k^t \neq \mbf{a}_{k'}^t$ when $k \neq k'$ (i.e., $\mbf{A}_t$ is an ordered list of $K$ distinct items). The user sequentially examines this list, clicks on the first item that attracts them, and stops examining items after clicking. Mathematically, we denote the first attractive item by $\mbf{C}_t = \inf \{ k \in [K] : \mbf{w}_t(\mbf{a}_k^t) = 1 \}$, where $\mbf{w}_t(e) \sim \text{Bernoulli}(\bar{w}(e))$ for each $e \in [L]$. If no item is clicked, i.e., if $\mbf{w}_t(\mbf{a}_k^t) = 0$ for all $k$, we set $\mbf{C}_t = \infty$. Note the learner only observes the realizations corresponding to the items that the user examined, i.e., $\{ \mbf{w}_t(\mbf{a}_k^t) : k \in [ \min \{ \mbf{C}_t , K \} ] \}$.

We denote by $\mc{H}_t = \cup_{s=1}^{t-1} \{ \mbf{A}_s \} \cup \{ \mbf{w}_s(\mbf{a}_k^s) : k \in [ \min \{ \mbf{C}_s , K \} ] \}$ the \textit{history} of actions and observations before time $t$.\footnote{Our notation mostly follows \cite{kveton2015cascading}, but we clarify that their definition of $\mc{H}_t$ includes $\mbf{A}_t$.} We let $\P_t(\cdot) = \P ( \cdot | \mc{H}_t \cup \{ \mbf{A}_t \} )$ and $\E_t [ \cdot ] = \E [ \cdot | \mc{H}_t \cup \{ \mbf{A}_t \} ]$ denote conditional probability and expectation given the history and current action. In the cascading bandit model, it is assumed that the feedback $\{ \mbf{w}_t(\mbf{a}_k^t) \}_{k=1}^K$ (i.e., the presence or absence of clicks) is conditionally independent given $\mc{H}_t$ and $\mbf{A}_t$. Therefore, the conditional click probability is
\begin{equation*}
\P_t ( \mbf{C}_t < \infty ) = 1 - \P_t ( \mbf{C}_t = \infty ) = 1 - \P_t ( \cap_{k=1}^K \{ \mbf{w}_t(\mbf{a}_k^t) = 0 \} ) = 1 - \prod_{k=1}^K ( 1 - \bar{w}(\mbf{a}_k^t) ) .
\end{equation*}
Mappings from $\mc{H}_t$ to $\mbf{A}_t$ are called \textit{policies}. Let $\Pi$ be the set of all policies. Given an instance $\bar{w} \in [0,1]^L$, a policy $\pi \in \Pi$, and a horizon $n \in \N$, the expected number of clicks is
\begin{equation*}
\E_{\pi,\bar{w}} \left[ \sum_{t=1}^n \ind ( \mbf{C}_t < \infty ) \right] = \E_{\pi,\bar{w}} \left[ \sum_{t=1}^n \P_t ( \mbf{C}_t < \infty ) \right] = \E_{\pi,\bar{w}} \left[ \sum_{t=1}^n \left(  1 - \prod_{k=1}^K ( 1 - \bar{w}(\mbf{a}_k^t) )  \right) \right] ,
\end{equation*}
where $\ind$ is the indicator function. Let $A^* = ( a_1^* , \ldots , a_K^* )$ be any action $A = (a_1, \ldots , a_K)$ that maximizes the click probability $1 - \prod_{k=1}^K ( 1 - \bar{w}(a_k) )$. Our goal is to minimize \textit{regret}, which is the expected difference in the number of clicks between $\pi$ and the policy that always plays $A^*$, i.e.,
\begin{equation} \label{eqRegret}
R_{\pi,\bar{w}}(n) = \E_{\pi,\bar{w}} \left[ \sum_{t=1}^n \left( \prod_{k=1}^K ( 1 - \bar{w}(\mbf{a}_k^t) ) - \prod_{k=1}^K ( 1 - \bar{w}(a_k^*) ) \right) \right] .
\end{equation}

\section{Results for the tabular case} \label{secResultsTabular}

We can now state our tabular results (the proofs are discussed in Section \ref{secAnalysis}). First, we have a minimax lower bound showing no algorithm can achieve $o(\sqrt{nL})$ uniformly over the mean rewards $\bar{w}$.

\begin{thm}\label{thmLower}
Suppose $N \triangleq L/K \in \{4,5,\ldots\}$ and $n \geq L$. Then for any policy $\pi \in \Pi$, there exists a mean reward vector $\bar{w} \in [0,1]^L$ such that $R_{\pi,\bar{w}}(n) = \Omega(\sqrt{nL})$.
\end{thm}

\begin{rem} \label{remReduction}
The proof of Theorem \ref{thmLower} is essentially a reduction to \cite{lattimore2018toprank}'s lower bound for the so-called document-based click model. Their proof and ours both use the assumption $L/K \in \N$ to simplify the analysis, which involves partitioning the $L$ items into $K$ subsets of size $L/K$ each. When $L/K \notin \N$, one of the subsets will have fewer items, which makes the analysis more cumbersome; however, this does not fundamentally alter either result.
\end{rem}

\begin{rem}
The theorem also requires $L \geq 4 K$, which is not very restrictive since $L \gg K$ in typical applications. However, this assumption does eliminate an interesting analytical regime, namely, when $K \rightarrow L$. We conjecture the minimax lower bound is $\Omega(\sqrt{n(L-K)})$ in this case, since any algorithm obtains zero regret when $K=L$ and there are no suboptimal items.
\end{rem}

We next consider the algorithms \texttt{CascadeKL-UCB} and \texttt{CascadeUCB1} from \cite{kveton2015cascading}, along with a new one called \texttt{CascadeUCB-V}. All follow a similar template, which is given in Algorithm \ref{algUcb}. This is a natural generalization of upper confidence bound (UCB) algorithms from the standard $L$-armed bandit setting. At each round $t \in [n]$, it computes UCBs $\mbf{U}_t(e)$ in a manner to be specified shortly, then chooses $\mbf{A}_t$ as the $K$ items with the highest UCBs (in order of UCB). After observing the click feedback $\mbf{C}_t$, the algorithm increments the number of observations $\mbf{T}_t(e)$ and updates the empirical mean $\hat{\mbf{w}}_{\mbf{T}_t(e)}(e)$ for each item $e$ that the user examined.

\begin{algorithm}
\caption{General UCB algorithm for tabular cascading bandits} \label{algUcb}

Initialize number of observations $\mbf{T}_t(e) = 0$ and empirical mean $\hat{\mbf{w}}_0(e) = 0$ for each $e \in [L]$

\For{$t=1,\ldots,n$}{

Compute $\mbf{U}_t(e)$ for each $e \in [L]$ (by \eqref{eqKlUcb}, \eqref{eqUcbV}, and \eqref{eqUcb1} for \texttt{CascadeKL-UCB}, \texttt{CascadeUCB-V}, and \texttt{CascadeUCB1}, respectively)

\For{$k=1, \ldots, K$}{

Let $\mbf{a}_k^t = \argmax_{ e \in [L] \setminus \{ \mbf{a}_i^t \}_{i=1}^{k-1} } \mbf{U}_t(e)$ be the item with the $k$-th highest UCB

}

Play $\mbf{A}_t = ( \mbf{a}_1^t , \ldots , \mbf{a}_K^t )$ and observe $\mbf{C}_t = \inf \{ k \in [K] : \mbf{w}_t(\mbf{a}_k^t) = 1 \}$ (where $\inf \emptyset = \infty$)

Let $\mbf{T}_t(e) = \mbf{T}_{t-1}(e)$ for each $e \in [L]$

\For{$k = 1 , \ldots , \min \{ \mbf{C}_t , K \}$}{

$e \leftarrow \mbf{a}_k^t$, $\mbf{T}_t(e) \leftarrow \mbf{T}_t(e) + 1$, $\hat{\mbf{w}}_{\mbf{T}_t(e)}(e) \leftarrow (  \mbf{T}_{t-1}(e) \hat{\mbf{w}}_{\mbf{T}_{t-1}(e)}(e) + \ind ( \mbf{C}_t = k ) ) / \mbf{T}_t(e)$

}

}
\end{algorithm}

For \texttt{CascadeKL-UCB}, the UCBs $\mbf{U}_t(e)$ are computed as follows:
\begin{equation} \label{eqKlUcb}
\mbf{U}_t(e) = \max \{ u \in [0,1] : d ( \hat{\mbf{w}}_{\mbf{T}_{t-1}(e)}(e) , u ) \leq \log(f(t)) / \mbf{T}_{t-1}(e) \} ,
\end{equation}
where $f(t) = t(\log t)^3$ and $d(p,q) = p \log(p/q)+(1-p)\log((1-p)/(1-q))$ is the relative entropy between Bernoullis with means $p, q \in [0,1]$. The set in \eqref{eqKlUcb} is a confidence interval for $\bar{w}(e)$, which is derived from the Chernoff bound. For \texttt{CascadeUCB-V}, the UCBs are instead given by
\begin{equation} \label{eqUcbV}
\mbf{U}_t(e) = \hat{\mbf{w}}_{\mbf{T}_{t-1}(e)}(e) + \sqrt{ 4 \hat{\mbf{v}}_{\mbf{T}_{t-1}(e)}(e) \log(t) /  \mbf{T}_{t-1}(e) } + 6 \log (t) /  \mbf{T}_{t-1}(e) ,
\end{equation}
where $\hat{\mbf{v}}_s(e) = \hat{\mbf{w}}_s(e) ( 1 - \hat{\mbf{w}}_s(e) )$ is the empirical variance from $s$ observations of item $e$. This UCB is derived from the coarser -- but crucially, still variance-aware -- Bernstein inequality. Finally, for \texttt{CascadeUCB1}, the UCBs are derived from the Hoeffing bound and computed as follows:
\begin{equation} \label{eqUcb1}
\mbf{U}_t(e) = \hat{\mbf{w}}_{\mbf{T}_{t-1}(e)}(e) + c_{t,\mbf{T}_{t-1}(e)}  , \quad \text{where} \quad c_{t,s} = \sqrt{1.5 \log(t) / s} .
\end{equation}

We can now show the variance-aware UCBs are nearly optimal, while \texttt{CascadeUCB1} is suboptimal.

\begin{thm}\label{thmKl}
Suppose $\pi$ is \texttt{CascadeKL-UCB} or \texttt{CascadeUCB-V}, i.e., the policy from Algorithm \ref{algUcb} with the UCBs given by \eqref{eqKlUcb} or \eqref{eqUcbV}. Then $R_{\pi,\bar{w}}(n) = \tilde{O}(\sqrt{nL})$ for any $\bar{w} \in [0,1]^L$.
\end{thm}

\begin{rem} \label{remUcbV}
The reader may wonder why we proposed \texttt{CascadeUCB-V}, since the \texttt{CascadeKL-UCB} bound is enough to establish the minimax regret. The main reason is to demonstrate that variance-awareness alone (no additional information encoded by \texttt{KL-UCB}) is enough to achieve the optimal regret, which helps motivate our linear algorithm. Furthermore, we show empirically in Section \ref{secExperiments} that, while \texttt{CascadeUCB-V} is inferior to \texttt{CascadeKL-UCB} in terms of regret, its closed form nature leads to quicker computation, while still improving the regret of \texttt{CascadeUCB1}.
\end{rem}

\begin{thm}\label{thmUcb}
Suppose $n \geq \max \{ LK , 49 K^4 \}$, $L \geq 800 K$, and $\pi$ is \texttt{CascadeUCB1}, i.e., the policy from Algorithm \ref{algUcb} with the UCBs given by \eqref{eqUcb1}. Then $R_{\pi,\bar{w}}(n)  = \Omega(\sqrt{nLK})$ for some $\bar{w} \in [0,1]^L$.
\end{thm}

\begin{rem} \label{remUcb1}
The assumed bounds on $n$ and $L$ simplify the calculations and can be improved (see Remark \ref{remLoosenAssumption} in Appendix \ref{appProofUcb}). The main point of Theorem \ref{thmUcb} is to show that, unlike the variance-aware algorithms, \texttt{CascadeUCB1} cannot satisfy the conclusion of Theorem \ref{thmKl} for \emph{all} choices of $n$ and $L$.
\end{rem}

\section{Results for the linear case} \label{secResultsLinear}

In light of the previous section, we seek a variance-aware algorithm for the linear case. Our method is based on the \texttt{WOFUL} algorithm of \cite{zhou2021nearly}, which was designed for the standard linear bandit setting (the case $K=1$). In this section, we review \texttt{WOFUL}, discuss how to overcome its limitations in the cascading setting, explain our \texttt{CascadeWOFUL} algorithm, and bound its regret.

{\bf Existing algorithm.} In Section 4.3 of their work, \cite{zhou2021nearly} consider the following problem. As above, there is a set of items $[L]$, a known feature map $\phi : [L] \rightarrow \R^d$, and an unknown parameter vector $\theta \in \R^d$. Successive plays of $e \in [L]$ give i.i.d.\ rewards with mean $\bar{w}(e) = \ip{\phi(e)}{\theta}$ and variance upper bounded by $\sigma_e^2$. Thus, at each round $t \in [n]$, the learner chooses $\mbf{a}^t \in [L]$ and receives a random reward $r_t = \ip{ \phi(\mbf{a}^t) }{ \theta } + \eta_t$, where $\E [ \eta_t | \mbf{a}^t ] = 0$ and $\E [ \eta_t^2 | \mbf{a}^t ] \leq \sigma_{\mbf{a}^t}^2$. For this setting, the authors proposed the \texttt{WOFUL} algorithm, which is based on the (unweighted) \texttt{OFUL} algorithm \citep{abbasi2011improved}. At each round $t \in [n]$, \texttt{WOFUL} chooses the item
\begin{equation} \label{eqWOFUL}
\mbf{a}^t = \argmax_{e \in [L]} \left( \ip{\phi(e)}{\hat{\theta}_t} + \alpha \| \phi(e) \|_{\mbf{\Lambda}_t^{-1}} \right) ,
\end{equation}
where $\alpha > 0$ is an exploration parameter, $\| x \|_B = \sqrt{x^\trans B x}$ is the norm induced by a positive definite matrix $B$, and $\hat{\theta}_t$ is the regularized and variance-weighted least-squares estimate given by
\begin{equation} \label{eqWeightedLeastSq}
\hat{\theta}_t = \mbf{\Lambda}_t^{-1} \sum_{s=1}^{t-1} \phi ( \mbf{a}^s ) r_s / \sigma_{\mbf{a}^s}^2 ,  \quad \text{where} \quad \mbf{\Lambda}_t = I + \sum_{s=1}^{t-1} \phi ( \mbf{a}^s ) \phi ( \mbf{a}^s )^\trans / \sigma_{\mbf{a}^s}^2 .
\end{equation}
To gain some intuition, we assume momentarily that $d=L$ and $\phi(e)$ is the $e$-th standard basis vector, i.e., the vector with $1$ in the $e$-th coordinate and $0$ elsewhere. In this case, one can easily calculate
\begin{equation*}
\ip{ \phi(e) }{ \hat{\theta}_t } =  \frac{ \sum_{s \in [t-1] : \mbf{a}^s = e } r_s / \sigma_e^2 }{ 1 + \mbf{T}_{t-1}(e) / \sigma_e^2 } , \quad \| \phi(e) \|_{\mbf{\Lambda}_t^{-1}} = \sqrt{ \frac{1}{1 + \mbf{T}_{t-1}(e) / \sigma_e^2 } } .
\end{equation*}
Therefore, for large $\mbf{T}_{t-1}(e)$ (large enough that $1 + \mbf{T}_{t-1}(e) / \sigma_e^2 \approx \mbf{T}_{t-1}(e) / \sigma_e^2$), we have
\begin{equation*}
\ip{\phi(e)}{\hat{\theta}_t} + \alpha \| \phi(e) \|_{\mbf{\Lambda}_t^{-1}} \approx \frac{ \sum_{s \in [t-1] : \mbf{a}^s = e } r_s }{ \mbf{T}_{t-1}(e) } + \alpha \sqrt{ \frac{   \sigma_e^2 }{   \mbf{T}_{t-1}(e) } } .
\end{equation*}
Similar to \texttt{UCB-V} \eqref{eqUcbV}, the right side of this equation is the empirical mean plus an exploration bonus that grows with the variance upper bound $\sigma_e^2$ and decays in $\mbf{T}_{t-1}(e)$. Hence, the term inside the $\argmax$ in \eqref{eqWOFUL} can be interpreted as a Bernstein-style UCB. The analysis of this UCB relies on a novel concentration inequality for vector-valued martingales \citep[Theorem 2]{zhou2021nearly}, which is a Bernstein analogue of the Hoeffding-style bound due to \cite{abbasi2011improved}.

{\bf Limitations.} The fact that \texttt{WOFUL} (roughly) generalizes \texttt{UCB-V} is promising. However, computing \eqref{eqWeightedLeastSq} requires knowledge of the variance upper bounds $\sigma_e^2$, and nontrivial bounds on the variance are rarely available in practice. We sidestep this issue with three simple observations: (i) cascading bandits only involve Bernoulli rewards, (ii) for Bernoulli rewards, variances are upper bounded by means, and (iii) these means can be learned efficiently since they are linearly-parameterized. This suggests the following algorithm: first, compute Hoeffding-style UCBs; second, treat these UCBs as upper bounds for the true means, and thus the true variances, in \texttt{WOFUL}.

{\bf Proposed algorithm.} Algorithm \ref{algWoful} formalizes this approach. It contains three steps. First, step 1 defines Hoeffding-style UCBs $\mbf{U}_{t,H}$ as in \cite{abbasi2011improved}. Next, step 2 uses $\mbf{U}_{t,H}$ as an upper bound for the variance and computes the Bernstein-style UCBs $\mbf{U}_{t,B}$ as in \texttt{WOFUL}. Finally, step 3 chooses $\mbf{A}_t$ as the $K$ items with the highest $\mbf{U}_{t,B}$, analogous to Algorithm \ref{algUcb}.

Two technical clarifications are in order. First, observe that in step 1, we clip the variance bound $\mbf{U}_{t,H}$ below by $1/K$. We do so to ensure that $\mbf{\Lambda}_{t,B}$ (which inverts $\mbf{U}_{t,H}$) remains bounded. Additionally, we note the choice $1/K$ is precisely motivated by Section \ref{secIntuition}, which shows this is a critical threshold for the small click-through rate. Second, $\mbf{\Lambda}_{t,B}$ uses the regularizer $K I$. This is to ensure that the regularizer is large enough compared to the summands $\phi(\mbf{a}_k^s) \phi(\mbf{a}_k^s)^\trans / \mbf{U}_{s,H}(\mbf{a}_k^s)$, which scale as $K$ in the worst case where the variance upper bound $\mbf{U}_{s,H}(\mbf{a}_k^s)$ is clipped to $1/K$.

\begin{rem} \label{remLinUcb}
\cite{zong2016cascading} proposed a \texttt{WOFUL}-style cascading algorithm called \texttt{CascadeLinUCB} (see Appendix \ref{appExperiments}), but they set the variance upper bound to fixed $\sigma^2 > 0$. In fact, they remark ``ideally, $\sigma^2$ should be the variance of the observation noises,'' which \texttt{CascadeWOFUL} essentially \emph{learns}.
\end{rem}

\begin{rem} \label{remImplement}
Appendix \ref{appImplement} contains an improved version of \texttt{CascadeWOFUL}. It is more efficient (for example, the inverses are iteratively updated via Sherman-Morrison), satisfies the same theoretical guarantee as Algorithm \ref{algWoful}, and includes some tweaks that improve performance in practice. The tradeoff is that Algorithm \ref{algWoful} is simpler to explain, which is why we prefer it for the main text.
\end{rem}

\begin{algorithm}
\caption{\texttt{CascadeWOFUL} for linear cascading bandits} \label{algWoful}

\KwIn{exploration parameters $\{ \alpha_{t,H} , \alpha_{t,B} \}_{t=1}^n$}

\For{$t=1,\ldots,n$}{

{\bf Step 1:} Define the (clipped) Hoeffding-style UCBs 
\begin{gather*}
\mbf{U}_{t,H}(\cdot) = \max \left\{ \ip{\phi(\cdot)}{\hat{\mbf{\theta}}_{t,H}} + \alpha_{t,H} \| \phi(\cdot) \|_{\mbf{\Lambda}_{t,H}^{-1}} , 1/K \right\} , \quad \text{where} \\
\hat{\mbf{\theta}}_{t,H} = \mbf{\Lambda}_{t,H}^{-1} \sum_{s=1}^{t-1} \sum_{k=1}^{\min\{\mbf{C}_s,K\}} \phi(\mbf{a}_k^s) \mbf{w}_s (\mbf{a}_k^s)  ,\ \mbf{\Lambda}_{t,H} = I + \sum_{s=1}^{t-1} \sum_{k=1}^{\min\{\mbf{C}_s,K\}} \phi(\mbf{a}_k^s) \phi(\mbf{a}_k^s)^\trans
\end{gather*}

{\bf Step 2:} Define the Bernstein-style UCBs
\begin{gather*}
\mbf{U}_{t,B}(\cdot) = \ip{\phi(\cdot)}{\hat{\mbf{\theta}}_{t,B}} + \alpha_{t,B} \| \phi(\cdot) \|_{\mbf{\Lambda}_{t,B}^{-1}}  , \quad \text{where} \\
 \hat{\mbf{\theta}}_{t,B} = \mbf{\Lambda}_{t,B}^{-1} \sum_{s=1}^{t-1} \sum_{k=1}^{\min\{\mbf{C}_s,K\}} \frac{\phi(\mbf{a}_k^s) \mbf{w}_s (\mbf{a}_k^s)}{ \mbf{U}_{s,H}(\mbf{a}_k^s)} ,\ \mbf{\Lambda}_{t,B} = K I + \sum_{s=1}^{t-1} \sum_{k=1}^{\min\{\mbf{C}_s,K\}} \frac{\phi(\mbf{a}_k^s) \phi(\mbf{a}_k^s)^\trans}{ \mbf{U}_{s,H}(\mbf{a}_k^s)} 
\end{gather*}

{\bf Step 3:} Let $\mbf{a}_k^t = \argmax_{ e \in [L] \setminus \{ \mbf{a}_i^t \}_{i=1}^{k-1} } \mbf{U}_{t,B}(e)$ be the item with the $k$-th highest UCB, play $\mbf{A}_t = ( \mbf{a}_1^t , \ldots , \mbf{a}_K^t )$, observe $\mbf{C}_t = \inf \{ k \in [K] : \mbf{w}_t(\mbf{a}_k^t) = 1 \}$ (as in Algorithm \ref{algUcb})

}
\end{algorithm}

{\bf Regret bound.} We can now state our main result for the linear case. Here $\mathscr{B}_d(M) = \{ x \in \R^d : \| x \|_2 \leq M \}$ denotes the Euclidean ball of radius $M > 0$ in $\R^d$.\footnote{Our results hold with minor modification if $\phi : [L] \rightarrow \mathscr{B}_d(M_1)$ and $\theta \in \mathscr{B}_d(M_2)$ for general $M_1, M_2 > 0$. However, as in prior work, the modified algorithm needs to know an upper bound on $M_2$.}

\begin{thm} \label{thmLinear}
Suppose $\bar{w} \in [0,1]^L$, $\phi : [L] \rightarrow \mathscr{B}_d(1)$, and $\theta \in \mathscr{B}_d(1)$ satisfy $\bar{w}(e) = \ip{ \phi(e) }{ \theta }$ for all $e \in [L]$. Let $\pi$ be the policy of Algorithm \ref{algWoful} with inputs $\alpha_{t,H} = \sqrt{ d \log ( 1 + t K / d ) + 2 \log (n) } + 1$ and $\alpha_{t,B} = 8 \sqrt{ d \log ( 1 + t K / d ) \log (n^3 K) } +  4 \sqrt{K}  \log ( n^3 K ) + \sqrt{K}$. Then
\begin{equation*}
R_{\pi,\bar{w}}(n) = \tilde{O} \left( \min \left\{ \sqrt{n d (d+K)}  + n^{\frac{1}{6}} d^{\frac{7}{6}} (d+K)^{\frac{1}{2}} K^{\frac{1}{3}} , \sqrt{n} \max \{d,K\} \min\{d,K\}^{1/3} \right\} \right).
\end{equation*}
\end{thm}

Note this bound is completely independent of the number of items $L$. For large $n$ (typically the case of interest), it becomes $\tilde{O} ( \sqrt{n d ( d + K )} )$, which improves the best known bound of $\tilde{O}(\sqrt{n d^2 K})$ \citep{li2018online}. The theorem also establishes $\tilde{O}(\sqrt{n} \max \{ d, K \} \min \{d,K\}^{1/3} )$ regret \textit{uniformly} in $n$ (i.e., without additive $o(\sqrt{n})$ terms which may dominate for small $n$), which improves the best known uniform-$n$ bound of $\tilde{O}(\sqrt{n} d K)$ \citep{zong2016cascading}. See Table \ref{tabSummary} for more details.

\section{Overview of the analysis} \label{secAnalysis}

We next discuss the key ideas behind our proofs. The details are deferred to Appendices \ref{appProofNotes}-\ref{appProofUcb}.

{\bf Theorem \ref{thmKl} proof.} For simplicity, we assume the arms are ordered by their means, i.e., $\bar{w}(1) \geq \cdots \geq \bar{w}(L)$. Under this assumption, the optimal action is $[K]$. We call those items \textit{optimal} and $[L] \setminus [K]$ \textit{suboptimal}. We let $\mc{E}_t$ be the ``bad event'' that the empirical and true means differ substantially at time $t$, $\bar{\mc{E}}_t$ its complement, $G_{e,e^*,t}$ the event that suboptimal $e > K$ was chosen in favor of optimal $e^* \leq K$ and subsequently examined by the user, and $\Delta_{e,e^*} = \bar{w}(e^*) - \bar{w}(e)$ the \textit{reward gap}. Then as in Appendix A.1 of \cite{kveton2015cascading}, we ``linearize'' regret as follows:
\begin{equation} \label{eqUpperKveton}
R_{\pi,\bar{w}}(n) \leq \E \left[ \sum_{t \leq n} \sum_{e > K} \sum_{e^* \leq K} \Delta_{e,e^*} \ind ( \bar{\mc{E}}_t , G_{e,e^*,t} ) \right] + \sum_{t \leq n} \P(\mc{E}_t) .
\end{equation}
The second term is small due to concentration. For the first term, define $\Delta = \bar{w}(K) \sqrt{L/n}$ and assume $n \gg L$ (so $\Delta \ll \bar{w}(K)$).\footnote{The full proof addresses the cases $n \not \gg L$ and $\bar{w}(K) = 0$ (here we implicitly assume the latter to invert $\Delta$).} For each $e > K$, let $\mc{K}_\Delta(e) = \{ e^* \leq K : \Delta_{e,e^*} \leq \Delta \}$ be the optimal items with small gap relative to $e$ and $\bar{\mc{K}}_\Delta(e) = [K] \setminus \mc{K}_\Delta(e)$ the other optimal items. Then
\begin{equation*}
\sum_{e^* \leq K} \Delta_{e,e^*} \ind ( \bar{\mc{E}}_t , G_{e,e^*,t} ) \leq \Delta \sum_{e^* \in \mc{K}_\Delta(e)} \ind ( G_{e,e^*,t} ) + \sum_{e^* \in \bar{\mc{K}}_\Delta(e)} \Delta_{e,e^*} \ind ( \bar{\mc{E}}_t , G_{e,e^*,t} ) .
\end{equation*}
Therefore, we can upper bound the first term in \eqref{eqUpperKveton} by
\begin{equation} \label{eqUpperGaps}
\Delta \sum_{t \leq n} \E \left[ \sum_{e > K} \sum_{e^* \in \mc{K}_\Delta(e)} \ind ( G_{e,e^*,t} ) \right] + \E \left[ \sum_{t \leq n} \sum_{e > K} \sum_{e^* \in \bar{\mc{K}}_\Delta(e)} \Delta_{e,e^*} \ind ( \bar{\mc{E}}_t , G_{e,e^*,t} ) \right] .
\end{equation}
For the first term in \eqref{eqUpperGaps}, the inner double summation is the number of $e > K$ that were chosen in favor of $e^* \in \mc{K}_\Delta(e)$, and subsequently examined by the user, at round $t$. Denote this number by $\mbf{O}_t$. Recall $\bar{w}(e^*) \geq \bar{w}(K)$ (by the assumed ordering) and $\Delta \ll \bar{w}(K)$, so for $e^* \in \mc{K}_\Delta(e)$, we have
\begin{equation}
\bar{w}(e) = \bar{w}(e^*) - \Delta_{e,e^*} \geq \bar{w}(e^*) - \Delta \geq \bar{w}(K) - \Delta \gtrapprox \bar{w}(K).
\end{equation}
Hence, $\mbf{O}_t$ is bounded by the number of items $e$ with $\bar{w}(e) \approx \bar{w}(K)$ that the user examined. Since the user stops examining at the first attractive item, this number is dominated by a $\text{Geometric}(\bar{w}(K))$ random variable. Therefore, the first term in \eqref{eqUpperGaps} is $\tilde{O}(\Delta n / \bar{w}(K)) = \tilde{O}(\sqrt{nL})$. 

The second term in \eqref{eqUpperGaps} accounts for choosing $e$ instead of $e^*$ when $\Delta_{e,e^*} \geq \Delta$ and the empirical means are concentrated. Building upon the intuition of Section \ref{secIntuition}, we exploit the variance-awareness of \texttt{KL-UCB} and \texttt{UCB-V} (and use $\bar{w}(e) \leq \bar{w}(K)$ by the assumed ordering) to bound this term by $\tilde{O} ( L \bar{w}(e)(1-\bar{w}(e)) / \Delta ) \leq \tilde{O} ( L \bar{w}(K) / \Delta )$. Thus, by choice of $\Delta$, this term is $\tilde{O}(\sqrt{nL})$ as well.

{\bf Theorem \ref{thmLinear} proof.} We decompose regret similar to \eqref{eqUpperKveton} and \eqref{eqUpperGaps}, though with different choices of $\mc{E}_t$ and $\Delta$.\footnote{In fact, the proofs of Theorems \ref{thmKl} and \ref{thmLinear} both rely on a more general gap-free regret decomposition for cascading bandits (Lemma \ref{lemGenUpper} in Appendix \ref{appGeneral}), which to our knowledge is novel and may be of independent interest.} To bound $\P(\mc{E}_t)$, we use the aforementioned result of \cite{abbasi2011improved} to show that the Hoeffding UCBs $\mbf{U}_{t,H}$ upper bound the variances with high probability, then prove a guarantee for the least-squares estimate $\hat{\theta}_{t,B}$ using the Bernstein-style bound from \cite{zhou2021nearly}. We bound the first term in \eqref{eqUpperGaps} using the exact same logic as the tabular case. The second term has a more complicated analysis, but (as in the tabular case) it amounts to bounding the number of times that items $\Delta$-far from optimal are chosen when $\hat{\theta}_{t,B}$ is well-concentrated. For this, we adapt techniques from the standard linear bandit setting (the case $K=1$) to general $K \in [L]$.

{\bf Theorem \ref{thmLower} proof.} We let $\mc{H}_t' = \cup_{s=1}^{t-1} \{ \mbf{A}_s , \mbf{w}_s(\mbf{a}_1^s) , \ldots , \mbf{w}_s(\mbf{a}_K^s) \}$ be the \textit{entire} history before time $t$, which includes unobserved rewards $\mbf{w}_s(\mbf{a}_k^s) , k > \mbf{C}_s$. We also let $\Pi'$ be the policies that map $\mc{H}_t'$ to $\mbf{A}_t$. Note $\mc{H}_t \subset \mc{H}_t'$, so $\Pi \subset \Pi'$. For any $\pi \in \Pi'$ and $\bar{w} \in [0,1]^L$, we define
\begin{equation} \label{eqDbmRegret}
R_{\pi,\bar{w}}'(n) = \E_{\pi,\bar{w}} \left[ \sum_{t=1}^n \sum_{k=1}^K ( \bar{w}(a_k^*) - \bar{w}(\mbf{a}_k^t) ) \right] ,
\end{equation}
where $a_k^* = \argmax_{e \in [L] \setminus \{ a_i^* \}_{i=1}^{k-1} } \bar{w}(e)$. Then a lower bound linearization analogous to \eqref{eqUpperKveton} shows that for any $p \in [0,1]$ and $\bar{w} \in [0,p]^L$, $R_{\pi,\bar{w}}(n) \geq (1-p)^{K-1} R_{\pi,\bar{w}}'(n)$, which implies
\begin{equation} \label{eqReductionMain}
\inf_{\pi \in \Pi} \sup_{\bar{w} \in [0,1]^L} R_{\pi,\bar{w}}(n) \geq \inf_{\pi \in \Pi'} \sup_{\bar{w} \in [0,p]^L}  R_{\pi,\bar{w}}(n) \geq (1-p)^{K-1} \inf_{\pi \in \Pi'} \sup_{\bar{w} \in [0,p]^L} R_{\pi,\bar{w}}'(n) .
\end{equation}
If we choose $p = 1$, the $\inf \sup$ at right is the minimax regret for the document-based model that was analyzed by \cite{lattimore2018toprank} (see Remark \ref{remReduction}), but this makes \eqref{eqReductionMain} vacuous. On the other hand, by choosing $p = O ( 1/K )$ (again, the small click-through rate of Section \ref{secIntuition}), so that the term $(1-p)^{K-1}$ in \eqref{eqReductionMain} is $\Omega(1)$, we can modify their analysis to prove Theorem \ref{thmLower}.

{\bf Theorem \ref{thmUcb} proof.} We linearize the regret similar to the proof of Theorem \ref{thmLower}, then define a problem instance reminiscent of Section \ref{secIntuition} and (roughly) follow the intuition for \texttt{UCB1} therein.

\section{Experiments} \label{secExperiments}

Before closing, we conduct experiments on both synthetic and real data. Some details regarding experimental setup are deferred to Appendix \ref{appExperiments}. Code is available in the supplementary material.

\begin{figure}
\centering
\includegraphics[height=1.1in]{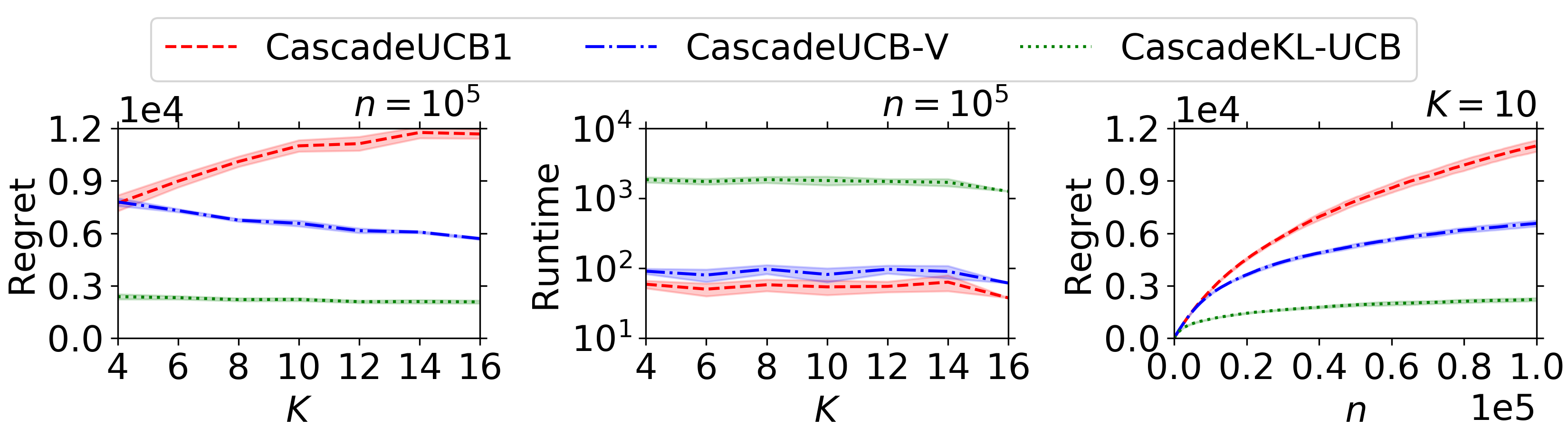}
\includegraphics[height=1.1in]{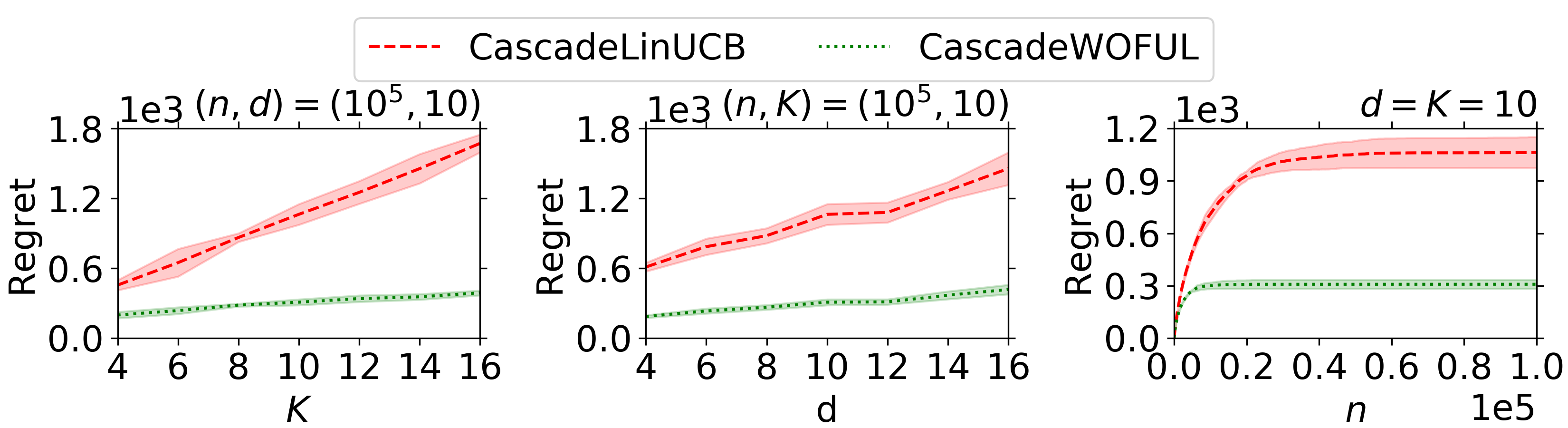}
\caption{Results for synthetic data (tabular on top, linear on bottom, $L=100$ in both)}
\label{figSyn}
\end{figure}

\begin{figure}
\centering
\includegraphics[height=1.3in]{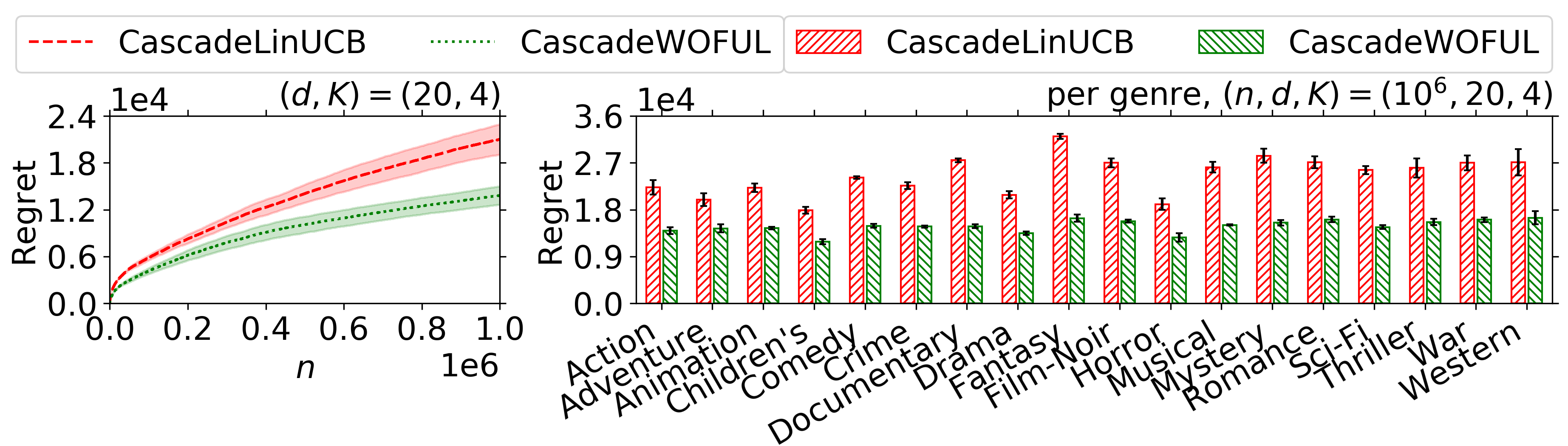}
\caption{Results for MovieLens data \citep{harper2015movielens}}
\label{figReal}
\end{figure}

{\bf Synthetic data.} We let $L=100$ and $K \in \{ 2 i \}_{i=2}^8$. For each $K$, we sample $\bar{w}(e)$ uniformly in $[\frac{2}{3K},\frac{1}{K}]$ for $e \leq K$ and in $[0, \frac{1}{3K}]$ for $e > K$. Note this yields a positive gap with the small click-through rate of Section \ref{secIntuition}. The top left plot in Figure \ref{figSyn} shows the regret at $n=10^5$ for the tabular albums of Section \ref{secResultsTabular} (the shaded regions are the standard deviations across five trials). As predicted by Theorems \ref{thmKl} and \ref{thmUcb}, the \texttt{CascadeUCB1} curve grows with $K$, while the variance-aware curves do not. The top middle plot shows the time spent computing UCBs for the same experiment, which confirms the behavior mentioned in Remark \ref{remUcbV}. For the linear case, we vary $d \in \{ 2 i \}_{i=2}^8$, generate the same $\bar{w}$, then compute unit-norm vectors $\theta$ and $\phi(e)$ satisfying $\bar{w}(e) = \ip{\phi(e)}{\theta}$ (see Appendix \ref{appExperiments}). We compare \texttt{CascadeWOFUL} to \texttt{CascadeLinUCB} \citep{zong2016cascading}, which \cite{li2018online} showed has the best existing regret guarantee. As suggested by Theorem \ref{thmLinear}, the left and middle plots on the bottom of Figure \ref{figSyn} show that our algorithm's regret has superior dependence on $K$ and $d$. The rightmost plots show that regret is sublinear in $n$ for the median values $K=d=10$.

{\bf Real data.} We replicate the first experiment from \cite{zong2016cascading} on the MovieLens-1M dataset (\url{grouplens.org/datasets/movielens/1m/}), which contains user ratings for $L \approx 4000$ movies. In brief, the setup is as follows. First, we use their default choices $d = 20$ and $K = 4$. Next, we divide the ratings into train and test sets based on the user who provided the rating. From the training data and a rank-$d$ SVD approximation, we learn a feature mapping $\phi$ from movies to the probability that a uniformly random training user rated the movie more than three stars. Finally, we run the algorithms as above, except at round $t \in [n]$, we sample a uniformly random user $\mbf{J}_t$ from the test set and define $\mbf{w}_t (\mbf{a}_k^t) = W(\mbf{J}_t, \mbf{a}_k^t)$, where $W(j,a) = \ind(\text{user $j$ rated movie $a$ more than $3$ stars})$. In other words, instead of the independent Bernoulli clicks of Section \ref{secModel}, we observe the actual feedback of user $\mbf{J}_t$. We point the reader to Section 4 of \cite{zong2016cascading} and Appendix \ref{appExperiments} for further details. The left plot of Figure \ref{figReal} shows that \texttt{CascadeWOFUL} outperforms \texttt{CascadeLinUCB} across $n$, eventually incurring less than $66\%$ of the regret. In addition to this setup from \cite{zong2016cascading}, we reran the experiment while restricting the set of items to movies of a particular genre, for each of $18$ genres in the dataset. This is intended to model platforms like Netflix that recommend movies in various categories. The right plot shows that \texttt{CascadeWOFUL} is superior for all genres; for some genres (e.g., fantasy) its regret is about half of \texttt{CascadeLinUCB}'s. Moreover, our experiments indicate that \texttt{CascadeWOFUL} improves \texttt{CascadeLinUCB} more dramatically for genres with smaller click-through rates (see Figure \ref{figRealApp} and surrounding discussion in Appendix \ref{appExperiments}), which reinforces a key message of this paper.

\section{Conclusion} \label{secConclusion}

In this work, we proved matching upper and lower bounds for the problem-independent regret of tabular cascading bandits and an upper bound for the linear case, all of which improve the best known. Our results suggest some interesting future directions, such as proving minimax lower bounds for the linear case and revisiting Thompson sampling for cascading bandits \citep{zhong2021thompson} in light of our variance-aware insight; see Appendix \ref{appFuture} for details. Finally, we note the paper is theoretical and has no immediate societal impact. Nevertheless, we urge caution for the negative impacts that could arise in practice. For example, our MovieLens experiments involved training on a subset of users, which could cause poor recommendations for demographics underrepresented in the training set. 

\begin{ack}
This work was partially supported by ONR Grant N00014-19-1-2566, NSF TRIPODS Grant 1934932, NSF Grants CCF 22-07547, CCF 19-34986, CNS 21-06801, 2019844, 2112471, 2107037, the Machine Learning Lab (MLL) at UT Austin, and the Wireless Networking and Communications Group (WNCG) Industrial Affiliates Program. We thank Advait Parulekar for helpful discussions.
\end{ack}

\bibliography{references}

\section*{Checklist}

\begin{enumerate}

\item For all authors...
\begin{enumerate}
  \item Do the main claims made in the abstract and introduction accurately reflect the paper's contributions and scope?
    \answerYes{The claims are accurate. Furthermore, we provide references to the formal results alongside each claim stated in the introduction.}
  \item Did you describe the limitations of your work?
    \answerYes{The main limitation is that our lower bounds (Theorems \ref{thmLower} and \ref{thmUcb}) require some additional assumptions. These are stated clearly in the theorem statements and discussed in Remarks \ref{remReduction} and \ref{remUcb1}.}
  \item Did you discuss any potential negative societal impacts of your work?
    \answerYes{Section \ref{secConclusion} addresses societal impact.}
  \item Have you read the ethics review guidelines and ensured that your paper conforms to them?
    \answerYes{}
\end{enumerate}

\item If you are including theoretical results...
\begin{enumerate}
  \item Did you state the full set of assumptions of all theoretical results?
    \answerYes{As mentioned above, the assumptions are clearly stated in the theorem statements.}
        \item Did you include complete proofs of all theoretical results?
    \answerYes{Section \ref{secIntuition} discusses the high-level intuition behind our proofs, Section \ref{secAnalysis} contains a more detailed proof sketch, and the appendices include complete proofs.}
\end{enumerate}

\item If you ran experiments...
\begin{enumerate}
  \item Did you include the code, data, and instructions needed to reproduce the main experimental results (either in the supplemental material or as a URL)?
    \answerYes{Complete code to recreate all plots is available in the supplementary material.}
  \item Did you specify all the training details (e.g., data splits, hyperparameters, how they were chosen)?
    \answerYes{The experimental setup is discussed in detail in Appendix \ref{appExperiments}, and other training details can be found in the code.}
        \item Did you report error bars (e.g., with respect to the random seed after running experiments multiple times)?
    \answerYes{All plots include error bars.}
        \item Did you include the total amount of compute and the type of resources used (e.g., type of GPUs, internal cluster, or cloud provider)?
    \answerYes{The experiments were run over several hours on a laptop, so no significant resources were used. Nevertheless, Appendix \ref{appExperiments} mentions the approximate runtime needed to recreate the figures.}
\end{enumerate}

\item If you are using existing assets (e.g., code, data, models) or curating/releasing new assets...
\begin{enumerate}
  \item If your work uses existing assets, did you cite the creators?
    \answerYes{We only used the MovieLens dataset, which we cited and also provided a link to.}
  \item Did you mention the license of the assets?
    \answerYes{The link to the MovieLens dataset contains a README with license information.}
  \item Did you include any new assets either in the supplemental material or as a URL?
    \answerNA{}
  \item Did you discuss whether and how consent was obtained from people whose data you're using/curating?
    \answerNA{}
  \item Did you discuss whether the data you are using/curating contains personally identifiable information or offensive content?
    \answerNA{}
\end{enumerate}

\item If you used crowdsourcing or conducted research with human subjects...
\begin{enumerate}
  \item Did you include the full text of instructions given to participants and screenshots, if applicable?
    \answerNA{}
  \item Did you describe any potential participant risks, with links to Institutional Review Board (IRB) approvals, if applicable?
    \answerNA{}
  \item Did you include the estimated hourly wage paid to participants and the total amount spent on participant compensation?
    \answerNA{}
\end{enumerate}

\end{enumerate}

\newpage \allowdisplaybreaks \appendix

\section{Details on related work} \label{appRelated}

In addition to the papers listed in Table \ref{tabSummary}, several others have considered cascading bandits and variants. In this appendix, we separately discuss the relevant work in the tabular and linear cases. We also point the reader to \cite{chuklin2015click} for a survey of click models, and to Chapter 32 of \cite{lattimore2020bandit} for an introduction to bandit-style ranking problems.

{\bf Tabular case.}
As mentioned in Section \ref{secIntro}, \cite{kveton2015cascading} and \cite{combes2015learning} introduced cascading bandits concurrently (though the latter work did not use that name and their model is more general). Both proved problem-dependent regret bounds of the form $(L-K) \log(n) / \Delta$, where $\Delta$ is a certain notion of mean reward gap. \cite{kveton2015combinatorial} provided gap-free bounds for a model whose reward structure is more general than that of cascading bandits. However, these results involve other problem-dependent quantities, so they can be much larger than ours for worst case mean reward vectors $\bar{w}$. \cite{zoghi2017online} generalized cascading bandits -- essentially, by generalizing the underlying click model -- and derived gap-dependent bounds. \cite{lattimore2018toprank} further generalized the underlying click model and established gap-free bounds (those shown in Table \ref{tabSummary}). \cite{li2019cascading} and \cite{wang2021near} proved gap-dependent bounds for a non-stationary variant of the cascading bandit model, where $\bar{w}$ changes arbitrarily at $M$ rounds. Theorem 3 of the latter also reports the gap-free lower bound $\Omega(\sqrt{nL(M+1)})$, which in our case $M=0$ appears tighter than the bound $\sqrt{nL/K}$ that the main text stated is the best known. However, their proof shows that the $\Omega(\cdot)$ notation hides a term which is exponentially small in $K$. Finally, the concurrent work by \cite{liu2022batch} also examines the role of variance in cascading bandits (as a special case of a more general model), though their focus is on gap-dependent regret bounds.

{\bf Linear case.} \cite{li2016contextual} considered the linear case of the model from \cite{kveton2015combinatorial}. Analogous to \cite{kveton2015combinatorial}'s results in tabular case, their regret bounds can be much worse than ours in the worst case. \cite{li2018online} proposed a model where a random user arrives at each round and the parameter vector $\theta$ depends on that user. When specialized to our case (i.e., the case of a single $\theta$), their algorithm reduces to \texttt{CascadeLinUCB} \citep{zong2016cascading}, and they sharpen the bound from \cite{zong2016cascading} by a factor of $\sqrt{K}$ (at the cost of some $o(\sqrt{n})$ additive terms, i.e., their bound does not hold uniformly in $n$); see Table \ref{tabSummary} for details. \cite{hiranandani2020cascading} studied a model that generalizes cascading bandits to account for position bias, but their regret bounds include problem-dependent quantities. \cite{li2019online} established $O(\sqrt{n d \log L} K)$ regret for a generalized click model reminiscent of the ones studied by \cite{zoghi2017online} and \cite{lattimore2018toprank} in the tabular case. Note this bound can be arbitrarily worse than those in Table \ref{tabSummary} due to the $\log L$ factor. Moreover, \cite{li2019online}'s experiments show the algorithm attaining their upper bound performs worse than \texttt{CascadeLinUCB} when their general model is specialized to cascading bandits (see their Figure 2), which is another reason why we used the latter for experiments. Finally, we note the work by \cite{zhong2021thompson} mentioned above is an extended version of the paper by \cite{cheung2019thompson}.

\section{Improved version of \texttt{CascadeWOFUL}} \label{appImplement}

Algorithm \ref{algWofulEff} provides the version of \texttt{CascadeWOFUL} that was mentioned in Remark \ref{remImplement}. There are four changes from Algorithm \ref{algWoful}. First, Step 1 defines the Bernstein UCB $\mbf{U}_{t,B}(e)$ from Algorithm \ref{algWoful}, along with an (unclipped) Hoeffding UCB $\tilde{\mbf{U}}_{t,H}(e)$. We then define the UCB $\mbf{U}_t(e)$ as the minimum of these two UCBs and $1$. The basic intuition is that $\mbf{U}_{t,B}(e)$, $\tilde{\mbf{U}}_{t,H}(e)$, and $1$ all upper bound $\bar{w}(e)$ (the former two with high probability; the latter almost surely), so we should use the tightest of the three for the final UCB. Second, Step 2 chooses the action $\mbf{A}_t$ similar to Algorithm \ref{algWoful}, though it is greedy with respect to $\mbf{U}_t(e)$ instead of $\mbf{U}_{t,B}(e)$. Third, Steps 3 and 4 update the inverses $\mbf{\Lambda}_{t,H}^{-1}$ and $\mbf{\Lambda}_{t,B}^{-1}$ iteratively, which reduces their computational complexity from $d^3$ to $d^2$. Furthermore, note that when updating the latter, we clip  $\tilde{\mbf{U}}_{t,H}(\mbf{a}_k^t)$ below by $1/K$ (as in Algorithm \ref{algWoful}) and above by $1$ (following the same intuition described above) when defining the normalized features $\phi_t(\mbf{a}_k^t)$. Fourth, Steps 5 and 6 iteratively update the estimates $\hat{\theta}_{t,H}$ and $\hat{\theta}_{t,B}$, again to save computation. In Appendix \ref{appImprovedProof}, we sketch a proof that this version also satisfies the guarantee of Theorem \ref{thmLinear}.

\begin{algorithm}
\caption{Improved version of \texttt{CascadeWOFUL} (recommended over Algorithm \ref{algWoful} in practice)} \label{algWofulEff}

\KwIn{exploration parameters $\{ \alpha_{t,H} , \alpha_{t,B} \}_{t=1}^n$}

Initialize $\hat{\theta}_{1,H} = 0$, $\mbf{\Lambda}_{1,H} = I$, $\mbf{\Lambda}_{1,H}^{-1} = I$, $\hat{\theta}_{1,B} = 0$,  $\mbf{\Lambda}_{1,B} = K I$, $\mbf{\Lambda}_{1,B}^{-1} = I/K$

\For{$t=1,\ldots,n$}{

\vspace{0.1in}

\textit{Step 1: compute upper confidence bounds (UCBs)}

\For{$e = 1 , \ldots , L$}{

$\tilde{\mbf{U}}_{t,H}(e) = \ip{\phi(e)}{\hat{\theta}_{t,H}} + \alpha_{t,H} \| \phi(e) \|_{\mbf{\Lambda}_{t,H}^{-1}}$ (Hoeffding-style UCB)

$\mbf{U}_{t,B}(e) = \ip{\phi(e)}{\hat{\mbf{\theta}}_{t,B}} + \alpha_{t,B} \| \phi(e) \|_{\mbf{\Lambda}_{t,B}^{-1}}$ (Bernstein-style UCB)

$\mbf{U}_t(e) = \min \{  \tilde{\mbf{U}}_{t,H}(e) , \mbf{U}_{t,B}(e) , 1 \}$ (tightest of the two UCBs and the trivial UCB)

}

\vspace{0.1in}

\textit{Step 2: play greedily with respect to UCBs (similar to Algorithms \ref{algUcb} and \ref{algWoful})}

\For{$k=1, \ldots, K$}{

$\mbf{a}_k^t = \argmax_{ e \in [L] \setminus \{ \mbf{a}_i^t \}_{i=1}^{k-1} } \mbf{U}_t(e)$ ($k$-th highest UCB; ties broken arbitrarily)

}

Choose $\mbf{A}_t = ( \mbf{a}_1^t , \ldots , \mbf{a}_K^t )$, observe $\mbf{C}_t = \inf \{ k \in [K] : \mbf{w}_t(\mbf{a}_k^t) = 1 \}$ (where $\inf \emptyset = \infty$)

\vspace{0.1in}

\textit{Step 3: update $\mbf{\Lambda}_{\cdot,H}$, $\mbf{\Lambda}_{\cdot,H}^{-1}$, $\mbf{\Lambda}_{\cdot,B}$, and $\mbf{\Lambda}_{\cdot,B}^{-1}$ (the inverses via Sherman-Morrison)}

Initialize $\mbf{\Lambda}_{t+1,H} = \mbf{\Lambda}_{t,H}$, $\mbf{\Lambda}_{t+1,H}^{-1} = \mbf{\Lambda}_{t,H}^{-1}$, $\mbf{\Lambda}_{t+1,B} = \mbf{\Lambda}_{t,B}$, and $\mbf{\Lambda}_{t+1,B}^{-1} = \mbf{\Lambda}_{t,B}^{-1}$

\For{$k=1, \ldots , \min \{ \mbf{C}_t , K \}$}{

$\mbf{\Lambda}_{t+1,H} \leftarrow \mbf{\Lambda}_{t+1,H} + \phi(\mbf{a}_k^t) \phi(\mbf{a}_k^t)^\trans$

$\mbf{\Lambda}_{t+1,H}^{-1} \leftarrow \mbf{\Lambda}_{t+1,H}^{-1} - \mbf{\Lambda}_{t+1,H}^{-1} \phi(\mbf{a}_k^t) \phi(\mbf{a}_k^t)^\trans \mbf{\Lambda}_{t+1,H}^{-1} / ( 1 + \phi(\mbf{a}_k^t)^\trans \mbf{\Lambda}_{t+1,H}^{-1} \phi(\mbf{a}_k^t) )$

$\mbf{U}_{t,H}(\mbf{a}_k^t) = \min \{ \max \{ \tilde{\mbf{U}}_{t,H}(\mbf{a}_k^t) , 1/K \} , 1 \}$ (clipped Hoeffding UCB)

$\phi_t(\mbf{a}_k^t) = \phi(\mbf{a}_k^t) / \sqrt{\mbf{U}_{t,H}(\mbf{a}_k^t)}$ (normalized feature vector)

$\mbf{\Lambda}_{t+1,B} \leftarrow \mbf{\Lambda}_{t+1,B} + \phi_t(\mbf{a}_k^t) \phi_t(\mbf{a}_k^t)^\trans$

$\mbf{\Lambda}_{t+1,B}^{-1} \leftarrow \mbf{\Lambda}_{t+1,B}^{-1} - \mbf{\Lambda}_{t+1,B}^{-1} \phi_t(\mbf{a}_k^t) \phi_t(\mbf{a}_k^t)^\trans \mbf{\Lambda}_{t+1,B}^{-1} / ( 1 + \phi_t(\mbf{a}_k^t)^\trans \mbf{\Lambda}_{t+1,B}^{-1} \phi_t(\mbf{a}_k^t) )$

}

\vspace{0.1in}

\textit{Step 4: compute $\hat{\theta}_{t+1,H}$ and $\hat{\theta}_{t+1,B}$ (the two cases correspond to ``click'' and ``no click'')}

\If{$\mbf{C}_t < \infty$}{

$\hat{\theta}_{t+1,H} = \mbf{\Lambda}_{t+1,H}^{-1} ( \mbf{\Lambda}_{t,H} \hat{\theta}_{t,H} + \phi ( \mbf{a}_{\mbf{C}_t}^t ) )$, $\hat{\theta}_{t+1,B} = \mbf{\Lambda}_{t+1,B}^{-1} ( \mbf{\Lambda}_{t,B} \hat{\theta}_{t,B} + \phi_t ( \mbf{a}_{\mbf{C}_t}^t ) )$

}
\Else{

$\hat{\theta}_{t+1,H} = \mbf{\Lambda}_{t+1,H}^{-1} \mbf{\Lambda}_{t,H} \hat{\theta}_{t,H}$, $\hat{\theta}_{t+1,B} = \mbf{\Lambda}_{t+1,B}^{-1} \mbf{\Lambda}_{t,B} \hat{\theta}_{t,B}$

}

\vspace{0.1in}

}
\end{algorithm}

\section{Details on experiments} \label{appExperiments}

{\bf Tabular algorithms.} For \texttt{CascadeUCB-V} and \texttt{CascadeUCB1}, respectively, we compute the UCBs $\mbf{U}_t(e)$ as shown in \eqref{eqUcbV} and \eqref{eqUcb1}, respectively, then take a minimum with the resulting UCB and 1 (analogous to Algorithm \ref{algWofulEff}). For \texttt{CascadeKL-UCB} (which lacks a closed form), the fact that $d(\hat{\mbf{w}}_{\mbf{T}_{t-1}(e)}(e), \cdot)$ is increasing on $[\hat{\mbf{w}}_{\mbf{T}_{t-1}(e)}(e),1]$ implies that we can estimate the UCB \eqref{eqKlUcb} via binary search, as shown in Algorithm \ref{algKlBinSearch} (we chose $\texttt{tol} = 10^{-4}$ in experiments).

{\bf Linear algorithms.} \texttt{CascadeLinUCB} \citep{zong2016cascading} is shown in Algorithm \ref{algLinUcb}. As mentioned in Remark \ref{remLinUcb}, it is a \texttt{WOFUL}-style algorithm that uses fixed $\sigma^2 > 0$ for the variances. We set $\sigma^2 = 1$ (as in their experiments) and choose the exploration parameter $\alpha = \sqrt{ d \log ( 1 + n K / d ) + 2 \log ( n K ) } + 1$ (which is the choice dictated by their Theorem 1 in the case $\sigma^2 = 1$ and $\| \theta \|_2 \leq 1$). For the experiments, we use an efficient implementation of Algorithm \ref{algLinUcb} analogous to Algorithm \ref{algWofulEff}. 

\begin{algorithm}[t]
\caption{Estimating the \texttt{KL-UCB} \eqref{eqKlUcb} via binary search} \label{algKlBinSearch}

$\ell = \hat{\mbf{w}}_{\mbf{T}_{t-1}(e)}(e), u = 1$

\While{$u - \ell > \texttt{tol}$}{

\lIf{$d(\hat{\mbf{w}}_{\mbf{T}_{t-1}(e)}(e) , ( u + \ell )/2) > \log(f(t)) / \mbf{T}_{t-1}(e)$}{$u = ( u + \ell ) / 2$}

\lElse{$\ell = ( u + \ell ) / 2$}

}

\Return{$u \approx \max \{ u' \in [0,1] : d(\hat{\mbf{w}}_{\mbf{T}_{t-1}(e)}(e) , u' ) \leq \log(f(t)) / \mbf{T}_{t-1}(e) \}$}
\end{algorithm}

\begin{algorithm}
\caption{\texttt{CascadeLinUCB} for linear cascading bandits \citep{zong2016cascading}} \label{algLinUcb}

\KwIn{exploration parameter $\alpha$, variance parameter $\sigma^2$}

\For{$t=1,\ldots,n$}{

Compute $\mbf{U}_t(e) = \min \{ \ip{\phi(e)}{\hat{\mbf{\theta}}_t} + \alpha \| \phi(e) \|_{\mbf{\Lambda}_t^{-1}} , 1 \}$ for each $e \in [L]$, where
\begin{equation*}
\hat{\mbf{\theta}}_t = \mbf{\Lambda}_t^{-1} \sum_{s=1}^{t-1} \sum_{k=1}^{\min\{\mbf{C}_s,K\}} \phi(\mbf{a}_k^s) \mbf{w}_s (\mbf{a}_k^s) / \sigma^2 ,\ \mbf{\Lambda}_t = I + \sum_{s=1}^{t-1} \sum_{k=1}^{\min\{\mbf{C}_s,K\}} \phi(\mbf{a}_k^s) \phi(\mbf{a}_k^s)^\trans / \sigma^2
\end{equation*}

Let $\mbf{a}_k^t = \argmax_{ e \in [L] \setminus \{ \mbf{a}_i^t \}_{i=1}^{k-1} } \mbf{U}_t(e)$ be the item with the $k$-th highest UCB, play $\mbf{A}_t = ( \mbf{a}_1^t , \ldots , \mbf{a}_K^t )$, observe $\mbf{C}_t = \inf \{ k \in [K] : \mbf{w}_t(\mbf{a}_k^t) = 1 \}$

}
\end{algorithm}

For \texttt{CascadeWOFUL}, we use the Hoeffding parameter suggested by Theorem \ref{thmLinear}, namely,
\begin{equation*}
\alpha_{t,H} = \beta_{t,H} + 1 , \quad \text{where} \quad \beta_{t,H} = \sqrt{ d \log ( 1 + t K / d ) + 2 \log (n) } .
\end{equation*}
Note the term $\beta_{t,H}$ accounts for concentration of the self-normalized estimate $\hat{\theta}_{t,H}$ of $\theta$ \citep{abbasi2011improved}, while the term $1$ accounts for the bias introduced by the regularizer $I$ used to compute $\hat{\theta}_{t,H}$. On the other hand, we found the Bernstein parameter $\alpha_{t,B}$ stated in Theorem \ref{thmLinear} is too conservative in practice, so we instead chose $\alpha_{t,B} = \beta_{t,H} + \sqrt{K}$, i.e., we use the same concentration term but change the other term to account for the $K I$ regularizer in $\mbf{\Lambda}_{t,B}$. Therefore, compared to the theorem, the experimental choice of $\alpha_{t,B}$ uses a smaller concentration parameter. We believe this is reasonable because the choice dictated by the theorem is likely loose. This is because the theoretical choice is inherited from the Bernstein analysis of \cite{zhou2021nearly}, which is based on the Hoeffding analysis of \cite{dani2008stochastic}, which \cite{abbasi2011improved} improved in the Hoeffding case. Put differently, state-of-the-art Bernstein bounds seem to be looser, which leads to a larger Bernstein parameter in Theorem \ref{thmLinear} that hamstrings the practical performance \texttt{CascadeWOFUL} compared to \texttt{CascadeLinUCB}.\footnote{The looser parameter is good enough to prove Theorem \ref{thmLinear}, so this is mostly an experimental issue.} Therefore, we use a tighter parameter for a fairer comparison.

{\bf Synthetic $\theta$ and $\phi$.} For the linear case of the synthetic data experiments, we randomly sample $\bar{w}(e)$ as described in Section \ref{secExperiments}, then construct $\theta$ and $\phi$ as follows. First, we generate $\theta$ and $\gamma_1$ uniformly on the sphere. Next, we define $\gamma_2 =  \gamma_1 - \ip{\gamma_1}{\theta} \theta$ and $\bar{\theta} = \gamma_2 / \| \gamma_2 \|_2$. Note that $\theta$ and $\bar{\theta}$ are on the unit sphere and are orthogonal. Therefore, if we define $\phi(e) = \bar{w}(e) \theta + \sqrt{ 1 - \bar{w}(e)^2 } \bar{\theta}$, we obtain
\begin{gather}
\ip{\phi(e)}{\theta} = \bar{w}(e) \| \theta \|_2^2 +  \sqrt{ 1 - \bar{w}(e)^2 } \ip{\bar{\theta}}{\theta} = \bar{w}(e) , \\
\| \phi(e) \|_2^2 = \bar{w}(e)^2 \|\theta\|_2^2 +  ( 1 - \bar{w}(e)^2 ) \| \bar{\theta} \|_2^2 + 2 \bar{w}(e)  \sqrt{ 1 - \bar{w}(e)^2 } \ip{\theta}{\bar{\theta}} = 1 .
\end{gather}
In summary, the construction above yields unit-norm features $\{ \phi(e) \}_{e=1}^L$ and a unit-norm parameter $\theta$ such that $\ip{\phi(e)}{\theta} = \bar{w}(e)$ for each $e \in [L]$.\footnote{This may seem overcomplicated; for example, we could have simply generated $\theta$ on the unit sphere and set $\phi(e) = \bar{w}(e) \theta$ to ensure $\ip{\phi(e)}{\theta} = \bar{w}(e)$. However, the latter is a bit unsatisfying, because then $\|\phi(e)\|_2 = \bar{w}(e) \|\theta\|_2 = \bar{w}(e)$, so knowledge of the features implies knowledge of the rewards.}

{\bf MovieLens, all genre experiment.} From the ratings data, we extract $W_{all} \in \{0,1\}^{m \times L}$, where $m = 6040$ is the number of users, $L = 3952$ is the number of movies, and $W_{all}(j,e) = 1$ if and only if the $j$-th user rated the $e$-th movie more than three stars. For each experimental trial, we randomly partition the rows of $W_{all}$ into $W_{train} \in \{0,1\}^{(m/2) \times L}$ and $W_{test} \in \{0,1\}^{(m/2) \times L}$, i.e., we partition the users into train and test sets. We compute a rank-$d$ SVD approximation $W_{train} \approx U \Sigma V^\trans$, where $U \in \R^{(m/2) \times d}$ and $V \in \R^{L \times d}$ have orthogonal columns and $\Sigma \in \R^{d \times d}$ is diagonal and nonnegative. Let $\{ \tilde{\phi}(e) \}_{e=1}^L$ be the rows of $V \Sigma$ (viewed as column vectors), $\iota = \max_{e \in [L]} \| \tilde{\phi}(e) \|_2$, $\phi(e) = \tilde{\phi}(e) / \iota$, $1_{m/2} \in \R^{m/2}$ the column vector of ones, and $\theta = 2 \iota U^\trans 1_{m/2} / m$. Notice that
\begin{equation*}
\phi(e)^\trans \theta = \frac{V(e,:) \Sigma U^\trans 1_{m/2}}{m/2} \approx \frac{W_{train}^\trans(e,:) 1_{m/2}}{ m/2 } = \frac{\sum_{j=1}^{m/2} W_{train}(j,e) }{m/2}  ,
\end{equation*} 
i.e., $\ip{\phi(e)}{\theta}$ is (approximately) the fraction of users in the training set who rated $e$ more than $3$ stars. Therefore, if this fraction is roughly similar across train and test sets, and if the implicit low-rank assumption approximately holds, then $\ip{\phi(e)}{\theta}$ is (roughly) the probability that a uniformly random user from the test set rates $e$ more than $3$ stars. Consequently, $\phi$ is a reasonable feature mapping if the Bernoulli click feedback $\{ \mbf{w}_t(\mbf{a}_k^t) \}_{k=1}^K$ is replaced by $\{ W_{test}(\mbf{J}_t,\mbf{a}_k^t) \}_{k=1}^K$, where $\mbf{J}_t$ is sampled uniformly from the $m/2$ users in the test set, as mentioned in Section \ref{secExperiments}.

Two technical clarifications are in order. First, note that we normalized $\tilde{\phi}(e)$ to ensure $\phi(e)$ has unit norm and absorbed the normalization constant $\iota$ into $\theta$. We initially ran the experiment without this normalization, but this yielded $\max_{e \in [L]} \|\phi(e)\|_2 \geq 10$, which violates the assumptions of Theorem \ref{thmLinear}. In contrast, the normalization ensures $\max_{e \in [L]} \|\phi(e)\|_2 \leq 1$ (as in the theorem), and we found in practice that $\theta$ remained in the unit ball (again, as in the theorem) after absorbing the constant $\iota$.

The second clarification is with regards to the optimal policy. In particular, replacing the independent Bernoullis $\{ \mbf{w}_t(\mbf{a}_k^t) \}_{k=1}^K$ with $\{ W_{test}(\mbf{J}_t,\mbf{a}_k^t) \}_{k=1}^K$ means that they are no longer conditionally independent given $\{ \mbf{a}_k^t \}_{k=1}^K$. As a consequence, finding the optimal policy is computationally hard. However, as discussed by \cite{zong2016cascading}, one can efficiently compute a constant factor approximation via the greedy algorithm that iteratively adds the item to the optimal set which attracts the most users (from the test set) who are not already attracted to at least one optimal item. We therefore use this greedy poilcy (instead of the globally optimal one) when computing regret.

{\bf MovieLens, per genre experiment.} For each genre, we run the same experiment but only include the movies whose lists of genres contain the given one. The left plot in Figure \ref{figRealApp} shows the regret of \texttt{CacadeWOFUL} relative to that of \texttt{CascadeLinUCB}, as a function of the click probability for the aforementioned greedy policy that we treat as optimal. Aside from some apparent outliers (e.g., ``Western''), the two quantities are roughly correlated. This suggests that \texttt{CacadeWOFUL} offers the most dramatic improvement when even the greedy policy struggles for clicks, which is reminiscent of the small click-through rate of Section \ref{secIntuition}. Additionally, we note that the average (across genres) of the relative regret is about $60\%$, which is better than the $66\%$ figure quoted in Section \ref{secExperiments} for the experiment that ignored genres. The reason for this discrepancy seems to be that genres with higher click rates under the near-optimal greedy policy (for which the relative regret is also higher -- e.g., ``Comedy'' and ``Drama'' in the left plot) are overrepresented in the dataset, as shown at right.

\begin{figure}
\centering
\includegraphics[height=2.4in]{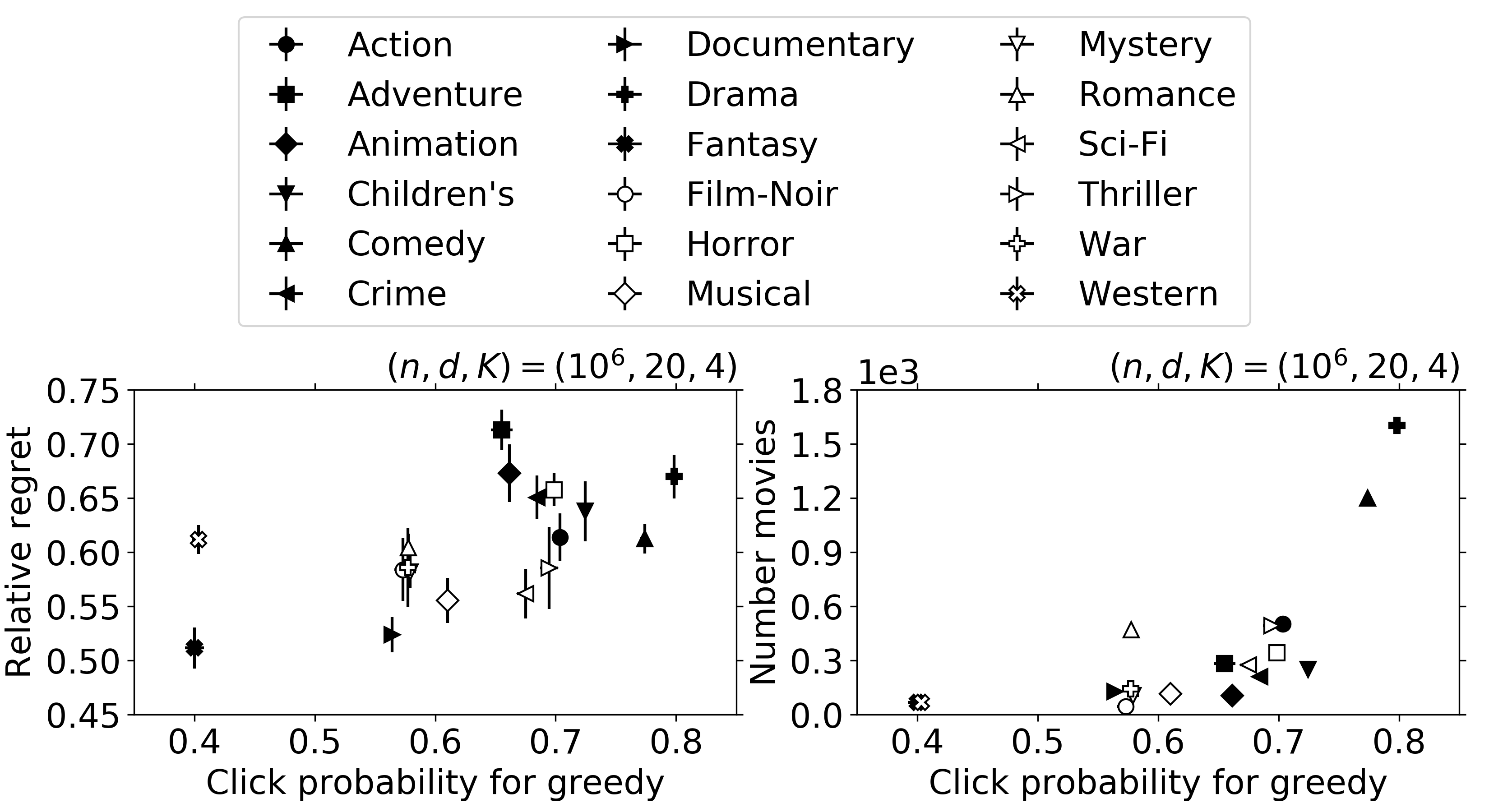}
\caption{Additional results for MovieLens-1M data}
\label{figRealApp}
\end{figure}

{\bf Other information on experiments.} Python code to recreate Figures \ref{figSyn}, \ref{figReal}, and \ref{figRealApp} is included in the supplementary material. All experiments were run on laptop with a 2.4GHz processor and 16GB of memory. The experiments took on the order of hours to complete -- specifically, $<20$ hours when the three Python files corresponding to synthetic tabular data, synthetic linear data, and real data are executed in parallel, or $\approx 27$ hours when executed sequentially via the shell script included with the supplementary material. See the ``README.txt'' file therein for further details.

\section{Future directions} \label{appFuture}

As mentioned in the conclusion, our work leaves several problems open.

{\bf Linear lower bounds.} While our upper and lower bounds match in the tabular case (up to logarithmic terms), we currently lack a lower bound for the linear case. We conjecture that this bound is $\Omega(\sqrt{n} d)$, which would demonstrate that our algorithm’s upper bound $\tilde{O}(\sqrt{n d (d+K)})$ is nearly optimal (at least when $K = \tilde{O}(d)$, which we feel is the more reasonable case in applications). This conjecture is based on the observations that (1) our tabular lower bound $\Omega(\sqrt{nL})$ matches the non-cascading (i.e., multi-armed bandit) bound up to constants, and (2) in the linear case, the analogous non-cascading bound is $\Omega(\sqrt{n} d)$ (see, e.g., Chapter 24 of \cite{lattimore2020bandit} and the references therein). We have yet to prove this conjecture, but we believe it can be established using the ideas behind our tabular lower bound that were discussed in Section \ref{secAnalysis} -- first, linearize the regret; next, lower bound the $\inf$ over cascading policies $\Pi$ by the $\inf$ over policies $\Pi'$ which observe the entire history $\mathcal{H}_t'$; and finally, lower bound the resulting non-cascading regret.

{\bf Thompson sampling (TS) for cascading bandits.} We focused on UCB algorithms, but TS solutions have also been proposed. An excellent reference is the recent journal paper by \cite{zhong2021thompson}. Among other results, they prove $\tilde{O}(\sqrt{nLK})$ regret for their tabular algorithm \texttt{TS-Cascade}. Note this algorithm explicitly uses the empirical variance $\hat{\mbf{v}}_t(e)$ as in \texttt{UCB-V} \eqref{eqUcbV} (see their Algorithm 2), so our ideas may be able to improve its regret bound to $\tilde{O}(\sqrt{nL})$ (the same improvement Theorem \ref{thmKl} establishes for variance-aware UCB compared to prior work). On the other hand, their linear algorithm \texttt{LinTS-Cascade} chooses items $e$ greedily with respect to $\ip{\phi(e)}{\theta_{t,H}^{TS}}$, where $\theta_{t,H}^{TS}$ is a Gaussian TS with variance-\textit{un}aware mean $\hat{\theta}_{t,H}$ and covariance $\mbf{\Lambda}_{t,H}^{-1}$ (in the notation of our Algorithm \ref{algWoful}). Hence, we believe \texttt{LinTS-Cascade} can be improved by replacing $\hat{\theta}_{t,H}$ and $\mbf{\Lambda}_{t,H}^{-1}$ with our variance-aware analogues $\hat{\theta}_{t,B}$ and $\mbf{\Lambda}_{t,B}^{-1}$. A tentative numerical comparison is shown in Figure \ref{figTS}, which confirms this belief under the experimental setup from the bottom right plot in Figure \ref{figSyn}. See Section 7 of \cite{zhong2021thompson} for a detailed numerical comparison of existing TS and UCB approaches.

\begin{figure}
\centering
\includegraphics[height=1.1in]{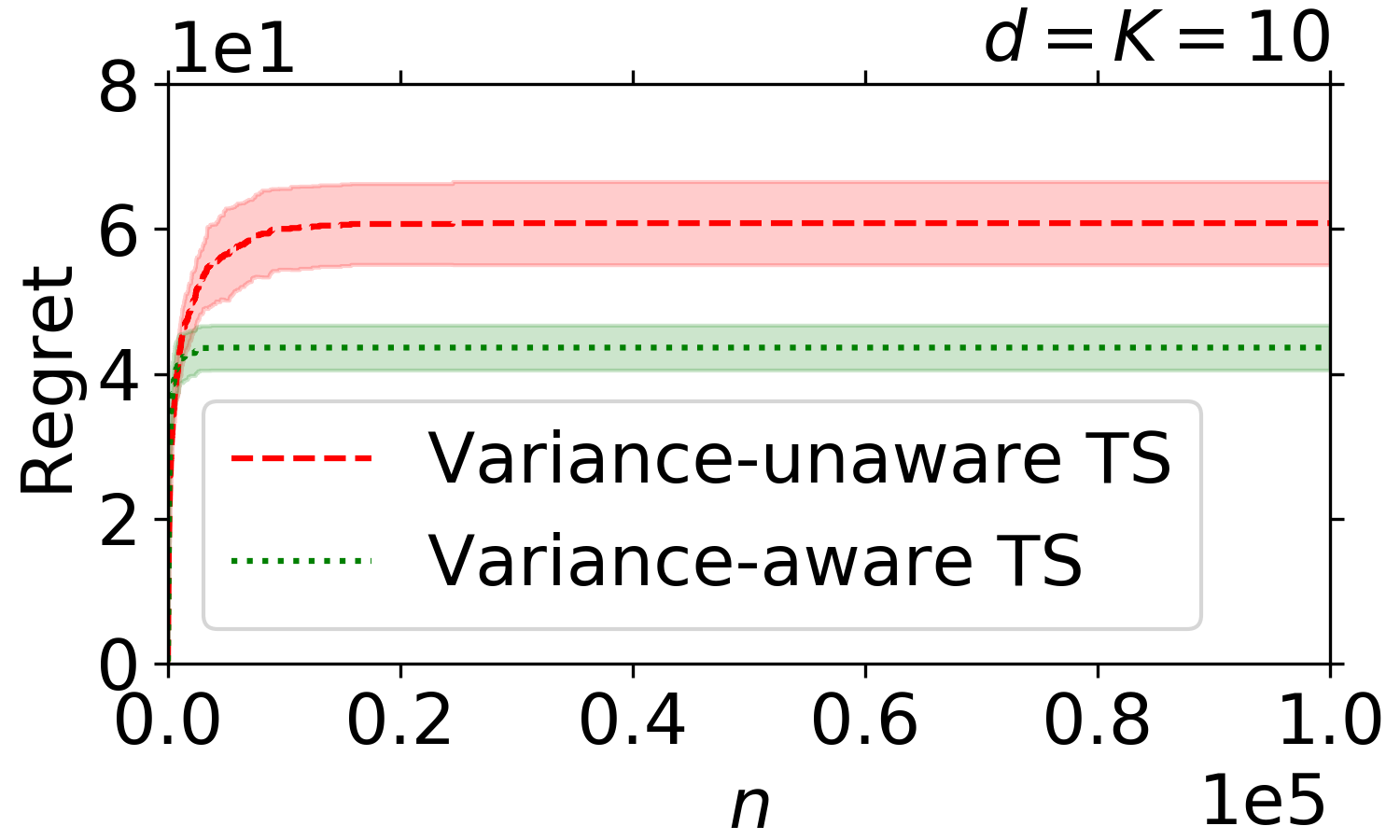}
\caption{Results for Thompson sampling experiment discussed in Appendix \ref{appFuture}}
\label{figTS}
\end{figure}

\section{Notes on proofs} \label{appProofNotes}

The remaining five appendices contain the complete proofs of our theoretical results. We begin with a general upper bound in Appendix \ref{appGeneral}, which is used to prove the upper bounds from Theorems \ref{thmKl} and \ref{thmLinear} in Appendices \ref{appKlProof} and \ref{appLinear}, respectively. We then prove the lower bounds from Theorems \ref{thmLower} and \ref{thmUcb} in Appendices \ref{appLower} and \ref{appProofUcb}, respectively. Throughout, all random variables are defined on a probability space $(\Omega,\mc{F},\P)$ with $\bar{A} = \Omega \setminus A$ denoting the complement of $A \in \mc{F}$ and $\sigma(\mc{X})$ the $\sigma$-algebra generated by a set of random variables $\mc{X}$. Other notation is local to each appendix; for example, $\Delta$ is a small positive number that takes different values in different appendices.

\section{General gap-free regret decomposition} \label{appGeneral}

As mentioned in the footnotes of Section \ref{secAnalysis} (and the previous appendix), the proofs of our upper bounds rely on a general regret decomposition (Lemma \ref{lemGenUpper} below). More specifically, this decomposition is intended for use in gap-free regret analyses of cascading bandit algorithms. Its proof is fairly straightforward but contains ideas that are crucial in formalizing the variance-aware insight discussed intuitively in Section \ref{secIntuition}. Moreover, it is general enough to be the starting point for the proofs of both our tabular and linear upper bounds. Therefore, it may be of independent interest.

Before stating this result, we introduce some notation. We fix an instance $\bar{w} \in [0,1]^L$ and a policy $\pi \in \Pi$ and suppress these objects in our notation; e.g., we write $R(n)$ instead of $R_{\pi,\bar{w}}(n)$. We will also assume the following, which is without loss of generality (after possibly relabeling items).
\begin{ass} \label{assOrder}
We have $\bar{w}(1) \geq \cdots \geq \bar{w}(L)$.
\end{ass}
Note that under this assumption, the optimal action is $A^* = (1,\ldots,K)$. We will therefore refer to $[K]$ as the \textit{optimal items} and $[L] \setminus [K]$ as the \textit{suboptimal items}.

Next, we recall some notation and basic observations from \cite{kveton2015cascading}. For any optimal item $e^* \in [K]$ and suboptimal item $e \in [L] \setminus [K]$, define the reward gap $\Delta_{e,e^*}= \bar{w}(e^*)-\bar{w}(e) \in [0,1]$. Let $\rho_t : [K] \rightarrow [K]$ be any $\sigma( \mc{H}_t \cup \{ \mbf{A}_t \} )$-measurable permutation satisfying $\mbf{U}_t ( \mbf{a}_k^t ) \geq \mbf{U}_t ( \rho_t(k) )$ for each $k \in [K]$ (such a permutation exists by Theorem 1 of \cite{kveton2015cascading}). For each optimal item $e^* \in [K]$, suboptimal item $e \in [L] \setminus [K]$, and time $t \in [n]$, define the event
\begin{equation*}
G_{e,e^*,t} = \cup_{k=1}^K \{ \mbf{a}_k^t = e , \rho_t(k) = e^* , \mbf{w}_t(\mbf{a}_1^t) = \cdots = \mbf{w}_t(\mbf{a}_{k-1}^t) = 0 \} .
\end{equation*}
Observe that when $G_{e,e^*,t}$ occurs, we have $\mbf{a}_k^t = e$ and $\rho_t(k) = e^*$ for some $k \in [K]$, where the permutation $\rho_t$ satisfies $\mbf{U}_t ( \mbf{a}_k^t ) \geq \mbf{U}_t ( \rho_t(k) )$ for all $k \in [K]$. Therefore, we know that
\begin{equation} \label{eqGtoUCB}
G_{e,e^*,t} \subset \{ \mbf{U}_t(e) \geq \mbf{U}_t(e^*) \}\ \forall\ e \in [L] \setminus [K] , e^* \in [K] , t \in [n] .
\end{equation}
Furthermore, note that by definition, at most one of the events $\{ G_{e,e^*,t} \}_{e^*=1}^K$ can occur for each $(e,t)$ pair. If any of them do, then the reward from item $e$ is observed at time $t$, so $\mbf{T}_t(e) = \mbf{T}_{t-1}(e)+1$; if none of them do, the reward is not observed and $\mbf{T}_t(e) = \mbf{T}_{t-1}(e)$. Consequently, we have
\begin{equation} \label{eqNumberOfG}
\sum_{e^*=1}^K \ind ( G_{e,e^*,t} ) = \ind(\mbf{T}_t(e) = \mbf{T}_{t-1}(e)+1) = \mbf{T}_t(e) - \mbf{T}_{t-1}(e)\ \forall\ e \in [L] \setminus [K] , t \in [n] .
\end{equation}
Finally, since the rewards from at most $K$ items are observed as time $t$, \eqref{eqNumberOfG} implies
\begin{equation} \label{eqSumOfG}
\sum_{e \in \mathscr{L}} \sum_{e^* \in \mathscr{K}} \ind ( G_{e,e^*,t} ) \leq K\ \forall\ \mathscr{L} \subset [L], \mathscr{K} \subset [K] , t \in [n] .
\end{equation}
We point the reader to Theorem 1 of \cite{kveton2015cascading} and its proof for further explanation.

We are now in position to state and prove the main result of this appendix.
\begin{lem} \label{lemGenUpper}
Fix $\Delta \in (0,\max\{ \bar{w}(K)/2 , 1 / K \})$ and, for each $e \in [L] \setminus [K]$, denote the optimal items $\Delta$-close to $e$ by $\mc{K}_\Delta(e) = \{ e^* \in [K] : \Delta_{e,e^*} \leq \Delta \}$ and those $\Delta$-far from $e$ by $\bar{\mc{K}}_\Delta(e) = [K] \setminus \mc{K}_\Delta$. Let $\mc{E}_t \in \sigma(\mc{H}_t \cup \{ \mbf{A}_t \})$ and $\mc{F}_t \in \mc{F}$ for each $t \in [n]$, and define
\begin{gather}
R_1(n)  = \E \left[ \sum_{t=1}^n \sum_{e = K+1}^L  \sum_{e^*\in \bar{\mc{K}}_\Delta(e)} \Delta_{e,e^*} \ind ( \bar{\mc{E}}_t , \bar{\mc{F}}_t  , G_{e,e^*,t} ) \right] , \\
R_2(n)  = \sum_{t=1}^n \P ( \mc{E}_t ) , \quad R_3(n)  = K \E \left[ \sum_{t=1}^n \ind ( \mc{F}_t ) \right]  .
\end{gather}
Then under Assumption \ref{assOrder}, $R(n) \leq \Delta n \min \{ 2 / \bar{w}(K) , K \} + \sum_{j=1}^3 R_j(n)$.
\end{lem}
\begin{rem}
When applying Lemma \ref{lemGenUpper}, we will choose $\bar{\mc{E}}_t$ as a ``good event'' under which the empirical and true mean rewards are close at time $t$. Thus, $R_1$ is the regret incurred from suboptimal items $e$ that are chosen in favor of optimal items $e^* \in \bar{\mc{K}}_\Delta(e)$ with mean reward at least $\Delta$ higher, on this good event. The term $R_2$ accounts for the failure of the good events, which will occur with low probability due to concentration. The term $R_3$ is needed for the linear case, where $\mc{F}_t \notin \sigma(\mc{H}_t \cup \{ \mbf{A}_t \})$ may occur for our choice of $\mc{F}_t$ (see the beginning of Appendix \ref{appLinear} for details).\footnote{In such cases, $\mc{F}_t$ cannot be absorbed into $\mc{E}_t$, which leads to an additional factor of $K$ in $R_3$ compared to $R_2$.} The term $\Delta n \min \{ 2 / \bar{w}(K) , K \}$ accounts choosing $e$ instead of $e^* \in \mc{K}_\Delta(e)$ with gap $\Delta_{e,e^*} \leq \Delta$, which is the cascading analogue of the standard (non-cascading) bound $\Delta n$ discussed in the first paragraph of Section \ref{secIntuition}.
\end{rem}
\begin{proof}
By the choice of $\mc{E}_t$ and the same logic as Appendix A.1 of \cite{kveton2015cascading}, we know
\begin{equation} \label{eqGenUpper}
R(n) \leq \E \left[ \sum_{t=1}^n \sum_{e=K+1}^L \sum_{e^* = 1}^K  \Delta_{e,e^*} \ind ( \bar{\mc{E}}_t , G_{e,e^*,t} ) \right] + R_2(n) .
\end{equation}
For the first term, we fix $t$ and $e$ and separately analyze regret due to $e^* \in \mc{K}_\Delta(e)$ and $e^* \in \bar{\mc{K}}_\Delta(e)$. 

For $\mc{K}_\Delta(e)$, we consider two cases. For the first case, suppose $\bar{w}(K) / 2 \geq \max \{ \bar{w}(e) , 1/K \}$. Then for any $e^* \in [K]$, we know that
\begin{equation*}
\Delta_{e,e^*} = \bar{w}(e^*) - \bar{w}(e) \geq \bar{w}(e^*) - \bar{w}(K)/2 \geq \bar{w}(K)/2 = \max \{ \bar{w}(K)/2 , 1/K \} > \Delta ,
\end{equation*}
where the inequalities hold by $\bar{w}(K)/2 \geq \bar{w}(e)$, Assumption \ref{assOrder}, and the assumption on $\Delta$, respectively. Therefore, $\mc{K}_\Delta(e) = \emptyset$ by definition in the first case. For the second case, suppose $\bar{w}(K) / 2 < \max \{ \bar{w}(e) , 1/K \}$. Then by definition of $\mc{K}_\Delta(e)$ and \eqref{eqNumberOfG}, 
\begin{equation} \label{eqGenUpperSmallW}
\sum_{e^* \in \mc{K}_\Delta(e)} \Delta_{e,e^*} \ind ( \bar{\mc{E}}_t , G_{e,e^*,t} ) \leq \Delta \sum_{e^* \in \mc{K}_\Delta(e)} \ind ( G_{e,e^*,t} ) \leq \Delta ( \mbf{T}_t(e) - \mbf{T}_{t-1}(e) ) .
\end{equation}
Combining the cases, we conclude that
\begin{equation} \label{eqDelIndSmall}
\sum_{e^* \in \mc{K}_\Delta(e)} \Delta_{e,e^*} \ind ( \bar{\mc{E}}_t , G_{e,e^*,t} ) \leq \Delta ( \mbf{T}_t(e) - \mbf{T}_{t-1}(e) ) \ind ( \bar{w}(K) < 2 \max \{ \bar{w}(e) , 1 / K \} ) .
\end{equation}
For $\bar{\mc{K}}_\Delta(e)$, we use $\bar{\mc{E}}_t \cap G_{e,e^*,t} \subset ( \bar{\mc{E}}_t \cap \bar{\mc{F}}_t \cap G_{e,e^*,t} )  \cup ( \mc{F}_t \cap G_{e,e^*,t} )$ and $\Delta_{e,e^*} \leq 1$ to write
\begin{equation} 
\Delta_{e,e^*} \ind ( \bar{\mc{E}}_t , G_{e,e^*,t} ) \leq  \Delta_{e,e^*} \ind ( \bar{\mc{E}}_t , \bar{\mc{F}}_t , G_{e,e^*,t} ) + \ind ( \mc{F}_t ,G_{e,e^*,t} ) .
\end{equation}
Summing this bound over $e^* \in \bar{\mc{K}}_\Delta(e)$ and again using \eqref{eqNumberOfG}, we obtain
\begin{equation} \label{eqDelIndLarge}
\sum_{ e^* \in \bar{\mc{K}}_\Delta(e) } \Delta_{e,e^*} \ind ( \bar{\mc{E}}_t , G_{e,e^*,t} ) \leq \sum_{ e^* \in \bar{\mc{K}}_\Delta(e) } \Delta_{e,e^*} \ind ( \bar{\mc{E}}_t , \bar{\mc{F}}_t , G_{e,e^*,t} ) + \ind ( \mc{F}_t )  ( \mbf{T}_t(e) - \mbf{T}_{t-1}(e) ) .
\end{equation}
Having separately analyzed $\mc{K}_\Delta(e)$ and $\bar{\mc{K}}_\Delta(e)$ for fixed $e \in [L] \setminus [K]$, we combine \eqref{eqDelIndSmall} and \eqref{eqDelIndLarge}, use $[K] = \mc{K}_\Delta(e) \cup \bar{\mc{K}}_\Delta(e)$ (by definition), sum over $e \in [L] \setminus [K]$, and use \eqref{eqNumberOfG}-\eqref{eqSumOfG} to write
\begin{align*}
\sum_{e=K+1}^L \sum_{e^*=1}^K \Delta_{e,e^*} \ind ( \bar{\mc{E}}_t , G_{e,e^*,t} ) & \leq \Delta \sum_{e = K+1}^L ( \mbf{T}_t(e) - \mbf{T}_{t-1}(e) ) \ind ( \bar{w}(K) < 2 \max \{ \bar{w}(e) , 1/K \} ) \\
& \quad  
+ \sum_{e=K+1}^L \sum_{ e^* \in \bar{\mc{K}}_\Delta(e) } \Delta_{e,e^*} \ind ( \bar{\mc{E}}_t , \bar{\mc{F}}_t , G_{e,e^*,t} ) + \ind ( \mc{F}_t )  K .
\end{align*}
Summing over $t \in [n]$ and taking expectation, the last two terms become $R_1(n)$ and $R_3(n)$, respectively. Plugging into \eqref{eqGenUpper}, we have therefore shown
\begin{align*}
R(n) \leq \Delta \sum_{t=1}^n \E \left[ \sum_{e = K+1}^L ( \mbf{T}_t(e) - \mbf{T}_{t-1}(e) ) \ind ( \bar{w}(K) < 2 \max \{ \bar{w}(e) , 1/K \} ) \right] + \sum_{j=1}^3 R_j(n) .
\end{align*}
Thus, it suffices to upper bound the remaining expectation term by $\min \{ 2 / \bar{w}(K) , K \}$. Toward this end, first note that if $\bar{w}(K) < 2/K$, then we can (naively) upper bound the indicators by $1$ and use \eqref{eqNumberOfG}-\eqref{eqSumOfG} to bound this term by $K= \min \{ 2 / \bar{w}(K) , K \}$, which completes the proof. Therefore, we assume for the remainder that $\bar{w}(K) \geq 2/K$. Under this assumption, consider any $e \in [L] \setminus [K]$ for which $\bar{w}(K) < 2 \max \{ \bar{w}(e) , 1/K \}$ (i.e., for which the indicator is $1$). Then we must have $\bar{w}(e) > \bar{w}(K)/2$, because if instead $\bar{w}(e) \leq \bar{w}(K)/2$, we obtain a contradiction:
\begin{equation}
\bar{w}(K) < 2 \max \{ \bar{w}(e) , 1 / K \} \leq 2 \max \{ \bar{w}(K) / 2 , 1/ K \} = \bar{w}(K) .
\end{equation}
In summary, if $e \in [L] \setminus [K]$ satisfies $\bar{w}(K) < 2 \max \{ \bar{w}(e) , 1/K \}$ (i.e., if its indicator is $1$), then it also satisfies $\bar{w}(e) > \bar{w}(K)/2$, and therefore $e \in \mc{L} \triangleq \{ e' \in [L] : \bar{w}(e') \geq \bar{w}(K)/2 \}$. Hence,
\begin{equation*}
\sum_{e = K+1}^L ( \mbf{T}_t(e) - \mbf{T}_{t-1}(e) ) \ind ( \bar{w}(K) < 2 \max \{ \bar{w}(e) , 1/K \} ) \leq \sum_{e \in \mc{L}} ( \mbf{T}_t(e) - \mbf{T}_{t-1}(e) ) \triangleq \mbf{O}_t .
\end{equation*}
Next, we let $\mbf{B}_t = ( \mbf{b}_1^t , \ldots , \mbf{b}_{|\mbf{B}_t|}^t )$ be the sublist of $\mbf{A}_t$ which contains only items from $\mc{L}$, with the same relative ordering as $\mbf{A}_t$. Notice that $\mbf{O}_t$ is the number of $e \in \mc{L}$ whose reward $\mbf{w}_t(e)$ was observed at time $t$. Therefore, we can rewrite and then upper bound this quantity as follows:
\begin{equation} 
\mbf{O}_t = \sum_{k=1}^K \ind ( \mbf{a}_k^t \in \mc{L} ) \prod_{i=1}^{k-1} ( 1 - \mbf{w}_t ( \mbf{a}_i^t ) ) \leq \sum_{k=1}^{|\mbf{B}_t|} \prod_{i=1}^{k-1} ( 1 - \mbf{w}_t(\mbf{b}_i^t) ) .
\end{equation}
By the conditional independence assumption (see Section \ref{secModel}) and the definition of $\mc{L}$, we then obtain
\begin{equation*}
\E_t [ \mbf{O}_t ] \leq \sum_{k=1}^{|\mbf{B}_t|} \prod_{i=1}^{k-1} (1- \bar{w}(\mbf{b}_i^t) ) \leq \sum_{k=1}^{|\mbf{B}_t|}  (1-\bar{w}(K)/2)^{k-1} \leq \sum_{k=0}^\infty ( 1 - \bar{w}(K)/2)^k = 2 / \bar{w}(K) .
\end{equation*}
Taking expectation shows $\E[\mbf{O}_t] \leq 2 / \bar{w}(K) = \min \{ 2 / \bar{w}(K) , K \}$, which completes the proof.
\end{proof}

\section{Proof of Theorem \ref{thmKl}} \label{appKlProof}

We prove the bound for \texttt{CascadeKL-UCB} and discuss how to modify the proof for \texttt{CascadeUCB-V} in Appendix \ref{appUcbV}. Throughout, we adopt the notation of Appendix \ref{appGeneral} and assume the following.
\begin{ass} \label{assKl}
We have $\bar{w}(1) \geq \cdots \geq \bar{w}(L)$ and $n > L$.
\end{ass}
As in Appendix \ref{appGeneral}, the first inequality is without loss of generality. Thus, Assumption \ref{assKl} is \textit{with} loss of generality only when the second inequality fails, i.e., only when  $n \leq L$. But in this case, we can naively bound regret by $n \leq \sqrt{nL}$, so the conclusion of Theorem \ref{thmKl} is immediate.

To begin the proof, we will invoke Lemma \ref{lemGenUpper} with a particular choice of $\Delta$, $\mc{E}_t$, and $\mc{F}_t$. First, we choose $\Delta = \max \{ \bar{w}(K) / 2 , 1/K \} \sqrt{L/n}$. Next, for each $e \in [L] \setminus [K]$ and $e^* \in \bar{\mc{K}}_\Delta(e)$, we let
\begin{equation*}
\tau_{e,e^*} = 8 \bar{w}(e^*) \max \{ \log(n^2 K) , \log ( f ( n ) ) \} / \Delta_{e,e^*}^2 .
\end{equation*}
Note that $\Delta_{e,e^*} \geq \Delta > 0$ (by $e^* \in \bar{\mc{K}}_\Delta(e)$ and choice of $\Delta$, respectively), so this quantity is well-defined. Intuitively, it will upper bound the number of plays needed to distinguish the suboptimal item $e$ from the optimal item $e^*$.\footnote{This bound is likely loose for some problem instances, but choosing a sharper $\tau_{e,e^*}$ only improves log terms in the ultimate regret bound (which holds for all problem instances).} Finally, for each $t \in [n]$, we define
\begin{gather*}
\mc{E}_t = \cup_{e=K+1}^L \cup_{e^* \in \bar{\mc{K}}_\Delta(e)} \{ \mbf{U}_t(e^*) \leq \mbf{U}_t(e) , \mbf{T}_{t-1}(e) \geq \tau_{e,e^*} \} ,  \quad \mc{F}_t = \emptyset .
\end{gather*}
Observe that $\Delta \in (0, \max \{ \bar{w}(K)/2 , 1/K \})$ (by Assumption \ref{assKl}), $\mc{E}_t \in \sigma(\mc{H}_t \cup \{\mbf{A}_t\})$, and $\mc{F}_t \in \mc{F}$, as required by Lemma \ref{lemGenUpper}. Therefore, we can specialize the bound from that lemma to obtain
\begin{equation*}
R(n) \leq \Delta n \min \{ 2 / \bar{w}(K) , K \} + \sum_{j=1}^3 R_j(n) = \sqrt{nL} + \sum_{j=1}^2 R_j(n) .
\end{equation*}
In light of this inequality and Assumption \ref{assKl} (the latter of which implies $K \leq L \leq \sqrt{nL}$), it suffices to show $R_j(n) = \tilde{O}(\sqrt{nL} + L + K)$ for each $j \in [2]$. We do so in the next two lemmas.

\begin{lem} \label{lemKlPlays}
Under Assumption \ref{assKl}, $R_1(n) \leq 24 \sqrt{nL}  \max \{ \log (n^2 K) , \log (f(n)) \} ( \log(n)+1 ) + L$.
\end{lem}
\begin{proof}
If $n=K=1$, then $\Delta_{e,e^*} \leq 1$ and \eqref{eqSumOfG} imply $R_1(n) \leq n K = 1$, so the bound is immediate. Hence, we assume $\max\{n,K\} > 1$. Fix $e \in [L] \setminus [K]$. Then for any $e^* \in \bar{\mc{K}}_\Delta(e)$, we know
\begin{align*}
\bar{\mc{E}}_t \subset \{ \mbf{U}_t(e^*) > \mbf{U}_t(e) \} \cup \{ \mbf{T}_{t-1}(e) < \tau_{e,e^*} \} .
\end{align*}
Combined with \eqref{eqGtoUCB} and the choice $\mc{F}_t = \emptyset$, we conclude
\begin{align*}
\bar{\mc{E}}_t \cap \bar{\mc{F}}_t  \cap G_{e,e^*,t} & \subset ( \{ \mbf{U}_t(e^*) > \mbf{U}_t(e) \} \cup \{ \mbf{T}_{t-1}(e) < \tau_{e,e^*} \} ) \cap \Omega \cap G_{e,e^*,t} \\
& = \{ \mbf{T}_{t-1}(e) < \tau_{e,e^*}  \} \cap G_{e,e^*,t} .
\end{align*}
Therefore, the regret that $e$ contributes to $R_1$ is at most (the expected value of)
\begin{equation} \label{eqKlPlays}
\sum_{t=1}^n \sum_{e^*\in \bar{\mc{K}}_\Delta(e)} \Delta_{e,e^*} \ind ( \mbf{T}_{t-1}(e) < \tau_{e,e^*} ) \ind ( G_{e,e^*,t} ) .
\end{equation}
Next, let $C = 8 \max \{ \log(n^2 K) , \log( f(n) ) \} > 0$ denote the term in the definition of $\tau_{e,e^*}$ that does not depend on the $(e,e^*)$ pair. Then for any $e^* \in \bar{\mc{K}}_\Delta(e)$, we can write
\begin{equation} \label{eqKlPlaysDeltaTau}
\frac{\Delta_{e,e^*} \tau_{e,e^*}}{C}  = \frac{ \bar{w}(e^*)}{\Delta_{e,e^*}} = \frac{\bar{w}(e) }{ \Delta_{e,e^*} } + 1 \leq \frac{ \bar{w}(K) }{ \Delta_{e,e^*} } + 1 \leq  2 \sqrt{ \frac{n}{L} } + 1 \leq 3 \sqrt{ \frac{n}{L} } ,
\end{equation}
where the first and third inequalities hold by Assumption \ref{assKl}, the second inequality uses $\Delta_{e,e^*} > \Delta \geq ( \bar{w}(K) / 2 ) \sqrt{L/n}$ (by definition of $\bar{\mc{K}}_\Delta(e)$ and $\Delta$, respectively), and the equalities hold by definition. Next, using the bound $\ind(x < y) \leq y/x$ for $x,y > 0$, we write
\begin{equation*}
\ind ( \mbf{T}_{t-1}(e) < \tau_{e,e^*} ) \ind ( \mbf{T}_{t-1}(e) > 0 ) \leq \tau_{e,e^*} \ind ( \mbf{T}_{t-1}(e) > 0 ) / \mbf{T}_{t-1}(e) ,
\end{equation*}
where by convention $\ind ( \mbf{T}_{t-1}(e) > 0 ) / \mbf{T}_{t-1}(e) = 0$ when $\mbf{T}_{t-1}(e)=0$. Finally, note that for any $[0,1]$-valued random variable $\mbf{X}$, we have
\begin{equation*}
\mbf{X} = \mbf{X} \ind ( \mbf{T}_{t-1}(e) = 0 ) + \mbf{X} \ind ( \mbf{T}_{t-1}(e) > 0 ) \leq \ind ( \mbf{T}_{t-1}(e) = 0 ) + \mbf{X} \ind ( \mbf{T}_{t-1}(e) > 0 ) .
\end{equation*}
Combining the previous three inequalities (with $\mbf{X} = \Delta_{e,e^*} \ind ( \mbf{T}_{t-1}(e) < \tau_{e,e^*} )$ in the third) gives
\begin{align*}
\Delta_{e,e^*} \ind ( \mbf{T}_{t-1}(e) < \tau_{e,e^*} ) & \leq \ind ( \mbf{T}_{t-1}(e) = 0 ) + \Delta_{e,e^*} \ind ( \mbf{T}_{t-1}(e) < \tau_{e,e^*} ) \ind ( \mbf{T}_{t-1}(e) > 0 ) \\
& \leq \ind ( \mbf{T}_{t-1}(e) = 0 ) + \Delta_{e,e^*} \tau_{e,e^*} \ind ( \mbf{T}_{t-1}(e) > 0 ) / \mbf{T}_{t-1}(e) \\
& \leq \ind ( \mbf{T}_{t-1}(e) = 0 ) + 3C \sqrt{ n/L } \ind ( \mbf{T}_{t-1}(e) > 0 ) / \mbf{T}_{t-1}(e)  \triangleq \mbf{V}_t(e) .
\end{align*}
Using this bound and \eqref{eqNumberOfG}, we conclude that \eqref{eqKlPlays} is upper bounded by
\begin{align*}
& \sum_{t=1}^n \mbf{V}_t(e) \sum_{e^* \in \bar{\mc{K}}_\Delta(e)}  \ind ( G_{e,e^*,t} ) \leq \sum_{t=1}^n \mbf{V}_t(e)  \ind ( \mbf{T}_t(e) = \mbf{T}_{t-1}(e) + 1 ) \\
& \quad = \sum_{t=1}^n \ind ( \mbf{T}_t(e) = 1 , \mbf{T}_{t-1}(e) = 0 ) + 3 C \sqrt{\frac{n}{L}} \sum_{t=1}^n \frac{ \ind ( \mbf{T}_t(e) = \mbf{T}_{t-1}(e)+1 , \mbf{T}_{t-1}(e) > 0 ) }{\mbf{T}_{t-1}(e)} .
\end{align*}
Therefore, because $\mbf{T}_t(e) = 1 , \mbf{T}_{t-1}(e)= 0$ can only occur for one $t \in [n]$, and because
\begin{equation*}
\sum_{t=1}^n \frac{ \ind (  \mbf{T}_t(e) = \mbf{T}_{t-1}(e) +1 , \mbf{T}_{t-1}(e) > 0 ) }{ \mbf{T}_{t-1}(e) }  = \sum_{s=1}^{ \mbf{T}_{n-1}(e) } \frac{1}{s} \leq \sum_{s=1}^n \frac{1}{s} \leq \log(n)+1 ,
\end{equation*}
we conclude that \eqref{eqKlPlays} (and thus the regret that $e$ contributes to $R_1$) is upper bounded by
\begin{equation*}
1 + 3 C \sqrt{ n/L } ( \log(n)+1) = 1 + 24 \sqrt{n/L} \max \{ \log(n^2 K) , \log ( f(n) ) \}  ( \log(n) + 1 ) .
\end{equation*}
Finally, summing this bound over $e \in [L] \setminus [K]$ completes the proof.
\end{proof}

\begin{rem}
The bound in Lemma \ref{lemKlPlays} can be sharpened in terms of constants and log terms using ideas from \cite{kveton2015cascading}, but this requires a more complicated proof. We opted for the simpler proof above because its ideas generalize more easily to the linear case.
\end{rem}

\begin{lem} \label{lemKlTailSub}
Under Assumption \ref{assKl}, $R_2(n) \leq L + 7 K \log \log n$.
\end{lem}

To prove this lemma, we require the following claim. 
\begin{clm} \label{clmKlLower}
For any $p,q \in (0,1)$, $d(p,q) \geq (p-q)^2 / ( 2 \max \{p,q\} )$.
\end{clm}
\begin{proof}
We prove the bound assuming $p < q$ (it follows analogously for $p > q$ and is immediate when $p=q$). For each $\Delta \in (-q,1-q)$, we define $g(\Delta) = d(q+\Delta,q)$. Then one can compute
\begin{equation*}
g'(\Delta) = \log \left( \frac{q+\Delta}{q} \right) - \log \left( \frac{1-q-\Delta}{1-q} \right) , \quad g''(\Delta) = \frac{1}{(q+\Delta)(1-q-\Delta)} \geq \frac{1}{q+\Delta} .
\end{equation*}
Thus, $g(0) = g'(0) = 0$, so by Taylor's theorem with remainder, we can find $\Delta \in (p-q,0)$ such that
\begin{equation*}
d(p,q) = g(p-q) = \frac{g''(\Delta) (p-q)^2}{2} \geq \frac{(p-q)^2}{2(q+\Delta)} \geq \frac{(p-q)^2}{2 q} = \frac{(p-q)^2}{2 \max \{p,q\}} . \qedhere
\end{equation*}  
\end{proof}

\begin{proof}[Proof of Lemma \ref{lemKlTailSub}]
Let $\mc{C}_t = \cup_{e^* = 1}^K \{ \mbf{U}_t(e^*) < \bar{w}(e^*) \}$. Then by definition,
\begin{equation*}
\mc{E}_t \subset ( \mc{E}_t \cap \bar{\mc{C}}_t ) \cup \mc{C}_t \subset \left( \cup_{e=K+1}^L \cup_{e^* \in \bar{\mc{K}}_\Delta(e)} \{ \mbf{U}_t(e) \geq \bar{w}(e^*) , \mbf{T}_{t-1}(e) \geq \tau_{e,e^*} \} \right) \cup \mc{C}_t .
\end{equation*}
Therefore, by the union bound, we obtain that
\begin{equation} \label{eqKlTailSub}
R_2(n) \leq \sum_{t=1}^n  \sum_{e=K+1}^L \sum_{ e^* \in \bar{\mc{K}}_\Delta(e)} \P ( \mbf{U}_t(e) \geq \bar{w}(e^*) , \mbf{T}_{t-1}(e) \geq \tau_{e,e^*} ) + \sum_{t=1}^n \P ( \mc{C}_t ) .
\end{equation}
As shown in Appendix A.2 of \cite{kveton2015cascading}, the second term is at most $7 K \log \log n$. Hence, it suffices to bound the first term by $L$. Toward this end, first note
\begin{align} \label{eqBeforeFixS}
\P (  \mbf{U}_t(e) \geq \bar{w}(e^*) , \mbf{T}_{t-1}(e) \geq \tau_{e,e^*} ) \leq \sum_{s = \ceil{\tau_{e,e^*}}}^{t-1} \P (  \mbf{U}_t(e) \geq \bar{w}(e^*)  , \mbf{T}_{t-1}(e) = s ) .
\end{align}
Now fix $s \in \{ \ceil{\tau_{e,e^*}}, \ldots,t-1\}$. We claim
\begin{equation}\label{eqFixSsubset}
\{ \mbf{U}_t(e) \geq \bar{w}(e^*) , \mbf{T}_{t-1}(e) = s \} \subset \{ \hat{\mbf{w}}_s(e) \geq \bar{w}(e) + \Delta_{e,e^*} / 2 \} .
\end{equation}
Suppose instead that $\mbf{U}_t(e) \geq \bar{w}(e^*)$, $\mbf{T}_{t-1}(e) = s$, and $\hat{\mbf{w}}_s(e) < \bar{w}(e)+\Delta_{e,e^*}/2$ all occur. Then
\begin{equation} \label{eqFixSzeroCont}
d ( \hat{\mathbf{w}}_s(e) , \mathbf{U}_t(e) ) = d ( \hat{\mathbf{w}}_{\mathbf{T}_{t-1}(e)}(e) , \mathbf{U}_t(e) )  \leq \frac{ \log f(t) }{ \mathbf{T}_{t-1}(e) } = \frac{ \log f(t) }{ s} \leq \frac{ \log f(n) }{ \tau_{e,e^*}} \leq \frac{ \Delta_{e,e^*}^2 }{8 \bar{w}(e^*)} ,
\end{equation}
where we used $\mbf{T}_{t-1}(e) = s \geq \tau_{e,e^*}$ and the definitions of $\mbf{U}_t(e)$ and $\tau_{e,e^*}$. But by Lemma 10.2(c) of \cite{lattimore2020bandit} and Claim \ref{clmKlLower}, $\mbf{U}_t(e) \geq \bar{w}(e^*) > \bar{w}(e)+\Delta_{e,e^*}/2 > \hat{\mbf{w}}_s(e)$ implies
\begin{equation*}
d ( \hat{\mathbf{w}}_s(e) , \mathbf{U}_t(e) ) \geq d ( \hat{\mbf{w}}_s(e) , \bar{w}(e^*) ) > d ( \bar{w}(e) + \Delta_{e,e^*}/2 , \bar{w}(e^*) ) \geq \frac{  ( \Delta_{e,e^*} /2 )^2 }{ 2 \bar{w}(e^*) } = \frac{ \Delta_{e,e^*}^2 }{8 \bar{w}(e^*)} ,
\end{equation*}
with contradicts \eqref{eqFixSzeroCont}. This completes the proof of \eqref{eqFixSsubset}. We then write
\begin{align*}
& \P( \mbf{U}_t(e) \geq \bar{w}(e^*) , \mbf{T}_{t-1}(e) = s ) \\
& \quad \leq\P ( \hat{\mbf{w}}_s(e) \geq \bar{w}(e) + \Delta_{e,e^*} / 2 )  \leq \exp ( -  s d \left( \bar{w}(e) + \Delta_{e,e^*} / 2, \bar{w}(e) \right) ) \\
& \quad \leq \exp ( -  s \Delta_{e,e^*}^2 / ( 8 \bar{w}(e^*) ) ) \leq  \exp ( - \tau_{e,e^*} \Delta_{e,e^*}^2 / ( 8 \bar{w}(e^*) ) )  \leq 1 / ( n^2 K ) ,
\end{align*}
where we used \eqref{eqFixSsubset}, the Chernoff bound, Claim \ref{clmKlLower}, the choice $s \geq \tau_{e,e^*}$, and the definition of $\tau_{e,e^*}$, respectively. Since $t \leq n$, this implies \eqref{eqBeforeFixS} is upper bounded by $1/(nK)$. Summing over $t \in [n]$, $e \in [L] \setminus [K]$ and $e^* \in \bar{\mc{K}}_\Delta(e)$ shows the first term in \eqref{eqKlTailSub} is upper bounded by $L$, as desired.
\end{proof}

\subsection{Analysis for \texttt{CascadeUCB-V}} \label{appUcbV}

For \texttt{CascadeUCB-V}, we invoke Lemma \ref{lemGenUpper} with the same choice of $\Delta$, $\mc{E}_t$, and $\mc{F}_t$, but we change the definition of $\tau_{e,e^*}$ to $\tau_{e,e^*} = C \bar{w}(e^*) / \Delta_{e,e^*}^2$, where here $C = 30 \log ( n K )$. As above, it suffices to show $R_j(n) = \tilde{O}(\sqrt{nL}+L+K)$ for each $j \in [2]$. For $j=1$, this follows as in the proof of Lemma \ref{lemKlPlays} (with a different $C$). For $j=2$, by the same logic used to prove Lemma \ref{lemKlTailSub}, it suffices to show
\begin{gather}
\sum_{t=1}^n \sum_{e^*=1}^K \P ( \mbf{U}_t(e^*) < \bar{w}(e^*) ) =  \tilde{O}(\sqrt{nL}+L+K) , \label{eqUcbVsts1} \\
\P ( \mbf{U}_t(e) \geq \bar{w}(e^*) , \mbf{T}_{t-1}(e) = s ) \leq \frac{1}{n^2 K}\  \forall\ e \in [L] \setminus [K] , e^* \in \bar{\mc{K}}_\Delta(e) , \tau_{e,e^*} \leq s \leq t \leq n . \quad  \label{eqUcbVsts2}
\end{gather}
To prove \eqref{eqUcbVsts1}, we note that for any $t \in [n]$ and $e^* \in [K]$,
\begin{align}
\P( \mbf{U}_t(e^*) < \bar{w}(e^*) ) & = \sum_{s=1}^{t-1} \P( \mbf{U}_t(e^*) < \bar{w}(e^*) , \mbf{T}_{t-1}(e) = s ) \\
& \leq \sum_{s=1}^{t-1} \P \left( \hat{\mbf{w}}_s(e^*) + \sqrt{ \frac{4 \hat{\mbf{v}}_s(e^*) \log t }{s} } + \frac{ 6 \log t }{s} < \bar{w}(e^*) \right) < 3 t^{-1} ,
\end{align}
where the last bound holds because each summand is at most $3 t^{-2}$ by Theorem 1 of \cite{audibert2009exploration}. Thus, the left side of \eqref{eqUcbVsts1} is at most $3 K \sum_{t=1}^n t^{-1} = \tilde{O}(K)$, as desired. 

To prove \eqref{eqUcbVsts2}, fix such an $e$, $e^*$, $s$, and $t$, and suppose $\mbf{U}_t(e) \geq \bar{w}(e^*)$ and $\mbf{T}_{t-1}(e) = s$. Then by definition of $\mbf{U}_t(e)$ and $\hat{\mbf{v}}_s(e)$ (see  \eqref{eqUcbV} and ensuing discussion), we have
\begin{align} \label{eqUcbVLargeUcb}
\hat{\mbf{w}}_s(e) + \sqrt{ \frac{ 4 \hat{\mbf{w}}_s(e) \log t}{s} } + \frac{ 6 \log t }{s} \geq \mbf{U}_t(e) \geq \bar{w}(e^*) = \bar{w}(e) + \Delta_{e,e^*} .
\end{align}
Next, we consider two cases. First, if $\bar{w}(e) = 0$, then $\hat{\mbf{w}}_s(e) = 0$ and $\bar{w}(e^*) = \Delta_{e,e^*}$, so by \eqref{eqUcbVLargeUcb},
\begin{equation}
s \leq 6 \log(t) / \Delta_{e,e^*} = 6 \log (t) \bar{w}(e^*) / \Delta_{e,e^*}^2 < \tau_{e,e^*} ,
\end{equation}
which contradicts the choice of $s$ and implies \eqref{eqUcbVsts2}. The second case is $\bar{w}(e) > 0$. Here we observe that when \eqref{eqUcbVLargeUcb} holds, one of the following inequalities must also hold:
\begin{equation} \label{eqUcbVthreeIneq}
\hat{\mbf{w}}_s(e) \geq \bar{w}(e) + \frac{ 2 \Delta_{e,e^*} }{5} , \quad \sqrt{ \frac{4 \hat{\mbf{w}}_s(e) \log t}{s} } \geq \frac{ 2\Delta_{e,e^*} }{5} , \quad  \frac{6 \log t }{s} \geq \frac{ \Delta_{e,e^*} }{5} .
\end{equation}
However, the third inequality cannot hold, because it and $\bar{w}(e^*) > \Delta_{e,e^*}$ (where the latter inequality is by definition in the case $\bar{w}(e) > 0$) together imply
\begin{equation*}
s \leq 30 \log (t) / \Delta_{e,e^*} < 30 \bar{w}(e^*) \log(t) / \Delta_{e,e^*}^2 \leq 30 \bar{w}(e^*) \log(nK) / \Delta_{e,e^*}^2 = \tau_{e,e^*}  ,
\end{equation*}
which contradicts the choice of $s$. Furthermore, if the second inequality holds, then
\begin{equation*}
\hat{\mbf{w}}_s(e) \geq \frac{ s \Delta_{e,e^*}^2 }{ 25 \log t } \geq \frac{ \tau_{e,e^*} \Delta_{e,e^*}^2 }{ 25 \log t } \geq \bar{w}(e^*) > \bar{w}(e) + \frac{2 \Delta_{e,e^*} }{5} ,
\end{equation*}
so the first must also hold. Therefore, we have shown that if $\mbf{U}_t(e) \geq \bar{w}(e^*)$ and $\mbf{T}_{t-1}(e) = s$ both occur, then the first inequality in \eqref{eqUcbVthreeIneq} must hold. It follows that
\begin{align*}
\P ( \mbf{U}_t(e) \geq \bar{w}(e^*) , \mbf{T}_{t-1}(e) = s ) & \leq \P ( \hat{\mbf{w}}_s(e) \geq \bar{w}(e) + 2 \Delta_{e,e^*} / 5 ) \\
& \leq \exp ( - 2 \tau_{e,e^*} \Delta_{e,e^*}^2 / ( 25 \bar{w}(e^*) ) ) \leq ( n K )^{-2} ,
\end{align*}
where the second inequality uses the Chernoff bound and Claim \ref{clmKlLower} similar to the proof of Lemma \ref{lemKlTailSub}. This is tighter than the bound in \eqref{eqUcbVsts2} that we set out to prove so it concludes the analysis.

\section{Proof of Theorem \ref{thmLinear}} \label{appLinear}

As in Appendix \ref{appKlProof}, we adopt the notation of Appendix \ref{appGeneral}. We will also assume the following.
\begin{ass} \label{assLinear}
There exists $\phi : [L] \rightarrow \mathscr{B}_d(1)$ and $\theta \in \mathscr{B}_d(1)$ such that $\bar{w}(e) = \ip{\theta}{\phi(e)}$ for each $e \in [L]$. Furthermore, we have $\bar{w}(1) \geq \cdots \geq \bar{w}(L)$ and $n > d (d+K)$.
\end{ass}
Analogous to the discussion of Assumption \ref{assKl} in Appendix \ref{appKlProof}, Assumption \ref{assLinear} is stronger than the assumptions of Theorem \ref{thmLinear} only when $n \leq d (d+K)$. But in this case, we have
\begin{equation} \label{eqLinearSmallN}
R(n) \leq n \leq \sqrt{ n d ( d + K ) } \leq \sqrt{ n \max \{ d , K \} ( 2 \max \{ d , K \} ) } = \sqrt{2n} \max \{ d , K \} ,
\end{equation}
so $R(n) = O ( \min \{ \sqrt{ n d ( d + K ) } , \sqrt{n} \max \{ d , K \} \} )$, which is sharper than the desired bound.

Next, similar to Appendix \ref{appKlProof}, we will invoke Lemma \ref{lemGenUpper}. First, we set 
\begin{equation} \label{eqDeltaLin}
\Delta = \max \{ \bar{w}(K) / 2 , 1/ K \} \sqrt{d(d+K)/n} \in ( 0 , \max \{ \bar{w}(K) / 2 , 1/ K \} ) ,
\end{equation}
where the upper bound holds by Assumption \ref{assLinear}. Next, we let $\eta_s^k = \mbf{w}_s ( \mbf{a}_k^s ) - \bar{w}(\mbf{a}_k^s)$ denote the $k$-th noise random variable at time $s$. We define the events
\begin{gather*}
\mc{E}_{t,H} = \cup_{\tau=1}^t \left\{ \left\| \sum_{s=1}^{\tau-1} \sum_{k=1}^{ \min \{ \mbf{C}_s , K \} } \phi(\mbf{a}_k^s) \eta_s^k \right\|_{ \mbf{\Lambda}_{\tau,H}^{-1} } > \beta_{\tau,H} \right\}  ,\\
\mc{E}_{t,B} = \cup_{\tau=1}^t \left\{ \left\| \sum_{s=1}^{\tau-1} \sum_{k=1}^{ \min \{ \mbf{C}_s , K \} } \phi ( \mbf{a}_k^s ) \eta_s^k / \mbf{U}_{s,H}(\mbf{a}_k^s) \right\|_{ \mbf{\Lambda}_{\tau,B}^{-1} } > \beta_{\tau,B} \right\} , \\
\mc{F}_{t,H} = \{ \det ( \mbf{\Lambda}_{t+1,H} ) \geq 2 \det ( \mbf{\Lambda}_{t,H} ) \} , \quad \mc{F}_{t,B} = \{ \det ( \mbf{\Lambda}_{t+1,B} ) \geq 2 \det ( \mbf{\Lambda}_{t,B} ) \} ,
\end{gather*}
where, with $\alpha_{t,H}$ and $\alpha_{t,B}$ defined as in the theorem statement, we let
\begin{gather}
\beta_{t,H} = \alpha_{t,H} - 1 = \sqrt{ d \log ( 1 + t K / d ) + 2 \log (n) } , \\
\beta_{t,B} = \alpha_{t,B} - \sqrt{K} = 8 \sqrt{ d \log ( 1 + t K / d ) \log (n^3 K) } +  4 \sqrt{K}  \log ( n^3 K ) .
\end{gather}
Finally, we set $\mc{E}_t = \mc{E}_{t,H} \cup \mc{E}_{t,B}$ and $\mc{F}_t = \mc{F}_{t,H} \cup \mc{F}_{t,B}$. Observe that $\mc{E}_t \in \sigma(\mc{H}_t \cup \{ \mbf{A}_t \} )$, though $\mc{F}_t$ need not belong to this sub-$\sigma$-algebra because it depends on the random rewards (i.e., the presence or absence of clicks) at time $t$. Therefore, we use Lemma \ref{lemGenUpper} to write
\begin{equation} \label{eqApplyGenLin}
R(n) \leq \Delta n \min \{ 2 / \bar{w}(K) , K \} + \sum_{j=1}^3 R_j(n) = \sqrt{ n d ( d + K ) } + \sum_{j=1}^3 R_j(n) .
\end{equation}
We claim that, in light of this bound, it suffices to show
\begin{equation} \label{eqLinearSts}
R_j(n) = \tilde{O} \left( \sqrt{ n d ( d + K ) } + n^{1/6} d^{7/6} ( d+K )^{1/2} K^{1/3} + d K \right)\ \forall\ j \in [3] .
\end{equation} 
Indeed, if \eqref{eqLinearSts} holds, then since $d K \leq d(d+K) \leq \sqrt{n d (d+K)}$ by Assumption \ref{assLinear}, \eqref{eqApplyGenLin} will imply
\begin{equation} \label{eqLinearStsFirstTerm}
R(n) = \tilde{O} \left( \sqrt{ n d ( d + K ) } + n^{1/6} d^{7/6} ( d+K )^{1/2} K^{1/3} \right) ,
\end{equation}
which is precisely the first term in the minimum of the Theorem \ref{thmLinear} regret bound. Therefore, we only need to show that each summand in \eqref{eqLinearStsFirstTerm} is bounded by a constant multiple of the second term in the minimum of Theorem \ref{thmLinear}. For the first summand, this follows as in the third inequality of \eqref{eqLinearSmallN}. For the second summand, we use Assumption \ref{assLinear} and $d+K \leq 2 \max \{ d, K \}$ to write
\begin{align*}
n^{1/6} d^{7/6} ( d+K )^{1/2} K^{1/3} \leq \sqrt{n} d^{5/6} ( d + K )^{1/6} K^{1/3} \leq 2^{1/6} \sqrt{n} \max \{ d , K \}^{1/6} d^{5/6} K^{1/3} ,
\end{align*}
which completes the proof by separately considering the cases the cases $d \geq K$ and $d < K$.

Hence, to prove the theorem, it only remains to prove \eqref{eqLinearSts}. Toward this end, we first state two claims. The first is accuracy guarantee for the least-squares estimates $\ip{\hat{\theta}_{t,\cdot}}{\phi(e)}$, which also asserts that $\mbf{U}_{t,\cdot}(e)$ is a valid UCB (on the good event $\bar{\mc{E}}_{t,\cdot}$). The second is a version of the so-called elliptical potential lemma (see Lemma 11 and ensuing historical discussion in \cite{abbasi2011improved}). Both claims are essentially known but are stated in forms convenient for our purposes, so the proofs are deferred to Appendix \ref{appLinearClaim}. Here and moving forward, for any $t \in [n]$ and $e \in [L]$, we use $\phi_t(e) = \phi(e) / \sqrt{ \mbf{U}_{t,H}(e) }$ to denote the normalized features for item $e$ at time $t$.

\begin{clm} \label{clmLeastSquares}
Under Assumption \ref{assLinear}, for any $t \in [n]$, $\tau \in [t]$, and $e \in [L]$, on the event $\bar{\mc{E}}_{t,H}$ we have
\begin{equation*}
| \ip{ \hat{\mbf{\theta}}_{\tau,H} }{ \phi(e) } - \bar{w}(e) | \leq \alpha_{\tau,H} \| \phi(e) \|_{\mbf{\Lambda}_{\tau,H}^{-1} } , \quad \mbf{U}_{\tau,H}(e) \geq \bar{w}(e) ,
\end{equation*}
and similarly, on the event $\bar{\mc{E}}_{t,B}$ we have
\begin{equation*}
| \ip{ \hat{\mbf{\theta}}_{\tau,B} }{ \phi(e) } - \bar{w}(e) | \leq \alpha_{\tau,B} \| \phi(e) \|_{\mbf{\Lambda}_{\tau,B}^{-1} } , \quad \mbf{U}_{\tau,B}(e) \geq \bar{w}(e) .
\end{equation*}
\end{clm}

\begin{clm} \label{clmElliptical}
Under Assumption \ref{assLinear}, for any $t \in [n]$, we have
\begin{gather*}
\sum_{s=1}^t \ind ( \bar{\mc{F}}_s ) \sum_{k=1}^{ \min \{ \mbf{C}_s , K \} } \| \phi ( \mbf{a}_k^s ) \|_{\mbf{\Lambda}_{s,H}^{-1}}^2 \leq 4 \log \left( \frac{ \det ( \mbf{\Lambda}_{t+1,H} ) }{ \det ( \mbf{\Lambda}_{1,H} ) } \right) \leq 4 d \log \left(  1 + \frac{t K}{d} \right) ,\\
\sum_{s=1}^t \ind (\bar{\mc{F}}_s ) \sum_{k=1}^{ \min \{ \mbf{C}_s , K \} } \| \phi_s ( \mbf{a}_k^s ) \|_{\mbf{\Lambda}_{s,B}^{-1}}^2 \leq 4 \log \left( \frac{ \det ( \mbf{\Lambda}_{t+1,B} ) }{ \det ( \mbf{\Lambda}_{1,B} ) } \right) \leq 4 d \log \left(  1 + \frac{t K}{d} \right) .
\end{gather*}
\end{clm}

We can now state and prove three lemmas that together establish \eqref{eqLinearSts}.

\begin{lem} \label{lemLinearMain}
Under Assumption \ref{assLinear},
\begin{align}
R_1(n) & \leq 16 d  \left( \frac{ 8 \sqrt{n} \alpha_{n,B}^2  }{ \sqrt{d(d+K)} } + \frac{ 3 n^{1/6} K^{1/3} \alpha_{n,B}^{4/3} \alpha_{n,H}^{2/3} }{ (d(d+K))^{1/6}  } \right) \log \left( 1 + \frac{nK}{d} \right) \\
& = \tilde{O} \left( \sqrt{n d (d+K)}  + n^{1/6} d^{7/6} (d+K)^{1/2} K^{1/3} \right)  .
\end{align}
\end{lem}
\begin{proof}
First, we fix $t \in [n]$, $e \in [L] \setminus [K]$, and $e^* \in \bar{\mc{K}}_\Delta(e)$, and we assume that $\bar{\mc{E}}_t$, $\bar{\mc{F}}_t$, and $G_{e,e^*,t}$ all hold. Under this assumption, we can write
\begin{align} \label{eqWofulChange1}
\mbf{U}_{t,B}(e) \leq \bar{w}(e) + 2 \alpha_{t,B} \| \phi(e) \|_{\mbf{\Lambda}_{t,B}^{-1}} = \bar{w}(e) + 2 \alpha_{t,B} \sqrt{ \mbf{U}_{t,H}(e) } \| \phi_t(e) \|_{\mbf{\Lambda}_{t,B}^{-1}} ,
\end{align}
where we used Claim \ref{clmLeastSquares} and the definition of $\phi_t$. We similarly have
\begin{align} \label{eqWofulChange2}
\mbf{U}_{t,H}(e) & \leq \max \left\{ \bar{w}(e) + 2 \alpha_{t,H}\| \phi(e) \|_{\mbf{\Lambda}_{t,H}^{-1}} , \frac{1}{K} \right\} \\
& \leq 2 \left( \max \left\{ \frac{\bar{w}(K)}{2} , \frac{1}{K} \right\} + \alpha_{t,H} \| \phi(e) \|_{\mbf{\Lambda}_{t,H}^{-1}} \right) ,
\end{align}
where the first inequality again holds by Claim \ref{clmLeastSquares} and the second uses $\bar{w}(e) \leq \bar{w}(K)$ by Assumption \ref{assLinear} and the fact that, for any $x,y,z \geq 0$,
\begin{equation*}
2 ( \max \{ x/2 , z \} + y ) = \max \{ x + 2 y , 2 z + 2 y \} \geq \max \{ x + 2 y , z \}  .
\end{equation*}
Combining the two inequalities with the bound $\sqrt{x+y} \leq \sqrt{x}+\sqrt{y}$ for $x,y\geq0$, we obtain
\begin{align} \label{eqWofulChange3}
\mbf{U}_{t,B}(e) & \leq \bar{w}(e) + 2 \alpha_{t,B} \sqrt{ 2 \max \{ \bar{w}(K) / 2 , 1 / K \}} \| \phi_t(e) \|_{\mbf{\Lambda}_{t,B}^{-1}} \\
& \quad + 2 \alpha_{t,B} \sqrt{ 2 \alpha_{t,H} \| \phi(e) \|_{\mbf{\Lambda}_{t,H}^{-1}}  } \| \phi_t(e) \|_{\mbf{\Lambda}_{t,B}^{-1}} \\
&  \leq \bar{w}(e) + \gamma_B \| \phi_t(e) \|_{\mbf{\Lambda}_{t,B}^{-1}} + \gamma_H \sqrt{ \| \phi(e) \|_{\mbf{\Lambda}_{t,H}^{-1}}  } \| \phi_t(e) \|_{\mbf{\Lambda}_{t,B}^{-1}} ,
\end{align}
where we defined $\gamma_B = 2 \alpha_{n,B} \sqrt{ 2 \max \{ \bar{w}(K) / 2 , 1 / K \} }$ and $\gamma_H = 2 \alpha_{n,B} \sqrt{ 2 \alpha_{n,H} }$. Furthermore, note that by Claim \ref{clmLeastSquares}, we have $\mbf{U}_{t,B}(e^*) \geq \bar{w}(e^*) = \bar{w}(e) + \Delta_{e,e^*}$ on $\bar{\mc{E}}_t$. Combined with the previous inequality and \eqref{eqGtoUCB}, we conclude that on $\bar{\mc{E}}_t \cap G_{e,e^*,t}$, we have
\begin{align} \label{eqWofulChange4}
\Delta_{e,e^*} & \leq \mbf{U}_{t,B}(e^*) - \bar{w}(e) \leq \mbf{U}_{t,B}(e) - \bar{w}(e) \\
& \leq  \gamma_B \| \phi_t(e) \|_{\mbf{\Lambda}_{t,B}^{-1}} + \gamma_H \sqrt{ \| \phi(e) \|_{\mbf{\Lambda}_{t,H}^{-1}}  } \| \phi_t(e) \|_{\mbf{\Lambda}_{t,B}^{-1}}  .
\end{align}
Next, notice that one of the summands at right must exceed $\Delta_{e,e^*}/2$ (else, their sum will be less than $\Delta_{e,e^*}$, a contradiction), so one of the following events must occur:
\begin{gather*}
\mc{A}_{e,e^*,t} = \left\{ \Delta_{e,e^*}^2  \leq   ( 2 \gamma_B )^2 \| \phi_t(e) \|_{\mbf{\Lambda}_{t,B}^{-1}}^2  \right\} , \\
\mc{B}_{e,e^*,t} = \left\{ \Delta_{e,e^*}^{4/3}  \leq ( 2 \gamma_H )^{4/3} \| \phi(e) \|_{\mbf{\Lambda}_{t,H}^{-1}}^{2/3} \| \phi_t(e) \|_{\mbf{\Lambda}_{t,B}^{-1}}^{4/3} \right\} .
\end{gather*}
In summary, we have shown $\bar{\mc{E}}_t \cap G_{e,e^*,t} \subset \mc{A}_{e,e^*,t} \cup \mc{B}_{e,e^*,t}$, which implies
\begin{align*}
\ind ( \bar{\mc{E}}_t , \bar{\mc{F}}_t , G_{e,e^*,t} ) = \ind ( \bar{\mc{E}}_t , G_{e,e^*,t} ) \ind ( \bar{\mc{F}}_t , G_{e,e^*,t} ) \leq ( \ind ( \mc{A}_{e,e^*,t} ) + \ind ( \mc{B}_{e,e^*,t} ) ) \ind ( \bar{\mc{F}}_t , G_{e,e^*,t} ) . 
\end{align*}
Furthermore, since $\ind ( x \leq y ) \leq y/x$ for $x , y > 0$ and $2^{4/3} = 16^{1/3} < 27^{1/3} = 3$, we know that
\begin{equation*}
\ind ( \mc{A}_{e,e^*,t} ) + \ind ( \mc{B}_{e,e^*,t} ) \leq \frac{ 4 \gamma_B^2  }{ \Delta_{e,e^*}^2 } \| \phi_t(e) \|_{\mbf{\Lambda}_{t,B}^{-1}}^2 + \frac{ 3 \gamma_H^{4/3}   }{ \Delta_{e,e^*}^{4/3} } \| \phi(e) \|_{\mbf{\Lambda}_{t,H}^{-1}}^{2/3} \| \phi_t(e) \|_{\mbf{\Lambda}_{t,B}^{-1}}^{4/3} .
\end{equation*}
Combining the previous two inequalities and using $\Delta_{e,e^*} \geq \Delta$ for $e^* \in \bar{\mc{K}}_\Delta(e)$ by definition, we get
\begin{equation*}
\Delta_{e,e^*} \ind ( \bar{\mc{E}}_t , \bar{\mc{F}}_t , G_{e,e^*,t} ) \leq \left( \frac{ 4 \gamma_B^2  }{ \Delta } \| \phi_t(e) \|_{\mbf{\Lambda}_{t,B}^{-1}}^2 + \frac{ 3 \gamma_H^{4/3}  }{ \Delta^{1/3}  }  \| \phi(e) \|_{\mbf{\Lambda}_{t,H}^{-1}}^{2/3} \| \phi_t(e) \|_{\mbf{\Lambda}_{t,B}^{-1}}^{4/3} \right) \ind ( \bar{\mc{F}}_t , G_{e,e^*,t} ) . 
\end{equation*}
Summing both sides over $e^* \in \bar{\mc{K}}_\Delta(e)$ and using \eqref{eqNumberOfG}, we conclude
\begin{align*}
& \sum_{e^* \in \bar{\mc{K}}_\Delta(e)} \Delta_{e,e^*} \ind ( \bar{\mc{E}}_t , \bar{\mc{F}}_t , G_{e,e^*,t} ) \\
& \quad \leq \ind ( \bar{\mc{F}}_t ) \left( \frac{ 4 \gamma_B^2  }{ \Delta } \| \phi_t(e) \|_{\mbf{\Lambda}_{t,B}^{-1}}^2 + \frac{ 3 \gamma_H^{4/3}  }{ \Delta^{1/3}  }  \| \phi(e) \|_{\mbf{\Lambda}_{t,H}^{-1}}^{2/3} \| \phi_t(e) \|_{\mbf{\Lambda}_{t,B}^{-1}}^{4/3} \right) \ind ( \mbf{T}_t(e) = \mbf{T}_{t-1}(e) +1 ) .
\end{align*}
Now observe the right side is nonzero if and only if $e = \mbf{a}_k^t$ for some $k \in [ \min \{ \mbf{C}_t , K \} ]$. Therefore, summing over $e \in [L] \setminus [K]$ and using nonnegativity of the summands gives
\begin{align*}
& \sum_{e=K+1}^L \sum_{e^* \in \bar{\mc{K}}_\Delta(e)} \Delta_{e,e^*} \ind ( \bar{\mc{E}}_t , \bar{\mc{F}}_t , G_{e,e^*,t} ) \\
& \quad \leq \frac{ 4 \gamma_B^2  }{ \Delta } \ind ( \bar{\mc{F}}_t ) \sum_{k=1}^{ \min \{ \mbf{C}_t , K \} }  \| \phi_t(\mbf{a}_k^t) \|_{\mbf{\Lambda}_{t,B}^{-1}}^2  + \frac{ 3 \gamma_H^{4/3}  }{ \Delta^{1/3}  } \ind ( \bar{\mc{F}}_t ) \sum_{k=1}^{ \min \{ \mbf{C}_t , K \} } \| \phi(\mbf{a}_k^t) \|_{\mbf{\Lambda}_{t,H}^{-1}}^{2/3} \| \phi_t(\mbf{a}_k^t) \|_{\mbf{\Lambda}_{t,B}^{-1}}^{4/3}  .
\end{align*}
Summed over $t \in [n]$, the remaining summations are both upper bounded by $4 d \log ( 1 + n K / d )$ (by Claim \ref{clmElliptical}, along with Holder's inequality for the second summation). We thus obtain
\begin{align} \label{eqLinR1fin}
R_1(n)  \leq 4 d  \left( \frac{ 4 \gamma_B^2 }{ \Delta } + \frac{ 3 \gamma_H^{4/3} }{\Delta^{1/3}} \right) \log \left( 1 + \frac{n K}{ d }  \right) .
\end{align}
Finally, recall that by definition, we have
\begin{equation*}
\Delta = \max \{ \bar{w}(K) / 2 , 1/ K \} \sqrt{d(d+K)/n} \geq \sqrt{ d(d+K) / (nK^2)} .
\end{equation*}
Plugging the middle and right expressions into the $\Delta$ and $\Delta^{1/3}$ terms in \eqref{eqLinR1fin}, respectively, and recalling $\gamma_B = 2 \alpha_{n,B} \sqrt{ 2 \max \{ \bar{w}(K) / 2 , 1 / K \} }$ and $\gamma_H = 2 \alpha_{n,B} \sqrt{ 2 \alpha_{n,H} }$, yields the first upper bound of the lemma. The second follows since $\alpha_{n,B} = \tilde{O} ( \sqrt{d} + \sqrt{K} )$ and $\alpha_{n,B} = \tilde{O} ( \sqrt{d}  )$.
\end{proof}

\begin{lem} \label{lemLinearE}
Under Assumption \ref{assLinear}, $R_2(n) \leq 4 K + 1$.
\end{lem}

To prove Lemma \ref{lemLinearE}, we require the following Bernstein-style concentration result. It is essentially implied by Theorem 2 of \cite{zhou2021nearly}, but the proof requires some minor modifications to accommodate the fact that, unlike their work, we do not assume an almost-sure upper bound on the conditional variances $\E [ \eta_t^2 | \mc{G}_t ]$. The details of these changes are discussed in Appendix \ref{appBernstein}.
\begin{clm} \label{clmBernstein}
Fix $L, R, \lambda , \delta > 0$.\footnote{To unify the presentation with \cite{zhou2021nearly}, we adopt their notation for this claim only. In particular, notice that $L$ is an upper bound on $\| \mbf{x}_t \|_2$, not the total number of items (as in the rest of the paper).} Suppose $\{ \mc{G}_t \}_{t=1}^T$ is a filtration and $\{ \mbf{x}_t , \eta_t \}_{t=1}^T$ is a stochastic process such that, for each $t \in [T]$, $\mbf{x}_t$ is $\mathscr{B}_d(L)$-valued and $\mc{G}_t$-measurable, $\eta_t$ is $[ - R , R ]$-valued and $\mc{G}_{t+1}$-measurable, and $\E [ \eta_t | \mc{G}_t ] = 0$. For each $t \in [T]$, define $\beta_t = 8 \sqrt{ d \log ( 1 + t L^2 / ( d \lambda ) ) \log ( 4 t^2 / \delta ) } + 4 R \log ( 4 t^2 / \delta )$ and $\mbf{Z}_t = \lambda I + \sum_{i=1}^t \mbf{x}_i \mbf{x}_i^\trans$. Then
\begin{equation*}
\P \left( \cup_{t=1}^T \left\{ \left\| \sum_{i=1}^t \mbf{x}_i \eta_i \right\|_{\mbf{Z}_t^{-1}} > \beta_t \right\} \cap \cap_{t=1}^T \{  \E [ \eta_t^2 | \mc{G}_t ] \leq 1 \}  \right) \leq \delta .
\end{equation*}
\end{clm}

\begin{proof}[Proof of Lemma \ref{lemLinearE}]
We first observe that by definition, the following holds for each $t \in [n]$:
\begin{equation*}
\P ( \mc{E}_t ) = \P ( \mc{E}_{t,H} \cup \mc{E}_{t,B} ) = \P ( \mc{E}_{t,H} \cup ( \mc{E}_{t,B} \cap \bar{\mc{E}}_{t,H} ) ) = \P ( \mc{E}_{t,H} ) + \P ( \mc{E}_{t,B} \cap \bar{\mc{E}}_{t,H} ) .
\end{equation*}
Thus, because $R_2(n) = \sum_{t=1}^n \P ( \mc{E}_t )$, it suffices to show
\begin{equation} \label{eqLinConcSts}
\P ( \mc{E}_{t,H} ) \leq 1 / n , \quad \P ( \mc{E}_{t,B} \cap \bar{\mc{E}}_{t,H} ) \leq 4 K / n \quad \forall\ t \in [n] .
\end{equation}
We fix $t \in [n]$ for the remainder of the proof and establish the two bounds in \eqref{eqLinConcSts} in turn.

To prove the first bound, for each $i \in [ (t-1) K ]$, we define 
\begin{gather*}
\rho(i) = \floor*{ 1 + \frac{i-1}{K} } \in [t-1] , \quad \sigma(i) = i - ( \rho(i) - 1 ) K \in [K] , \\
\tilde{\phi}_{i,H} = \phi ( \mbf{a}_{\sigma(i)}^{\rho(i)} ) \ind ( \sigma(i) \leq \mbf{C}_{ \rho(i) } ) , \quad \tilde{\eta}_{i,H} = \eta_{\rho(i)}^{\sigma(i)} , \quad \tilde{\mbf{\Lambda}}_{i,H} = I + \sum_{j=1}^i \tilde{\phi}_{j,H} \tilde{\phi}_{j,H}^\trans .
\end{gather*}
In words, $\tilde{\eta}_{i,H}$ is the $i$-th of the noises $\{ \eta_s^k \}_{(s,k) \in [t-1] \times [K]}$ under the lexicographical ordering.\footnote{The ordering $\prec_{\text{lex}}$ over $[t-1] \times [K]$ that satisfies $(s_1,k_1) \prec_{\text{lex}} (s_2,k_2)$ if $s_1 < s_2$ or $s_1 = s_2$ and $k_1 < k_2$.} Analogously, $\tilde{\phi}_{i,H}$ is the $i$-th element of $\{ \phi( \mbf{a}_k^s ) \ind ( k \leq \mbf{C}_s ) \}_{(s,k) \in [t-1] \times [K]}$ under the same ordering. We define the $i$-th bad event with respect to this ordering by
\begin{equation*}
\tilde{\mc{E}}_{i,H} = \left\{ \left\| \sum_{j=1}^{ i } \tilde{\phi}_{j,H} \tilde{\eta}_{j,H}  \right\|_{ \tilde{\mbf{\Lambda}}_{i,H}^{-1} }^2 > 2 \log \left( \frac{ \sqrt{ \det(\tilde{\mbf{\Lambda}}_{i,H}) / \det(\tilde{\mbf{\Lambda}}_{0,H}) } }{ 1/n } \right) \right\} .
\end{equation*}
Now to prove the first bound in \eqref{eqLinConcSts}, we will show $\mc{E}_{t,H} \subset \cup_{i=0}^{(t-1)K} \tilde{\mc{E}}_{i,H}$ and $\P ( \cup_{i=0}^{(t-1)K} \tilde{\mc{E}}_{i,H} ) \leq 1/n$. For the first result, we  observe that for any $\tau \in [t]$,
\begin{equation*}
\sum_{j=1}^{ (\tau-1) K } \tilde{\phi}_{j,H} \tilde{\eta}_{j,H}  = \sum_{s=1}^{\tau-1} \sum_{k=1}^K \phi(\mbf{a}_k^s) \ind ( k \leq \mbf{C}_s ) \eta_s^k = \sum_{s=1}^{\tau-1} \sum_{k=1}^{ \min \{ \mbf{C}_t , K \} } \phi(\mbf{a}_k^s) \eta_s^k   ,
\end{equation*}
and similarly, $\tilde{\mbf{\Lambda}}_{(\tau-1)K} = \mbf{\Lambda}_{\tau,H}$. Therefore, if $\mc{E}_{t,H}$ occurs, we can find $\tau \in [t]$ such that
\begin{equation*}
\left\| \sum_{j=1}^{ (\tau-1) K } \tilde{\phi}_{j,H} \tilde{\eta}_{j,H}  \right\|_{ \tilde{\mbf{\Lambda}}_{(\tau-1) K,H}^{-1} } = \left\| \sum_{s=1}^{\tau-1} \sum_{k=1}^{ \min \{ \mbf{C}_t , K \} } \phi(\mbf{a}_k^s) \eta_s^k \right\|_{ \mbf{\Lambda}_{\tau,H}^{-1} } > \beta_{\tau,H} .
\end{equation*}
Furthermore, observe that by $\tilde{\mbf{\Lambda}}_{(\tau-1)K} = \mbf{\Lambda}_{\tau,H}$ and Claim \ref{clmElliptical}, we have
\begin{equation} \label{eqBetaTauHlower}
\beta_{\tau,H}^2 = 2 \log \left( \frac{ \sqrt{ (1 + \tau K / d)^d } }{ 1/n } \right) \geq 2 \log \left( \frac{ \sqrt{ \det(\tilde{\mbf{\Lambda}}_{(\tau-1)K,H}) / \det(\tilde{\mbf{\Lambda}}_{0,H}) } }{ 1/n } \right) .
\end{equation}
Combining the previous two inequalities, we obtain the first desired result:
\begin{equation} \label{eqEtchToTilde}
\mc{E}_{t,H} \subset \cup_{\tau=1}^t \tilde{\mc{E}}_{ ( \tau-1 ) K , H } = \cup_{i \in \{ 0 , K , \ldots , (t-1) K \} } \tilde{\mc{E}}_{i,H} \subset \cup_{i=0}^{(t-1)K} \tilde{\mc{E}}_{i,H} .
\end{equation}
Finally, let $\mathscr{G}_{i,H}$ denote the $\sigma$-algebra generated by $\{ \tilde{\phi}_{i',H} \}_{i'=1}^i \cup \{ \tilde{\eta}_{i',H} \}_{i'=1}^{i-1}$. Note that each $\tilde{\phi}_{i,H}$ is $\mathscr{G}_{i,H}$-measurable and each $\tilde{\eta}_{i,H}$ is $[-1,1]$-valued and $\mathscr{G}_{i+1,H}$-measurable with $\E [ \tilde{\eta}_{i,H} | \mathscr{G}_{i,H} ] = 0$. Thus, applying Theorem 1 of \cite{abbasi2011improved} with parameters $R=1$ and $\delta = 1/n$ (in their notation), we conclude that $\P ( \cup_{i=0}^{(t-1)K} \tilde{\mc{E}}_{i,H} ) \leq 1/n$, as desired.

For the second bound in \eqref{eqLinConcSts}, we first define the normalized versions of $\tilde{\phi}_{i,H}$, $\tilde{\eta}_{i,H}$, and $\tilde{\mbf{\Lambda}}_{i,H}$:
\begin{equation*}
\tilde{\phi}_{i,B} = \tilde{\phi}_{i,H} / \sqrt{ \mbf{U}_{ \rho(i) , H } ( \mbf{a}_{\sigma(i)}^{\rho(i)} ) } , \ \tilde{\eta}_{i,B} = \tilde{\eta}_{i,H} / \sqrt{ \mbf{U}_{ \rho(i) , H } ( \mbf{a}_{\sigma(i)}^{\rho(i)} ) } , \ \tilde{\mbf{\Lambda}}_{i,B} = K I + \sum_{j=1}^i \tilde{\phi}_{j,B} \tilde{\phi}_{j,B}^\trans .
\end{equation*}
Next, for any $i \leq n K$, it is easily verified that $n^3 K \geq  4 i^2/ ( 4 K / n )$. Combined with the fact that $\rho(i) + 1 \geq 1 + (i-1)/K \geq i/K$, we can write
\begin{align} \label{eqBetaRhoLower}
\beta_{ \rho(i)+1 , B } \geq \beta_{i/K,B} \geq 8 \sqrt{ d \log \left( 1 + \frac{i}{d} \right) \log \left( \frac{4 i^2}{4 K/n} \right) } + 4 \sqrt{K} \log \left( \frac{4 i^2}{4 K/n} \right) \triangleq \beta_{i,B}' .\ \
\end{align}
Using \eqref{eqBetaRhoLower} instead of \eqref{eqBetaTauHlower} but otherwise following the approach leading to \eqref{eqEtchToTilde}, we analogously obtain $\mc{E}_{t,B} \subset \cup_{i=0}^{(t-1)K} \tilde{\mc{E}}_{i,B}$, where here we define
\begin{equation*}
\tilde{\mc{E}}_{i,B} = \left\{ \left\| \sum_{j=1}^{ i } \tilde{\phi}_{j,B} \tilde{\eta}_{j,B}  \right\|_{ \tilde{\mbf{\Lambda}}_{i,B}^{-1} } > \beta_{i,B}'  \right\} .
\end{equation*}
Again analogous to the above, we let $\mathscr{G}_{i,B}$ be the $\sigma$-algebra generated by $\{ \tilde{\phi}_{i',B} \}_{i'=1}^i \cup \{ \tilde{\eta}_{i',B} \}_{i'=1}^{i-1}$. Then each $\tilde{\phi}_{i,B}$ is $\mathscr{G}_{i,B}$-measurable and $\mathscr{B}_d(\sqrt{K})$-valued, while each $\tilde{\eta}_{i,B}$ is $[-\sqrt{K},\sqrt{K}]$-valued and $\mathscr{G}_{i+1,B}$-measurable with $\E [ \tilde{\eta}_{i,B} | \mathscr{G}_{i,B} ] = 0$. Furthermore, on $\bar{\mc{E}}_{t,H}$, Claim \ref{clmLeastSquares} ensures
\begin{equation} \label{eqVarOnBarEh}
\E [ \tilde{\eta}_{i,B}^2 | \mathscr{G}_{i,B} ] = \frac{ \E [ ( \eta_{\rho(i)}^{\sigma(i)} )^2 | \mathscr{G}_{i,B} ] }{  \mbf{U}_{ \rho(i), H } ( \mbf{a}_{\sigma(i)}^{\rho(i)} ) } = \frac{ \bar{w}( \mbf{a}_{\sigma(i)}^{\rho(i)} ) ( 1 - \bar{w} ( \mbf{a}_{\sigma(i)}^{\rho(i)} ) ) }{  \mbf{U}_{ \rho(i) , H } ( \mbf{a}_{\sigma(i)}^{\rho(i)} ) } \leq \frac{ \bar{w}( \mbf{a}_{\sigma(i)}^{\rho(i)} ) }{  \mbf{U}_{ \rho(i) , H } ( \mbf{a}_{\sigma(i)}^{\rho(i)} ) } \leq 1 .
\end{equation}
Taken together, we have shown
\begin{equation*}
\mc{E}_{t,B}  \cap \bar{\mc{E}}_{t,H} \subset  \cup_{i=0}^{(t-1)K} \tilde{\mc{E}}_{i,B} \cap \cap_{i = 0}^{ (t-1)K }  \{ \E [ \tilde{\eta}_{i,B}^2 | \mathscr{G}_{i,B} ] \leq 1 \} .
\end{equation*}
Applying Claim \ref{clmBernstein} with $L = R = \sqrt{K}$, $\lambda = K$, and $\delta = 4 K / n$ (recall the definition of $\beta_{i,B}'$ in \eqref{eqBetaRhoLower}), we obtain that the probability of the event at right is at most $4K/n$, as desired.
\end{proof}

\begin{lem} \label{lemLinearF}
Under Assumption \ref{assLinear}, $R_3(n) \leq 2 d K \log_2(1+nK/d)$.
\end{lem}
\begin{proof}
Let $\mc{T}_H = \{ t \in [n] : \det(\mbf{\Lambda}_{t+1,H}) \geq 2 \det(\mbf{\Lambda}_{t,H}) \}$ and $\mc{T}_B = \{ t \in [n] : \det(\mbf{\Lambda}_{t+1,B}) \geq  2 \det(\mbf{\Lambda}_{t,B}) \}$. Note that $R_2(n) = K \E [ | \mc{T}_H \cup \mc{T}_B | ] \leq K \E [ | \mc{T}_H | + | \mc{T}_B | ]$. Furthermore, observe
\begin{equation*}
2^{| \mc{T}_H |} \leq \prod_{t \in \mc{T}_H} \frac{\det(\mbf{\Lambda}_{t+1,H})}{\det(\mbf{\Lambda}_{t,H}) } \leq \prod_{t=1}^n \frac{\det(\mbf{\Lambda}_{t+1,H})}{\det(\mbf{\Lambda}_{t,H})} = \frac{ \det(\mbf{\Lambda}_{n+1,H}) }{ \det(\mbf{\Lambda}_{1,H}) } \leq \left( 1 + \frac{nK}{d} \right)^d  ,
\end{equation*}
where we used Claim \ref{clmElliptical}. Therefore, $|\mc{T}_H| \leq d \log_2(1+nK/d)$. The analogous argument gives the same bound for $|\mc{T}_B|$, so this completes the proof.
\end{proof}

\subsection{Proofs of Claims \ref{clmLeastSquares} and  \ref{clmElliptical}} \label{appLinearClaim}

To prove the claims, we begin with two general results. Throughout, we use the notation $A \succ B$ and $A \succeq B$ when $A-B$ is positive definite and positive semidefinite, respectively.

\begin{clm} \label{clmLinAlgJustB}
Let $\lambda > 0$, $\{ x_i \}_{i=1}^m \subset \mathscr{B}_d(\sqrt{\lambda})$, and $B = \lambda I + \sum_{i=1}^m x_i x_i^\trans$. Then (i) $B \succ 0$, (ii) $\| y \|_{B^{-1}} \leq \| y \|_2 / \sqrt{\lambda}\ \forall\ y \in \R^d$, (iii) $\Tr(B) \leq \lambda ( d + m )$, (iv) $\det(B) \leq ( \lambda ( 1 + m/d ) )^d$.
\end{clm}
\begin{proof}
Let $\{ \lambda_j \}_{j=1}^d$ and $\{ v_j \}_{j=1}^d$ be the eigenvalues and eigenvectors of $B$. Then for any $j \in [d]$,
\begin{equation} \label{eqLambdaJ}
\lambda_j = \lambda_j \| v_j \|_2^2 = v_j^\trans ( \lambda_j v_j ) = v_j^\trans ( B v_j ) = v_j^\trans \left( \lambda v_j + \sum_{i=1}^m x_i x_i^\trans v_j \right) = \lambda + \sum_{i=1}^m ( x_i^\trans v_j )^2 .
\end{equation}
Therefore, $\min_{j \in [d]} \lambda_j \geq \lambda > 0$, which implies (i) and ensures that $B$ is invertible. Next, let $V$ be the orthogonal matrix with columns $\{ v_j \}_{j=1}^d$. Then by unitary invariance, we know that
\begin{equation} \label{eqUnitaryInvar}
\sum_{j=1}^d ( v_j^\trans y )^2 = \sum_{j=1}^d ( V^\trans y )_j^2 = \| V^\trans y \|_2^2 = \| y \|_2^2\ \forall\ y \in \R^d .
\end{equation}
Combining \eqref{eqUnitaryInvar} and the above bound $\min_{j \in [d]} \lambda_j \geq \lambda$ establishes (ii):
\begin{equation*}
\| y \|_{B^{-1}}^2 = y^\trans B^{-1} y = y^{\trans} \left( \sum_{j=1}^d v_j v_j^\trans / \lambda_j \right) y = \sum_{j=1}^d ( v_j^\trans y )^2 / \lambda_j \leq \sum_{j=1}^d ( v_j^\trans y )^2 / \lambda = \| y \|_2^2 / \lambda .
\end{equation*}
Similarly, using \eqref{eqLambdaJ}, \eqref{eqUnitaryInvar}, and $\max_{i \in [m]} \| x_i \|_2 \leq \sqrt{\lambda}$, we obtain (iii):
\begin{equation*}
\Tr(B) = \sum_{j=1}^d \lambda_j \leq \lambda d + \sum_{i=1}^m \sum_{j=1}^d ( x_i^\trans v_j )^2 \leq \lambda d  + m \max_{i \in [m]} \| x_i \|_2^2 \leq \lambda ( d + m )
\end{equation*}
Finally, (iii) and the inequality of the arithmetic and geometric means together yield (iv):
\begin{equation*}
\det(B)^{1/d} = \left( \prod_{j=1}^d \lambda_j \right)^{1/d} \leq  \sum_{j=1}^d \lambda_j / d = \Tr(B) / d \leq \lambda ( 1 + m / d ) . \qedhere
\end{equation*}
\end{proof}

\begin{clm} \label{clmLinAlgBjs}
Let $\lambda > 0$, $\{ x_i \}_{i=1}^m \subset \mathscr{B}_d(\sqrt{\lambda})$, and $B_j = \lambda I + \sum_{i=1}^{j-1} x_i x_i^\trans$ for each $j \in [m+1]$. Then (i) $\sum_{j=1}^m \| x_j \|_{B_j^{-1}}^2 \leq 2 \log ( \det(B_{m+1}) / \det(B_1) )$ and (ii) for any $y \in \R^d$ and $j_1 , j_2 \in [m]$ such that $j_1 \geq j_2$, we have $\| y \|_{B_{j_1}^{-1}} \leq \| y \|_{B_{j_2}^{-1}} \leq \sqrt{ \det(B_{j_1}) / \det ( B_{j_2} ) } \| y \|_{B_{j_1}^{-1}}$.
\end{clm}
\begin{proof}
By part (ii) of Claim \ref{clmLinAlgJustB} and $x_j \in \mathscr{B}_d(\sqrt{\lambda})$, we know $\| x_j \|_{B_j^{-1}} \leq 1$ for any $j \in [m]$. Hence, by Lemma 11 of \cite{abbasi2011improved}, we obtain (i):
\begin{equation*}
\sum_{j=1}^m \| x_j \|_{B_j^{-1}}^2 = \sum_{j=1}^m \min \{ 1 , \| x_j \|_{B_j^{-1}}^2 \} \leq 2 \log ( \det(B_{m+1}) / \det(B_1) ) .
\end{equation*}
For (ii), note that $y^\trans ( B_{j_1} - B_{j_2} ) y = \sum_{i=j_2}^{j_1-1} ( x_i^\trans y )^2 \geq 0$, so $B_{j_1} \succeq B_{j_2}$. This implies $B_{j_2}^{-1} \succeq B_{j_1}^{-1}$ and yields the bound $\| y \|_{B_{j_1}^{-1}} \leq \| y \|_{B_{j_2}^{-1}}$. For the other bound in (ii), let $\Xi = \det(B_{j_1}) / \det(B_{j_2})$. Then Lemma 12 of \cite{abbasi2011improved} implies $y^\trans B_{j_1} y \leq \Xi y^\trans B_{j_2} y$, so $B_{j_2} \succeq B_{j_1} / \Xi$. Therefore, $\Xi B_{j_1}^{-1} \succeq B_{j_2}^{-1}$, so $\| y \|_{B_{j_2}^{-1}} \leq \| y \|_{\Xi B_{j_1}^{-1}} = \sqrt{\Xi} \| y \|_{B_{j_1}^{-1}}$.
\end{proof}

\begin{proof}[Proof of Claim \ref{clmLeastSquares}] Fix $t \in [n]$, $\tau \in [t]$, and $e \in [L]$, and assume $\bar{\mc{E}}_{t,H}$ holds. Then
\begin{equation*}
\theta = \mbf{\Lambda}_{\tau,H}^{-1} \left( ( \mbf{\Lambda}_{\tau,H} - I ) \theta + \theta \right) = \mbf{\Lambda}_{\tau,H}^{-1} \left( \sum_{s=1}^{\tau-1} \sum_{k=1}^{\min\{\mbf{C}_s,K\}} \phi(\mbf{a}_k^s) \bar{w}(\mbf{a}_k^s) + \theta \right) ,
\end{equation*} 
where we used the definition of $\mbf{\Lambda}_{\tau,H}$ and Assumption \ref{assLinear}. Therefore, we can write
\begin{equation*}
\| \hat{\mbf{\theta}}_{\tau,H} - \theta \|_{\mbf{\Lambda}_{\tau,H}} = \left\| \sum_{s=1}^{\tau-1} \sum_{k=1}^{\min\{\mbf{C}_s,K\}} \phi(\mbf{a}_k^s) \mbf{\eta}_k^s - \theta \right\|_{\mbf{\Lambda}_{\tau,H}^{-1}} \leq \beta_{\tau,H} + \| \theta \|_{\mbf{\Lambda}_{\tau,H}^{-1}} \leq \beta_{\tau,H} + \| \theta \|_2 \leq \alpha_{\tau,H} ,
\end{equation*}
where the inequalities hold by the triangle inequality and the definition of $\bar{\mc{E}}_{t,H}$; part (ii) of Claim \ref{clmLinAlgJustB}, which applies with parameter $\lambda = 1$ by definition of $\mbf{\Lambda}_{\tau,H}$ and Assumption \ref{assLinear}; and Assumption \ref{assLinear} and the definition of $\beta_{\tau,H}$; respectively. By this inequality and Cauchy-Schwarz, we obtain
\begin{equation*}
| \ip{ \hat{\mbf{\theta}}_{\tau,H} }{ \phi(e) } - \bar{w}(e) | = | \ip{ \hat{\mbf{\theta}}_{\tau,H} - \theta }{ \phi(e) } | \leq \| \hat{\mbf{\theta}}_{\tau,H} - \theta \|_{\mbf{\Lambda}_{\tau,H}} \| \phi(e) \|_{\mbf{\Lambda}_{\tau,H}^{-1}} \leq \alpha_{\tau,H}  \| \phi(e) \|_{\mbf{\Lambda}_{\tau,H}^{-1} } .
\end{equation*}
As an immediate consequence, we see that
\begin{equation} \label{eqUcbHonBarE}
\mbf{U}_{\tau,H}(e) = \max \{ \ip{ \hat{\mbf{\theta}}_{\tau,H} }{ \phi(e) }  + \alpha_{\tau,H}  \| \phi(e) \|_{\mbf{\Lambda}_{\tau,H}^{-1} } , 1 / K \} \geq \max \{ \bar{w}(e) , 1 / K \} \geq \bar{w}(e) .
\end{equation}
The remaining bounds follow similarly, except we must invoke Claim \ref{clmLinAlgJustB} with $\lambda = K$ and use the fact that $\| \phi_t(e) \|_2 = \| \phi(e) \|_2 / \sqrt{\mbf{U}_{t,H}(e)} \leq \sqrt{K}$ by definition and Assumption \ref{assLinear}.
\end{proof}

\begin{proof}[Proof of Claim \ref{clmElliptical}]
Fix $t \in [n]$ and let $\mbf{\Lambda}_{s,H}^k = \mbf{\Lambda}_{s,H} + \sum_{i=1}^{k-1} \phi(\mbf{a}_i^s) \phi(\mbf{a}_i^s)^\trans$  for each $s \in [t]$ and $k \in [ \min \{ \mbf{C}_s , K \} ]$. Then on the event $\bar{\mc{F}}_s$, we can apply part (ii) of Claim \ref{clmLinAlgBjs} twice (with parameter $\lambda = 1$ by definition of $\mbf{\Lambda}_{s,H}^k$ and Assumption \ref{assLinear}) to obtain that for any $y \in \R^d$,
\begin{equation*}
\| y \|_{ \mbf{\Lambda}_{s,H}^{-1} }^2 \leq 2\| y \|_{ \mbf{\Lambda}_{s+1,H}^{-1} }^2 \leq 2 \| y \|_{ (\mbf{\Lambda}_{s,H}^k )^{-1} }^2 .
\end{equation*}
Using this bound, nonnegativity, and part (i) of Claim \ref{clmLinAlgBjs}, we get
\begin{equation*}
\sum_{s=1}^t \ind ( \bar{\mc{F}}_s ) \sum_{k=1}^{ \min \{ \mbf{C}_s , K \} } \| \phi ( \mbf{a}_k^s ) \|_{\mbf{\Lambda}_{s,H}^{-1}}^2 \leq 2 \sum_{s=1}^t \sum_{k=1}^{ \min \{ \mbf{C}_s , K \} } \| \phi ( \mbf{a}_k^s ) \|_{( \mbf{\Lambda}_{s,H}^k )^{-1}}^2 \leq 4 \log \left( \frac{ \det ( \mbf{\Lambda}_{t+1,H}) }{ \det ( \mbf{\Lambda}_{1,H}) } \right) .
\end{equation*}
Observing $\det ( \mbf{\Lambda}_{1,H}) = 1$ by definition and $\det ( \mbf{\Lambda}_{t+1,H}) \leq ( 1 + t K / d )^d$ by part (iv) of Claim \ref{clmLinAlgJustB} (again, with $\lambda=1$) completes the proof of the first bound. For the second,  one proceeds similarly except uses $\lambda = K$ when invoking Claims \ref{clmLinAlgJustB} and \ref{clmLinAlgBjs} and $\det ( \mbf{\Lambda}_{1,B}) = K^d$ by definition. 
\end{proof}

\subsection{Proof of Claim \ref{clmBernstein}} \label{appBernstein}

As discussed before the statement of the claim, the proof follows \cite{zhou2021nearly} with some minor changes that are necessary because we do not assume an almost-sure bound on $\E [ \eta_t^2 | \mc{G}_t ]$. For brevity, we only explain the changes, and to unify the presentation, we adopt their notation (which is slightly inconsistent with the rest of Appendix \ref{appLinear}) for Appendix \ref{appBernstein} only.  First, we replace their Lemma 11 (a version of Freedman's inequality that assumes the almost-sure bound) with the following.
\begin{lem} \label{lemFreedman}
Fix $M , v , \delta > 0$. Suppose $\{ \mc{G}_t \}_{t=1}^\infty$ is a filtration and $\{ x_t \}_{t=1}^\infty$ is a stochastic process such that, for each $t \in \N$, $x_t$ is $[-M,M]$-valued and $\mc{G}_t$-measurable with $\E [ x_t | \mc{G}_{t-1} ] = 0$. Let $\Gamma = \sqrt{ 2 v \log ( 1 / \delta ) } + 2 M \log ( 1 / \delta ) / 3$. Then
\begin{equation} \label{eqFreedman}
\P \left( \cup_{t=1}^\infty \left( \left\{ \sum_{i=1}^t x_i > \Gamma \right\} \cap \left\{ \sum_{i=1}^t \E [ x_i^2 | \mc{G}_i ] \leq v \right\} \right) \right) \leq \delta .
\end{equation}
\end{lem}
This result follows from the standard form of Freedman's inequality \citep{freedman1975tail}, which says the left side of \eqref{eqFreedman} is at most $\exp ( - ( \Gamma^2 / 2 ) / ( v + M \Gamma / 3 )$, along with the lower bound
\begin{equation*}
\Gamma^2 / 2= \sqrt{ v \log (1/\delta) / 2 }\Gamma + M \log(1/\delta) \Gamma / 3  \geq v \log(1/\delta) + M \log(1/\delta) \Gamma / 3  .
\end{equation*}
Next, in the proof of their Lemma 13, their inequality (29) need not hold almost-surely, but it does hold on $\cap_{t=1}^T \{  \E [ \eta_t^2 | \mc{G}_t ] \leq 1 \}$ (with constant $\sigma^2 = 1$). Therefore, using our Lemma \ref{lemFreedman} instead of their Lemma 11 but otherwise following the proof of their Lemma 13, one can show
\begin{equation} \label{eqLem13analogue}
\P \left(  \cup_{t=1}^T \left\{  \sum_{i=1}^t \frac{\eta_i \mbf{x}_i^\trans \mbf{Z}_{i-1} \mbf{d}_{i-1}}{1+w_i^2} \mc{E}_{i-1} > \frac{3 \beta_t^2}{8}  \right\} \cap \cap_{t=1}^T \{  \E [ \eta_t^2 | \mc{G}_t ] \leq 1 \} \right) \leq \frac{\delta}{2} ,
\end{equation}
where $\mbf{d}_t = \sum_{i=1}^t \mbf{x}_i \eta_i$, $w_t = \| \mbf{x}_t \|_{\mbf{Z}_{t-1}^{-1}}$, and $\mc{E}_t = \ind ( \| \mbf{d}_s \|_{\mbf{Z}_{s-1}^{-1}} \leq \beta_s\ \forall\ s \in [t] )$. By analogously modifying the proof of their Lemma 14, one can also show
\begin{equation} \label{eqLem14analogue}
\P \left(  \cup_{t=1}^T \left\{  \sum_{i=1}^t \frac{\eta_i^2 w_i^2}{1+w_i^2} > \frac{\beta_t^2}{4} \right\} \cap \cap_{t=1}^T \{  \E [ \eta_t^2 | \mc{G}_t ] \leq 1 \} \right) \leq \frac{\delta}{2} .
\end{equation}
Finally, the inductive argument from the end of Appendix B.1 in \cite{zhou2021nearly} implies
\begin{equation*}
\cup_{t=1}^T \left\{ \left\| \sum_{i=1}^t \mbf{x}_i \eta_i \right\|_{\mbf{Z}_t^{-1}}  > \beta_t \right\} \subset \cup_{t=1}^T  \left\{ \sum_{i=1}^t \frac{\eta_i \mbf{x}_i^\trans \mbf{Z}_{i-1} \mbf{d}_{i-1}}{1+w_i^2} \mc{E}_i > \frac{3 \beta_t^2}{8} \right\} \cup \left\{ \sum_{i=1}^t \frac{\eta_i^2 w_i^2}{1+w_i^2} > \frac{\beta_t^2}{4} \right\} ,
\end{equation*}
so taking the intersection with $\cap_{t=1}^T \{  \E [ \eta_t^2 | \mc{G}_t ] \leq 1 \}$ and using \eqref{eqLem13analogue} and \eqref{eqLem14analogue} completes the proof.

\subsection{Proof sketch for Algorithm \ref{algWofulEff}} \label{appImprovedProof}

Compared to Algorithm \ref{algWoful}, the proof for Algorithm \ref{algWofulEff} requires three changes. The first change is that in Claim \ref{clmLeastSquares}, we replace the inequality $\mbf{U}_{\tau,H}(e) \geq \bar{w}(e)$ with $\min \{ \mbf{U}_{\tau,H}(e) , \tilde{\mbf{U}}_{\tau,H}(e) \} \geq \bar{w}(e)$. To prove this stronger form of the claim, we first use an argument like \eqref{eqUcbHonBarE} from the proof of the claim to conclude $\tilde{\mbf{U}}_{\tau,H}(e) \geq \bar{w}(e)$; the other inequality follows, since 
\begin{equation} \label{eqAltUcbH}
\mbf{U}_{\tau,H}(e) = \min \{ \max \{ \tilde{\mbf{U}}_{\tau,H}(e) , 1/K \} , 1 \} \geq \min \{ \tilde{\mbf{U}}_{\tau,H}(e) , 1 \}  \geq \min \{ \bar{w}(e) , 1 \} = 
\bar{w}(e) .
\end{equation}
The second change is in the proof of Lemma \ref{lemLinearMain}. The inequality \eqref{eqWofulChange1} therein holds by the same logic since $\mbf{U}_{t,B}(e)$ is defined identically in the two algorithms. Next, \eqref{eqWofulChange2} continues to hold because it is an upper bound on $\mbf{U}_{t,H}(e)$, whose definition in Algorithm \ref{algWofulEff} was only changed to include a minimum with $1$. Thus, \eqref{eqWofulChange3}, which is a consequence \eqref{eqWofulChange1}-\eqref{eqWofulChange2}, continues to hold as well; moreover, since $\mbf{U}_t(e) \leq \mbf{U}_{t,B}(e)$ by definition in Algorithm \ref{algWofulEff}, the upper bound in \eqref{eqWofulChange3} holds for $\mbf{U}_t(e)$ as well. Next, we observe that by the stronger form of Claim \ref{clmLeastSquares} (the one just discussed) and $\bar{w}(e^*) \leq 1$ that $\mbf{U}_t(e^*) \geq \bar{w}(e^*)$ on $\bar{\mc{E}}_t$. Therefore, the \eqref{eqWofulChange4} still holds, with $\mbf{U}_{t,B}$ replaced by $\mbf{U}_t$. The remainder of the proof does not rely on the form of the UCB, so it goes through identically. Finally, the third change is that in the proof of Lemma \ref{lemLinearE}, we again invoke the stronger form of Claim \ref{clmLeastSquares}, in particular, to establish \eqref{eqVarOnBarEh}. The remainder of the Lemma \ref{lemLinearE} proof is identical.

\section{Proof of Theorem \ref{thmLower}} \label{appLower}

Set $p = 1/(2K)$ and $\Delta = \sqrt{L/(4nK^2)}$. Note that by assumption $n \geq L$, we have $\Delta \in [0,p]$. Let $\mathscr{W}_{p,\Delta}$ denote the set of vectors $\bar{w} \in [0,p]^L$ in which $\bar{w}(e) = (p+\Delta)/2$ for exactly $K$ (optimal) items $e$ and $\bar{w}(e) = p/2$ for the remaining (suboptimal) items. Define $\Pi'$ and $R_{\pi,\bar{w}}'$ as in Section \ref{secAnalysis}. Then using $\mathscr{W}_{p,\Delta} \subset [0,1]^L$, Theorem 1 of \cite{kveton2015cascading} (see the proof of Theorem 4 therein for a similar application), Bernoulli's inequality, and $\Pi \subset \Pi'$, respectively, we can write
\begin{align*}
\inf_{\pi \in \Pi} \sup_{\bar{w} \in [0,1]^L} R_{\pi,\bar{w}}(n) & \geq \inf_{\pi \in \Pi} \sup_{\bar{w} \in \mathscr{W}_{p,\Delta}} R_{\pi,\bar{w}}(n) \geq (1-p)^{K-1} \inf_{\pi \in \Pi} \sup_{\bar{w} \in \mathscr{W}_{p,\Delta}} R_{\pi,\bar{w}}'(n) \\
& \geq \inf_{\pi \in \Pi} \sup_{\bar{w} \in \mathscr{W}_{p,\Delta}} R_{\pi,\bar{w}}'(n) / 2 \geq \inf_{\pi \in \Pi'} \sup_{\bar{w} \in \mathscr{W}_{p,\Delta}} R_{\pi,\bar{w}}'(n) / 2 .
\end{align*}
Now fix $\pi \in \Pi'$. By the previous inequality, it suffices to show $R_{\pi,\bar{w}}'(n) = \Omega(\sqrt{nL})$ for some $\bar{w} \in \mathscr{W}_{p,\Delta}$. For this, we essentially follow \cite{lattimore2018toprank}, but there are some changes that tricky to explain (because our notation -- which follows \cite{kveton2015cascading} -- differs drastically, and because we noticed some small typos in their analysis). For clarity, we therefore provide a complete proof, but we emphasize that it is essentially restated from \cite{lattimore2018toprank}.

First, we introduce some notation. Recall $N=L/K \in \N$, and for each $m \in [N]^K$, define
\begin{gather}
\mc{L}_i = \{ (i-1) N + 1 , \ldots , i N \} , \quad a_i^*(m) = (i-1) N + m(i) , \\
 \mc{L}_i(m) = \mc{L}_i \setminus a_i^*(m) \ \forall\ i \in [K] , \quad \mc{L}(m) = \cup_{i=1}^K \mc{L}_i(m) , \\
A^*(m) = \{ a_i^*(m) \}_{i=1}^K = [L] \setminus \mc{L}(m) , \\
\bar{w}_m(e) = \frac{ p + \Delta \ind ( e \in A^*(m) ) }{2}\ \forall\ e \in [L] .
\end{gather}
In words, we partitioned $[L]$ into groups $\mc{L}_i$. Each vector $m \in [N]^K$ defines an instance $\bar{w}_m$ where one item $a_i^*(m)$ in each group $i$ is $(\Delta/2)$-better than the rest of the group $\mc{L}_i(m)$, and where $A^*(m)$ is the optimal action ($\mc{L}(m)$ are the suboptimal items). More precisely, note each such instance has $K$ items with $\bar{w}(e) = (p+\Delta)/2$ and $\bar{w}(e) = p/2$ for the other items. Therefore, $\{ \bar{w}_m \}_{m \in [N]^K} \subset \mathscr{W}_{p,\Delta}$, so it suffices to show $R_m'(n) \triangleq R_{\pi,\bar{w}_m}'(n) = \Omega(\sqrt{nL})$ for some $m \in [N]^K$.

Toward this end, for any $e \in [L]$, $t \in [n]$, and $\mc{L} \subset [L]$, we let $\mbf{T}_t'(e) = \sum_{s=1}^t \ind ( e \in \mbf{A}_s )$ denote the number of times $e$ was played up to and including time $t$ and $\mbf{T}_t'(\mc{L}) = \sum_{e \in \mc{L}} \mbf{T}_t'(e)$. For $m \in [N]^K$, we define $\E_m = \E_{\pi,\bar{w}_m}$ to be expectation under instance $\bar{w}_m$. Then by definition \eqref{eqDbmRegret}, we have
\begin{equation} \label{eqInstMreg}
R_m'(n) = \frac{\Delta}{2} \E_m \left[ \sum_{t=1}^n \sum_{k=1}^K \ind ( \mbf{a}_k^t \in \mc{L}(m)  ) \right] = \frac{\Delta}{2} \E_m [ \mbf{T}_n' ( \mc{L}(m) ) ] .
\end{equation}
Next, as observed by \cite{lattimore2018toprank}, we can rewrite $\mbf{T}_n'(\mc{L}(m))$ in two useful ways:
\begin{gather}
\mbf{T}_n'(\mc{L}(m)) = n K - \mbf{T}_n'(A^*(m)) = \sum_{i=1}^K ( n - \mbf{T}_n' ( a_i^*(m) ) ) , \quad \mbf{T}_n'(\mc{L}(m)) = \sum_{i=1}^K  \mbf{T}_n'(\mc{L}_i(m))  .
\end{gather}
Therefore, adding the two expressions together, we obtain
\begin{equation} \label{eqInstMaddTwo}
\mbf{T}_n'(\mc{L}(m)) = \frac{2 \mbf{T}_n'(\mc{L}(m))}{2} = \frac{ \sum_{i=1}^K ( n - \mbf{T}_n' ( a_i^*(m) ) + \mbf{T}_n'(\mc{L}_i(m)) )}{2} .
\end{equation}
For each $i \in [K]$, let $\tau_i = \min \{ n , \min \{ t \in [n] : \mbf{T}_t' ( \mc{L}_i ) \geq n \}$. Then if $\tau_i = n$, we have
\begin{align*}
n - \mbf{T}_n' ( a_i^*(m) ) + \mbf{T}_n'(\mc{L}_i(m)) \geq n - \mbf{T}_n' ( a_i^*(m) ) = n - \mbf{T}_{\tau_i}' ( a_i^*(m) ) ,
\end{align*}
while if $\tau_i < n$, we use $\mbf{T}_n' ( a_i^*(m) ) \leq n$, monotocity of $\mbf{T}_\cdot'$, and the definition of $\tau_i$ to write
\begin{align*}
n - \mbf{T}_n' ( a_i^*(m) ) + \mbf{T}_n'(\mc{L}_i(m)) & \geq \mbf{T}_n'(\mc{L}_i(m)) \geq \mbf{T}_{\tau_i}'(\mc{L}_i(m)) \\
& = \mbf{T}_{\tau_i}'(\mc{L}_i  ) - \mbf{T}_{\tau_i}' (a_i^*(m) ) \geq n - \mbf{T}_{\tau_i}' ( a_i^*(m) )  .
\end{align*}
Therefore, in both cases, we have the lower bound $n - \mbf{T}_{\tau_i}'(a_i^*(m))$. Combined with \eqref{eqInstMreg} and \eqref{eqInstMaddTwo}, and lower bounding a maximum by an average, we can therefore write
\begin{equation} \label{eqMaxToAverage}
\max_{m \in [N]^K} R_m'(n) \geq \frac{\Delta}{4} \left( n K - \frac{1}{N^K}  \sum_{i=1}^K \sum_{m \in [N]^K} \E_m [ \mbf{T}_{\tau_i}'(a_i^*(m) ) ]  \right) .
\end{equation}
Fix $i \in [K]$, and for any $m \in [N]^{K-1}$ and $z \in [N]$, let $g_i(m,z) = [\cdots\ m(i-1)\ z\ m(i+1)\ \cdots ] \in [N]^K$. Note that $\{ g_i(m,z): m \in [N]^{K-1} , z \in [N] \} = [N]^K$. Also, recalling $a_i^*(\tilde{m}) = (i-1) N + \tilde{m}(i)$ for any $\tilde{m} \in [N]^K$ by definition, we see that $a_i^*( g_i(m,z) ) = (i-1) N + z$. Thus, we obtain
\begin{align} \label{eqTwoWaysToSum}
\sum_{m \in [N]^K} \E_m [ \mbf{T}_{\tau_i}'(a_i^*(m) ) ] = \sum_{m \in [N]^{K-1}} \sum_{z \in [N]} \E_{ g_i(m,z) } [ \mbf{T}_{\tau_i}' ( (i-1)N+z) ] .
\end{align}
Now fix $m \in [N]^{K-1}$ and let $\bar{w}_{m'}$ be the instance with 
\begin{equation*}
A^*(m') = \{ (j-1) N + m(j) \}_{j=1}^{i-1} \cup \{ j N + m(j) \}_{j=i}^{K-1} , \quad \bar{w}_{m'}(e) = \frac{ p + \Delta \ind ( e \in A^*(m') ) }{2} .
\end{equation*}
Observe that, for any $z \in [N]$, $\bar{w}_{m'}$ is the same as $\bar{w}_{g_i(m,z)}$ except $(i-1)N+z$ is only optimal in the former instance (since $A^*(m') \subset \{ 1 , \ldots , (i-1) N , i N + 1 , \ldots , L \}$). Denote expectation in the former instance by $\E_{m'}$. Then for any $z \in [N]$, following \cite{lattimore2018toprank}, we can use Pinsker's inequality and the chain rule for KL-divergence (see also Section 3.3 of \cite{bubeck2012regret}) to obtain
\begin{align}
\E_{ g_i(m,z) } [ \mbf{T}_{\tau_i}'( (i-1)N+z) ] & \leq \E_{m'} [ \mbf{T}_{\tau_i}'( (i-1)N+z ) ] \\
& \quad + n \sqrt{  \E_{m'} [ \mbf{T}_{\tau_i}'( (i-1)N+z ) ] d ( p/2 , (p + \Delta)/2  ) /2 } .
\end{align}
By Claim \ref{clmKlUpper} below, $\Delta \leq p$, and $p \leq 1/2$, we have $d(p/2 , (p+\Delta)/2) \leq  \Delta^2/p$. Combining with the previous inequality, then summing over $z \in [N]$ and using Cauchy-Schwarz, we obtain
\begin{equation} \label{eqBeforeTauiDefn}
\sum_{z \in [N]} \E_{ g_i(m,z) } [ \mbf{T}_{\tau_i}' ( (i-1)N+z) ]  \leq \E_{m'} [ \mbf{T}_{\tau_i}' (\mc{L}_i ) ] + n \sqrt{  N  \E_{m'} [ \mbf{T}_{\tau_i}' ( \mc{L}_i ) ]  \Delta^2 / (2 p) } .
\end{equation}
Now by definition of $\tau_i$, we know that $\mbf{T}_{\tau_i}' (\mc{L}_i ) = \mbf{T}_{\tau_i-1}' (\mc{L}_i ) + ( \mbf{T}_{\tau_i}' (\mc{L}_i ) - \mbf{T}_{\tau_i-1}' (\mc{L}_i ) ) \leq n+K$. Therefore, we can further upper bound the right side of \eqref{eqBeforeTauiDefn} by
\begin{equation*}
n+K + n \sqrt{ \frac{N(n+K) \Delta^2}{2p}} = n + K + \frac{n N}{2} \sqrt{ \frac{n+K}{ n } } \leq n N \left( \frac{1}{N} + \frac{1}{N^2} + \frac{ \sqrt{1+1/N}}{2}  \right) < \frac{9n N}{10} ,
\end{equation*}
where the equality holds by definitions $p = 1/(2K)$, $\Delta = \sqrt{L/(4nK^2)}$, and $N = L/K$; the first inequality by definition $N = L/K$ and assumption $n \geq L$ (so $K = L/N \leq n/N =n N/N^2$); and the second by $N \geq 4$. Combining with \eqref{eqMaxToAverage}, \eqref{eqTwoWaysToSum}, \eqref{eqBeforeTauiDefn}, and the definition of $\Delta$ gives
\begin{equation*}
\max_{m \in [N]^K} R_m'(n) \geq \frac{\Delta}{4} \left( n K - \frac{1}{N^K} \cdot K \cdot N^{K-1} \cdot \frac{9n N}{10} \right) = \frac{\Delta nK}{40} = \frac{\sqrt{n L}}{80} .
\end{equation*}

\begin{clm} \label{clmKlUpper}
For any $p,q \in (0,1)$, $d(p,q) \leq (p-q)^2 / ( q(1-q) )$.
\end{clm}
\begin{proof}
Since $(p-q)/q , (q-p)/(1-q) \geq -1$, we can use $\log(1+x) \leq x\ \forall\ x \geq -1$ to write
\begin{equation*}
d(p,q) \triangleq p \log \left( 1 + \frac{p-q}{q} \right) + (1-p) \log \left( 1 + \frac{q-p}{1-q} \right) \leq p \left( \frac{p-q}{q} \right) + (1-p) \left( \frac{q-p}{1-q} \right) .
\end{equation*}
Finally, a simple calculation shows that the right side equals the desired upper bound.
\end{proof}

\section{Proof of Theorem \ref{thmUcb}} \label{appProofUcb}

Let $\chi \geq 4$ be an absolute constant to be chosen later. Set $p = 1/(2K)$, $\Delta = \sqrt{ L  / ( \chi n K ) }$, $\bar{w}(e^*) = p\ \forall\ e^* \in [K]$, and $\bar{w}(e) = p - \Delta\ \forall\ e \in [L]\setminus[K]$. Note $p \in [\Delta,1/2]$ by the assumption $n \geq LK$ and the choice $\chi \geq 4$, so $\bar{w} \in [0,1]^L$. Let $\pi$ denote \texttt{CascadeUCB1} and $R(n) = R_{\pi,\bar{w}}(n)$. For each $t \in [n]$ and $\mc{L} \subset [L]$, define $\mbf{T}_t(\mc{L}) = \sum_{e \in \mc{L}} \mbf{T}_t(e)$. Then similar to the beginning of Appendix \ref{appLower}, we can use Theorem 1 of \cite{kveton2015cascading} and Bernoulli's inequality to write
\begin{equation*}
R(n) \geq \Delta \E [ \mbf{T}_n([L] \setminus [K]) ] / 2 = \sqrt{ L / ( 4 \chi n K ) }  \E [ \mbf{T}_n([L]\setminus[K]) ] .
\end{equation*}
Now define the events
\begin{gather}
\mc{E} = \{ \mbf{T}_n([L]\setminus[K]) < n K/4  \} , \quad \mc{F} = \{ \mbf{T}_n([L]) < n K / 2 \} , \\ 
\mc{G} = \{ \mbf{T}_n([K]) > n K/4 \}  , \quad \mc{G}_{e^*} = \{ \mbf{T}_n(e^*) > n / 4 \}\ \forall\ e^* \in [K] .
\end{gather}
Note that $\mbf{T}_n([L]\setminus[K]) \geq nK/4$ on the event $\bar{\mc{E}}$. Substituting into the previous inequality gives
\begin{equation} \label{eqUcbRegToE}
R(n) \geq   \sqrt{L/(4 \chi nK)} \E [ \mbf{T}_n([L]\setminus[K]) \ind ( \bar{\mc{E}} ) ] \geq \sqrt{ n L K / ( 64 \chi )  } ( 1 - \P(\mc{E}))  .
\end{equation}
To upper bound $\P(\mc{E})$, first notice that when $\mc{E}$ and $\bar{\mc{F}}$ both occur, we have
\begin{equation*}
\mbf{T}_n([K]) = \mbf{T}_n([L] ) - \mbf{T}_n([L] \setminus [K]) > (n K/2) - (nK/4) = n K / 4 , 
\end{equation*}
so $\mc{E} \cap \bar{\mc{F}} \subset \mc{G}$. Also observe $\mc{G} \subset \cup_{e^*=1}^K \mc{G}_{e^*}$. Combining and applying the union bound,
\begin{equation} \label{eqUcbProbE}
\P(\mc{E}) = \P(\mc{E} \cap \mc{F}) + \P(\mc{E} \cap \bar{\mc{F}}  ) \leq \P(\mc{F})  + \P(\mc{E}\cap\mc{G}) \leq \P(\mc{F})  + \sum_{e^*=1}^K \P ( \mc{E} \cap \mc{G}_{e^*} ) .
\end{equation}
For the first term, let $\mbf{O}_t = \mbf{T}_t([L]) - \mbf{T}_{t-1}([L])$ denote the number of observations at $t$. Then
\begin{equation*}
\E_t [ \mbf{O}_t ] =  \sum_{k=1}^K \prod_{i=1}^{k-1} ( 1 - \bar{w}(\mbf{a}_k^t ) ) \geq \sum_{k=1}^K ( 1 - p )^{k-1} = \frac{1 - (1-p)^K}{p} \geq \frac{1-\exp(-p K)}{p} ,
\end{equation*}
so by definition $p = 1/(2K)$, we have $\E_t[\mbf{O}_t] \geq 2 (1-\exp(-1/2)) K > 3 K / 4$. Thus, we obtain
\begin{equation*}
\E [ \mbf{T}_n([L]) ] = \E \left[ \sum_{t=1}^n ( \mbf{T}_t([L]) - \mbf{T}_{t-1}([L]) ) \right] = \sum_{t=1}^n \E [ \E_t[\mbf{O}_t]] \geq 3 n K / 4 .
\end{equation*}
Hence, because $\mbf{T}_n([L])$ is $[nK]$-valued, we can use Markov's inequality to write
\begin{equation} \label{eqUcbProbF}
\P ( \mc{F} ) = \P ( n K - \mbf{T}_n([L]) > nK / 2 ) \leq \frac{ \E [ n K -  \mbf{T}_n([L]) ] }{ nK/2 } \leq \frac{nK - 3nK/4}{nK/2} = \frac{1}{2} .
\end{equation}
Now fix $e^* \in [K]$ and suppose $\mc{G}_{e^*}$ occurs. Then $\mbf{T}_n(e^*) \geq \ceil{n/4}$, in which case we can find $t \geq \ceil{n/4}$ such that $\mbf{T}_{t-1}(e^*) = \ceil{n/4}-1$ and $\mbf{U}_t(e^*)$ is among the $K$ highest UCBs. Mathematically, we can express the latter event as $| \{ e \in [L] : \mbf{U}_t(e) \leq \mbf{U}_t(e^*) \} | \geq L-K$, which implies
\begin{align} 
3 L / 4 & \leq L-2K \leq | \{ e \in [L] : \mbf{U}_t(e) \leq \mbf{U}_t(e^*) \} | - K \\
& \leq | \{ e \in [L] \setminus [K] : \mbf{U}_t(e) \leq \mbf{U}_t(e^*) \} | \triangleq \mc{L}_{e^*} ,
\end{align}
where we also used $L \geq 8 K$ by assumption. On the other hand, when $\mc{E}$ occurs, there must exist $e \in \mc{L}_{e^*}$ such that $\mbf{T}_{t-1}(e) < nK / ( 4 |\mc{L}_{e^*}| )$, because otherwise, we can write $\mbf{T}_n([L]\setminus[K]) \geq \mbf{T}_{t-1}(\mc{L}_{e^*}) \geq nK/4$, which contradicts $\mc{E}$. Combining these observations, we see that when both $\mc{G}_{e^*}$ and $\mc{E}$ occur, we can find $t \geq \ceil{n/4}$ and $e \in \mc{L}_{e^*} \subset [L] \setminus [K]$ such that 
\begin{equation*}
\mbf{T}_{t-1}(e^*) = \ceil{n/4}-1 , \quad \mbf{U}_t(e^*) \geq \mbf{U}_t(e) , \quad \mbf{T}_{t-1}(e) \in [ \floor{ n K / ( 4 |\mc{L}_{e^*}| ) } ] \subset [ \floor{ n K / (3L) } ] .
\end{equation*}
Therefore, taking three union bounds, we obtain
\begin{equation} \label{eqUcbBeforeUt}
\P  (\mc{E} \cap \mc{G}_{e^*} ) \leq \sum_{e=K+1}^L \sum_{t=\ceil{\frac{n}{4}}}^n \sum_{s=1}^{\floor{\frac{nK}{3L}}} \P ( \mbf{T}_{t-1}(e^*) = \ceil{n/4}-1 , \mbf{U}_t(e^*) \geq \mbf{U}_t(e) , \mbf{T}_{t-1}(e) = s ) .
\end{equation}
Now fix $e$, $t$, and $s$ as in the triple summation. Then by definition of $\mbf{U}_t$ \eqref{eqUcb1}, we know that
\begin{align*}
& \mbf{T}_{t-1}(e^*) = \ceil{n/4}-1 , \quad \mbf{U}_t(e^*) \geq \mbf{U}_t(e) , \quad \mbf{T}_{t-1}(e) = s \\
&  \Rightarrow \quad \hat{\mbf{w}}_{\ceil{\frac{n}{4}}-1}(e^*) + c_{t,\ceil{\frac{n}{4}}-1} \geq \hat{\mbf{w}}_s(e) + c_{t,s} .
\end{align*}
Note that if the event at bottom occurs, one of the following inequalities must hold:
\begin{equation} \label{eqUcbThreeIneq} 
\hat{\mbf{w}}_{\ceil{\frac{n}{4}}-1}(e^*) \geq p + c_{t,\ceil{\frac{n}{4}}-1} , \ \hat{\mbf{w}}_s(e) \leq p - \Delta - \sqrt{5/6} c_{t,s}  , \ \Delta \geq ( 1 - \sqrt{5/6} ) c_{t,s} - 2 c_{t,\ceil{\frac{n}{4}}-1} .
\end{equation}
Indeed, if none of these inequalities hold, then we obtain a contradiction:
\begin{equation}
\hat{\mbf{w}}_{\ceil{\frac{n}{4}}-1}(e^*) + c_{t,\ceil{\frac{n}{4}}-1} < p + 2 c_{t,\ceil{\frac{n}{4}}-1} < \hat{\mbf{w}}_s(e) + \Delta + 2 c_{t,\ceil{\frac{n}{4}}-1} + \sqrt{5/6} c_{t,s} < \hat{\mbf{w}}_s(e)  + c_{t,s} .
\end{equation}
However, Claim \ref{clmChoiceOfChi} below shows the third inequality in \eqref{eqUcbThreeIneq} cannot occur, provided we choose $\chi$ large. Furthermore, by the union and Hoeffding bounds and the definition $c_{t,s} = \sqrt{ (3/2) \log(t) / s }$, the probability that either of the first two inequalities holds is at most $t^{- 2 (3/2)} + t^{- 2 (3/2) (5/6)} = t^{-3} + t^{-5/2}$. Plugging into \eqref{eqUcbBeforeUt} and using Claim \ref{clmIntApprox} below and the assumption $n \geq 49 K^4$ gives
\begin{equation} \label{eqUcb1lastStep}
\P  (\mc{E} \cap \mc{G}_{e^*} ) \leq \frac{ nK}{3} \cdot \sum_{t=\ceil{n/4}}^\infty ( t^{-3} + t^{-5/2} ) \leq \frac{nK}{3} \cdot 10 n^{-3/2} = \frac{10K}{3 \sqrt{n}} \leq \frac{10}{21K} .
\end{equation}
Combining with \eqref{eqUcbProbE} and \eqref{eqUcbProbF} shows $\P(\mc{E}) \leq 41/42$; substituting into \eqref{eqUcbRegToE} gives the desired bound.

\begin{rem} \label{remLoosenAssumption}
If we replace the above constant $5/6$ with $1-\delta/3$ for fixed $\delta > 0$ (we chose $\delta=1/2$ above), the third inequality in \eqref{eqUcbThreeIneq} again fails for appropriate $\chi$, and the Hoeffding bound gives $t^{-2(3/2)(1-\delta/3)} = t^{\delta-3}$ for the probability of the first two. Following the argument above then shows $\P(\mc{E}\cap\mc{G}_{e^*}) = O ( K n^{\delta-1} )$, which is $O(1/K)$ as in \eqref{eqUcb1lastStep} when $n = \Omega ( K^{2/(1-\delta)} )$. Thus, the assumption $n = \Omega(K^4)$ can be weakened to $n = \Omega ( K^\eta )$ for fixed $\eta > 2$.
\end{rem}

\begin{clm} \label{clmChoiceOfChi}
Under the assumptions of Theorem \ref{thmUcb}, there exists an absolute constant $\chi \geq 4$ such that, for any $t \geq \ceil{n/4}$ and $s \leq n K / (3L)$, we have $( 1 - \sqrt{5/6} ) c_{t,s} - 2 c_{t,\ceil{\frac{n}{4}}-1} > \sqrt{L/(\chi n K)}$.
\end{clm}
\begin{proof}
By assumption $L \geq 800 K$ and $n \geq L K$, we know $n \geq 800 K^2 \geq 800$. The fact that $n \geq 100$ ensures $\ceil{n/4} - 1 \geq (n/4) -1 \geq 6 n / 25$; combined with $L \geq 800 K$, we obtain
\begin{equation*}
2 c_{t,\ceil{\frac{n}{4}}-1} = \sqrt{ \frac{6 \log t}{ \ceil{n/4} - 1 } } \leq \sqrt{ \frac{ 25 \log t }{ n } } = \sqrt{ \frac{ 800 \log t }{ 32 n } } \leq \sqrt{ \frac{L \log t}{ 32 n K } } .
\end{equation*}
For $s \leq n K / (3L)$, we also know $c_{t,s} = \sqrt{3 \log (t) / (2s)} \geq \sqrt{9 L \log(t) / ( 2 n K ) }$, so we conclude
\begin{equation*}
( 1 - \sqrt{5/6} ) c_{t,s} - 2 c_{t,\ceil{\frac{n}{4}}-1} \geq \sqrt{ \frac{ L \log t }{ n K } } \left( ( 1 - \sqrt{5/6} ) \sqrt{9/2} - \sqrt{1/32} \right) .
\end{equation*}
The term in parentheses on the right side is positive. Therefore, using the above bound $n \geq 800$ and the choice $t \geq n/4$, the result follows for $\chi > 1 / ( \log(200) (  ( 1 - \sqrt{5/6} ) \sqrt{9/2} - \sqrt{1/32} )^2 ) \geq 4$.
\end{proof}

\begin{clm} \label{clmIntApprox}
Under the assumptions of Theorem \ref{thmUcb}, $\sum_{t=\ceil{\frac{n}{4}}}^\infty ( t^{-3} + t^{-5/2} ) \leq 10 n^{-3/2}$.
\end{clm}
\begin{proof}
Let $T = \ceil{n/4}-1$. As in Claim \ref{clmChoiceOfChi}, we know $T \geq 6 n / 25 > 0$. Thus, for any $c>1$,
\begin{equation*}
\sum_{t=T+1}^\infty t^{-c} = \sum_{t=T+1}^\infty \int_{x=t-1}^t t^{-c} dx \leq \sum_{t=T+1}^\infty \int_{x=t-1}^t x^{-c} dx = \int_{x=T}^\infty x^{-c} dx = \frac{ T^{1-c} }{ c-1 } .
\end{equation*} 
Using this inequality for $c \in \{3,5/2\}$ and the bound $T \geq 6n/25$, we thus obtain
\begin{equation*}
\sum_{t=T+1}^\infty ( t^{-3} + t^{-5/2} ) \leq \frac{ T^{-2} }{ 2 } + \frac{ T^{-3/2} }{ 3/2 } \leq \frac{7}{6} T^{-3/2} \leq \frac{7}{6} \left( \frac{25}{6} \right)^{3/2} n^{-3/2} \leq 10 n^{-3/2}. \qedhere
\end{equation*}
\end{proof}

\end{document}